%% file: main.tex
\documentclass{article}
\usepackage{graphicx} 

\input{imports}

\title{Statistical Inference and Learning for Shapley Additive Explanations (SHAP)}
\author[1]{Justin Whitehouse}
\author[1]{Ayush Sawarni}
\author[1]{Vasilis Syrgkanis\thanks{VS was supported by NSF Award IIS-2337916.}}

\affil[1]{Management Science and Engineering, Stanford University\par
\texttt{\{jwhiteho, ayushsaw, vsyrgk\}@stanford.edu}}

\date{\today}

\begin{document}
\maketitle
\input{abstract}

\input{intro_new}


\input{background}




\input{shap_powers}

\input{learning_revised}

\input{conclusion}

\bibliography{bib}
\bibliographystyle{plainnat}

\appendix

\input{appendix/related_work}
\input{appendix/shap_powers_proofs}
\input{appendix/smooth_clt}
\input{appendix/learning}
\input{appendix/experiments}
\end{document}

%% file: imports.tex
\usepackage[margin=1in]{geometry}
\usepackage{amsmath}
\usepackage{algorithm}
\usepackage[noend]{algpseudocode}
\usepackage{amsfonts}
\usepackage{amssymb}
\usepackage{bm}
\usepackage[numbers]{natbib}
\usepackage{enumitem}
\usepackage{bbm}
\usepackage{comment}
\usepackage{stmaryrd}
\usepackage{theoremref}
\usepackage{mathrsfs}
\usepackage{amsthm}

\usepackage{xcolor}

\renewcommand{\P}{\mathbb{P}}
\newcommand{\N}{\mathbb{N}}

\newcommand{\E}{\mathbb{E}}
\newcommand{\R}{\mathbb{R}}
\newcommand{\G}{\mathbb{G}}

\newcommand{\Rem}{\mathrm{Rem}}
\newcommand{\sech}{\mathrm{sech}}
\newcommand{\Bias}{\mathrm{Bias}}

\newcommand{\calF}{\mathcal{F}}
\newcommand{\sgn}{\mathrm{sgn}}
\newcommand{\calX}{\mathcal{X}}
\newcommand{\calV}{\mathcal{V}}
\newcommand{\calY}{\mathcal{Y}}
\newcommand{\calH}{\mathcal{H}}
\newcommand{\calG}{\mathcal{G}}

\newcommand{\calR}{\mathcal{R}}

\newcommand{\calC}{\mathcal{C}}

\newcommand{\calZ}{\mathcal{Z}}
\newcommand{\calA}{\mathcal{A}}
\newcommand{\calN}{\mathcal{N}}
\newcommand{\calS}{\mathcal{S}}
\newcommand{\calW}{\mathcal{W}}

\newcommand{\calU}{\mathcal{U}}

\newcommand{\Cov}{\mathrm{Cov}}

\newcommand{\Var}{\mathrm{Var}}

\newcommand{\1}{\mathbbm{1}}
\newcommand{\wh}[1]{\widehat{#1}}
\newcommand{\wt}[1]{\widetilde{#1}}
\newcommand{\wb}[1]{\overline{#1}}

\newcommand{\corr}{\mathrm{corr}}

\newcommand{\trunc}{\mathrm{trunc}}

\newcommand{\naive}{\mathrm{naive}}

\algrenewcommand\algorithmicrequire{\textbf{Input:}}
\algrenewcommand\algorithmicensure{\textbf{Output:}}

\newtheorem{theorem}{Theorem}
\numberwithin{theorem}{section}
\newtheorem{lemma}[theorem]{Lemma}

\newtheorem{prop}[theorem]{Proposition}

\newtheorem{ass}{Assumption}

\theoremstyle{definition}
\newtheorem{definition}[theorem]{Definition}

\theoremstyle{remark}
\newtheorem{remark}[theorem]{Remark}

\usepackage{hyperref}
\hypersetup{
	colorlinks=false,
	bookmarks=true,
	breaklinks=true,
	hidelinks=true,
	pdfpagemode=empty,
}
\usepackage{nameref}
\usepackage{authblk}
\usepackage[noabbrev, capitalize, nameinlink]{cleveref}

\usepackage[colorinlistoftodos,textsize=footnotesize]{todonotes}


%% file: abstract.tex
\begin{abstract}

The SHAP (short for Shapley additive explanation) framework has become an essential tool for attributing importance to variables in predictive tasks.
In model-agnostic settings, SHAP uses the concept of Shapley values from cooperative game theory to fairly allocate credit to the features in a vector $X \in \R^d$ based on their contribution to an outcome $Y$. 
While the explanations offered by SHAP are local by nature, learners often need global measures of feature importance in order to improve model explainability and perform feature selection. 
The most common approach for converting these local explanations into global ones  is to compute either the mean absolute SHAP or mean squared SHAP.
However, despite their ubiquity, there do not exist approaches for performing statistical inference on these quantities.

In this paper, we take a semi-parametric approach for calibrating confidence in estimates of the $p$th powers of Shapley additive explanations. We show that, by treating the SHAP curve as a nuisance function that must be estimated from data, one can reliably construct asymptotically normal estimates of the $p$th powers of SHAP. When $p \geq 2$, we show a de-biased estimator that combines U-statistics with Neyman orthogonal scores for functionals of nested regressions is asymptotically normal.
When $1 \leq p < 2$ (and the hence target parameter is not twice differentiable), we construct de-biased U-statistics for a smoothed alternative. In particular, we show how to carefully tune the temperature parameter of the smoothing function in order to obtain inference for the true, unsmoothed $p$th power. 
We complement these results by presenting a Neyman orthogonal loss that can be used to learn the SHAP curve via empirical risk minimization and discussing excess risk guarantees for commonly used function classes.

\end{abstract}

%% file: intro_new.tex
\section{Introduction}
\label{sec:intro}
Model-agnostic variable importance metrics (VIMs), which quantify the degree to which features in a vector $X \in \R^d$ contribute to an outcome $Y$, play an integral role in explainable machine learning (ML) and statistics.
Because model-agnostic VIMs attempt to explain the true, unknown data generating process, these metrics may be preferred to model-specific ones, which instead justify how features in $X \in \R^d$ explain the predictions of a model $\wh{\mu}(X)$. 
In medicine, doctors may leverage model-agnostic VIMs to decide which of a patient's attributes to use in deciding whether or not to administer a risky drug. Likewise, in online marketing, one may use such VIMs to explain to stakeholders which customer characters drive engagement with advertisements.
Given the importance of model-agnostic VIMs in feature selection, decision making, and hypothesis testing problems, it is essential to construct valid confidence intervals for these metrics. 

Of the countless variable importance metrics (either model-agnostic or model-specific) that have proposed~\citep{breiman2001random, fisher2019all, ribeiro2016should, bach2015pixel, van2006statistical, lei2018distribution,  zhang2020floodgate, wang2023total}, one of the most commonly leveraged in both industry and academic settings is the SHAP (Shapley additive explanation) framework~\citep{lundberg2017unified, lundberg2018consistent, lundberg2020local, lundberg2025shap, R-shapley}. 
This framework, which is based off of Shapley values from cooperative game theory~\citep{shapley1953value, roth1988introduction, winter2002shapley}, locally attributes credit to features in a way that satisfies four key axioms (which are noted in Section~\ref{sec:back}). In its model-agnostic formulation, given a random feature vector $X \in \R^d$, an outcome of interest $Y \in \R$, and a target feature $a \in [d]$, the SHAP curve $\phi_a(x)$ at a test point $x \in \R^d$ is defined as 
\begin{equation}
\label{eq:shap}
\phi_a(x) := \E_{X'}\Bigg[\underbrace{\frac{1}{d}\sum_{S \subset [d] : a \notin S}\binom{d - 1}{|S|}^{-1}\big\{\mu_0(x_{S \cup a}, X_{-S \cup a}') - \mu_0(x_{S}, X_{-S}')\big\}}_{= :\psi_a(x, X'; \mu_0)}\Bigg]
\end{equation}
where $\mu_0(x) := \E[Y \mid X = x]$ is the (unknown) regression function, the vector $(x_S , X_{-S})$ has $i$th coordinate $x_i$ if $i \in S$ and $X_i$ otherwise for $S \subset [d]$, and the expectation is over the distribution of covariates $X'$.  While the above formulation is complicated, the SHAP curve $\phi_a(x)$ can be understood as measuring the incremental contribution of a feature $a \in [d]$ to the regression $\mu_0(x)$ averaged over all feature permutations.

The SHAP framework, by nature, provides local explanations of feature importance. In many settings, learners may desire global measures of variable importance instead.
While there are many possible ways one could distill the SHAP curve into a global variable importance metric, the two most common approaches are to compute the mean absolute SHAP $\E|\phi_a(X)|$ and the mean squared SHAP $\E[\phi_a(X)^2]$~\citep{R-shapley, lundberg2025shap, keany2020borutashap,kraev2024shap, marcilio2020explanations, verhaeghe2023powershap, sebastian2024feature, hardt2021amazon, lundberg2018consistent, lundberg2020local, williamson2020efficient, dibaeinia2025interpretable, liu2022diagnosis, herbinger2024decomposing}.
In particular, these powers of SHAP are leveraged in algorithms for feature selection~\citep{marcilio2020explanations,dibaeinia2025interpretable,sebastian2024feature} and are commonly computed in open source software~\citep{R-shapley, lundberg2025shap, keany2020borutashap}. However, despite their ubiquity, there do not exists methods for constructing confidence intervals for $\E|\phi_a(X)|$ and $\E[\phi_a(X)^2]$. Further, if a learner tries to naively construct estimates of these powers by plugging in a model $\wh{\mu}$ predicting $\mu_0$ into $\psi_a(X, X'; \mu)$ outlined in Equation~\eqref{eq:shap} and take a sample average, first order biases will prevent the estimate from being asymptotically normal unless $\E\left[(\wh{\mu}(X) - \mu_0(X))^2\right]^{1/2} = o_\P(n^{-1/2})$. That is, the learner will need to estimate $\mu_0$ at unattainable, super-parametric rates. 

Given the importance of the SHAP curve $\phi_a(X)$, mean absolute SHAP, and mean squared SHAP in explainable ML, it is essential to develop estimators for these quantities. In this paper, we address this desideratum, showing that one can obtain asymptotically normal estimates for the $p$th power of the SHAP curve for $p \geq 1$ and high-fidelity estimates of the SHAP curve $\phi_a$ itself by combining elements from semi-parametric statistics, statistical learning theory, and analytical smoothing methods. 

\subsection{Our Contributions}
\label{sec:intro:contributions}
We present an estimator for  $\theta_p := \E|\phi_a(X)|^p$, the expected absolute $p$th power of the SHAP curve $\phi_a(X)$. 
Our estimator combines techniques related to de-biased inference on functionals of nested regressions, analytical smoothing of irregular functionals, and the theory of U-statistics to provide an asymptotically normal estimate of $\theta_p$.
We view the SHAP curve $\phi_a$ as a non-parametric nuisance that must be learned from data, and thus complement our main results by describing a de-biased loss that one can minimize over flexible classes of learners to estimate $\phi_a(X)$.
In more detail, our contributions are as follows.
\begin{enumerate}
    \item In Section~\ref{sec:shap:regular}, we describe an estimator for $\theta_p := \E|\phi_a(X)|^p$ in the setting $p \geq 2$.
    We start by describing a de-biased/Neyman orthogonal score for $\theta_p$ that depends on two independent samples $Z =(X, Y), Z'= (X', Y')$ from the data generating distribution. Our score takes the form
    \[
    m_p(Z, Z'; g) = |\phi(X)|^p + \underbrace{\gamma(X)\{\psi_a(X, X'; \mu) - \phi(X)\}}_{\text{De-biasing for $\phi_a$}} + \underbrace{\alpha(X')\{Y - \mu(X)\}}_{\text{De-biasing for $\mu_0$}},
    \]
    where $g = (\phi, \mu, \gamma, \alpha)$, $\phi(X)$ and $\mu(X)$ represent estimates of $\phi_a(X)$ and $\mu_0(X)$ respectively, and $\gamma(X)$, $\alpha(X)$ are estimates of ``correction'' terms, which we describe in detail later. 
    We then construct a U-statistic for $\theta_p$ by symmetrizing $m_p(Z,Z'; g)$ and prove asymptotic linearity (and hence normality).
    \item In Section~\ref{sec:shap:irregular}, we construct estimates for $\theta_p$ in the setting $1 \leq p < 2$. Here, $\theta_p$ is an \textit{irregular functional}, as the mapping $x \mapsto |x|^p$ is not generally twice differentiable (it is not even first differentiable when $p = 1$). To handle this, we propose replacing $\theta_p$ with  $\theta_{p, \beta} := \E\left[\varphi_{p, \beta}(\phi_a(X))\right]$,
    where $\varphi_{p, \beta}(u)$ is an appropriately-chosen smoothed surrogate. We then construct a de-biased score and U-statistic for the smoothed surrogate. We show that if $\beta > 0$ is carefully chosen, the U-statistic will be asymptotically normal about  $\theta_p$, thus allowing valid inference. In Appendix~\ref{app:clt}, we additionally  prove a central limit for smoothed U-statistics involving nuisance components, which may be of broader interest. 
    \item In Section~\ref{sec:learning}, we complement our inferential results by describing a statistical learning approach for estimating the entire SHAP curve $\phi_a(X)$. We start by describing an appropriate de-biased loss for $\phi_a(x)$ that again leverages two samples $Z, Z'$ from the data generating distribution. We then show that, by performing empirical risk minimization (ERM) according to this loss, one can obtain favorable  convergence guarantees that depend on the complexity of the underlying function class. We evaluate our risk minimization procedure in Appendix~\ref{app:experiments}.
\end{enumerate}


\subsection{Related Work}
\label{sec:intro:related}

We now briefly discuss related work. See Appendix~\ref{app:related} for a more detailed overview. The SHAP framework was developed by \citet{lundberg2017unified}, who use Shapley values \citep{shapley1953value} to define a variable importance metric satisfying desirable consistency properties. In followup work, \citet{lundberg2018consistent} describe efficient dynamic programming algorithms for computing SHAP on trees, and \citet{lundberg2020local} recommend computing the mean absolute SHAP to measure global variable importance. Many works since have used  used either the local SHAP curve \citep{yap2021verifying, mahajan2023development, aas2021explaining, peng2023research, poudel2024explaining} or global measures like the mean absolute SHAP \citep{keany2020borutashap, marcilio2020explanations, verhaeghe2023powershap, sebastian2024feature, hardt2021amazon, williamson2020efficient, dibaeinia2025interpretable, liu2022diagnosis, herbinger2024decomposing} to describe variable importance.
Open source software~\citep{R-shapley, lundberg2025shap,kraev2024shap} also gives users the option to compute both local and global SHAP values as well. 
However, these works do not provide a means of doing inference on $\phi_a(X)$ or powers of SHAP.

Some works do provide means of performing inference on either the SHAP curve itself or related global variable importance metrics. \citet{miftachov2024shapley} describe a method based on kernel smoothing for performing inference on $\phi_a(x_0)$, the value of the SHAP curve at a target point $x_0$. However, the authors obtain slower than $\sqrt{n}$-rates and require smoothness assumptions on $\phi_a$. \citet{morzywolek2025inference} study inference on the projection of $\phi_a$ onto a ball in RKHS's, but their results do not enable inference on/learning of the true curve. Closely related to our contributions are those of \citet{williamson2020efficient} and \citet{williamson2021nonparametric, williamson2023general}, who propose Shapley Population Variable Importance Metrics (SPVIMs) for measuring the global contributions of features to an outcome of interest $Y$. The SPVIM framework can be used to measure the contributions of features to quantities such as $R^2$ an classification accuracy, but does not apply to powers of the SHAP or local importance measures. Further, it is not possible to extend the estimator to our setting, since it would suffer bias from first-order nuisance estimation errors.

Lastly, our work is related to the long literature on semi-parametric inference and causal inference~\citep{kosorok2008introduction, bickel1993efficient, kennedy2016semiparametric, van2006targeted, chernozhukov2018double}. Most closely-related to our work are those that focus on performing semi-parametric inference on functionals of nested regressions~\citep{chernozhukov2022nested}, performing inference on functionals involving covariate shifts~\citep{chernozhukov2023automatic}, and works focusing on the statistical learning of heterogeneous treatment effects~\citep{foster2023orthogonal}.
Our work is also related to works that use either softmax smoothing~\citep{whitehouse2025inference,chen2023inference} or softplus smoothing~\citep{levis2023covariate, goldberg2014comment} for inference on irregular functionals. Our smoothed surrogate (which is based on $\tanh(x)$) can be viewed as contributing to this literature on irregular functionals, and may have broader applications.


%% file: background.tex
\section{Background}
\label{sec:back}
\paragraph{Notation}
Before presenting our main assumptions and results, we present notation we will use throughout the sequel. We let $[d] := \{1, 2, \dots, d\}$, and for any $S \subset [d]$ and $x, y \in \R^d$, we let $(x_S, y_{-S})$ be the vector defined as having $x_i$ in its $i$th coordinate if $i \in S$, and $y_i$ in its $i$th coordinate otherwise. Likewise, we let $x_S := (x_i : i \in S) \subset x$ denote the coordinates of $x$ having their indices in the set $S$. We let $\E[\cdots]$ and $\P(\cdots)$ denote expectation and probability over all sources of randomness. If there are independent random variables $U \in \calU$ and $V \in \calV$ with  distributions $P_U$ and $P_V$ respectively, and if $f : \calU \times \calV \rightarrow \R^p$ is a measurable function of $U$ and $V$, we let $E_V[f(U, V)] := \int_{\calV} f(U, v)P_V(dv)$ denote the expectation of $f(U, V)$ just with respect to the randomness in $V$. Likewise, we define for $q > 1$, $\|f\|_{L^q(P_V)} := \left(\int_{\calZ}\|f(U, v)\|_q^q P_V(dv)\right)^{1/q} = \left(\E_V\left[\|f(U, V)\|_q^q\right]\right)^{1/q}$ where $\|x\|_q := (\sum_i |x_i|^q)^{1/q}$. We define $\|f\|_{L^\infty(P_V)} := \inf\left\{t \geq 0 : \P_V\left(\|f(U, V)\|_\infty \geq t\right) = 0\right\}$ where $\|x\|_\infty := \max_i |x_i|$. 

Given a Banach space $B$,  a mapping $T : B \rightarrow \R^d$, and  $f^\ast, g \in B$, we define the $k$th Gateaux derivative of $T$ at $f^\ast$ in the direction $g$ as $D_f^k T(f^\ast)(g) := \frac{\partial^k}{\partial t^k}T(f^\ast + tg) \vert_{t = 0}$. If we have  a bi-variate map $T : B\times B \rightarrow \R^d$ and $f_1^\ast, f_2^\ast, g_1, g_2 \in B$, we let the cross Gateaux derivative be defined as $D_{f_1, f_2}T(f_1^\ast, f_2^\ast)(g_1, g_2) := \frac{\partial^2}{\partial s \partial t}T(f_1^\ast + tg_1, f_2^\ast + s g_2)\vert_{t = s = 0}$. For functions $f, g : \calX \rightarrow \R^d$, we define the intervals $[f, g] := \{\lambda f(x) + (1 - \lambda)g(x) : \lambda \in [0, 1]\}$. Lastly, we define the sign function as $\sgn(u) := \1\{u > 0\} - \1\{u < 0\}$.

\paragraph{Data Generating Processes and Assumptions}
We let $P_Z$ be a distribution over $Z = (X, Y) \in \calZ := \calX \times \calY$, where $X \in \calX \subset \R^d$ and $Y \in \calY \subset \R$. We assume $X$ admits a density $p(x)$ with respect to some dominating product measure $\nu$ on $\calX$. For $S \subset [d]$, we let $p_S(x_S)$ denote the marginal density of $X_S$, and $p_S(x_S \mid x_{-S})$ the conditional density of $x_S$ given $X_{-S} = x_{-S}$.  We let $\mu_0(x) := \E[Y \mid X = x]$ denote the regression of $Y$ onto covariates $X$. We assume the following about the data generating distribution.
\begin{ass}
\label{ass:dgp}
Let $Z := (X, Y)$, $Z' :=(X', Y') \sim P_Z$ be independent and $S \subset [d]$. First, we assume that $Y$ is bounded almost surely, i.e.\ $\|Y\|_{L^\infty(P_Z)} < \infty$. Second, we assume that the distributions of $(X_S, X'_{-S})$ and  $X$ are mutually absolutely continuous.\footnote{A measure $P$ is absolutely continuous with respect to $Q$ if $Q(E) = 0$ implies $P(E) = 0$ for any event $E$.}
Finally, we assume the likelihood ratio $\omega_S(X) \equiv \omega_S(X_S, X_{-S}) := \frac{p_S(X_S)p_{-S}(X_{-S})}{p(X_S, X_{-S})}$ has $\left\|\omega_S\right\|_{L^{2 + \epsilon}(P_X)} \leq C < \infty$ for some $C, \epsilon > 0$.
\end{ass}
The first assumption is standard in both causal inference and statistical learning, and can be typically relaxed to moment conditions on $Y$ at the cost of complicating analysis. The second assumption will 
hold under minor assumptions on the distribution of covariates $X$. In particular, it will hold if $\rho$ is the Lebesgue measure and $p(x) > 0$ for all $x \in \R^d$ or if $\rho$ is a mixed product of Lebesgue and counting measures, $\calX$ can be written as a cartesian product, and $p(x) >0$ on $\calX$. The assumption will generally fail when $(X_S, X'_{-S})$ and $X$ have differing supports. Lastly, the final assumption is a mild regularity assumption on the likelihood. 

\paragraph{Shapley Values and Shapley Additive Explanations} Shapley values, introduced in seminal work by \citet{shapley1953value}, provide a payout system for fairly attributing value to players in a \textit{coalitional game}. In more detail, a coalitional game consists of a set $[d]$ of players and a value function $v : 2^{[d]} \rightarrow \R$. The Shapley value for player $a \in [d]$ is defined to be 
\[
\phi_{a}(v) := \sum_{S \subset [d] :a \notin S}w(S)\left\{v(S \cup \{a\}) - v(S)\right\}, \;\text{where } w(S) := \frac{1}{d}\binom{d - 1}{|S|}^{-1}.
\]
Noting that $\sum_{a \notin S}w(S) = 1$, we see the weights describe a probability distribution, and hence we can write $\E_S$ for the expectation with respect to the $w(S)$'s. In words, the Shapley value measures the incremental gain in value obtained by including player $a$ in the coalition $S \subset [d] \setminus \{a\}$ averaged over all permutations of the $d$ players. A key result due to \citet{shapley1953value} shows that the Shapley values $\phi_1(v), \dots, \phi_d(v)$ are the unique payouts satisfying four key properties: \textit{efficiency}, \textit{symmetry}, the \textit{null player} property, and \textit{linearity}.\footnote{In more detail, the first property, called \textit{efficiency}, stipulates that a fair payout should sum to the total value of including all players minus the value of not including any players at all, i.e.\ $\sum_{a = 1}^d\phi_a(v) = v([d]) - v(\emptyset)$. The second axiom, \textit{symmetry}, states that if there are two players $a, b \in [d]$ such that $v(S \cup \{a\}) = v(S \cup \{b\})$ for all $S \subset [d] \setminus \{a, b\}$, then we should have $\varphi_{a}(v) = \varphi_b(v)$. The \textit{null player} property states that if there is a player $a \in [d]$ such that $v(S \cup \{a\}) = v(S)$ for all $S \subset [d]\setminus \{a\}$, then we should have $\phi_a(v) = 0$.
Lastly, \textit{linearity} states that if $v, w : 2^{[d]} \rightarrow \R$ are two value functions, $A, B \in \R$ are scalars, then for all $a \in [d]$ we should have $\phi_a(A v + B w) = A\phi_a(v) + B\phi_a(w)$.}

Authors  have observed that aforementioned four properties of Shapley values are naturally desirable in VIMs. In particular, by carefully choosing the value function $v$, one can arrive at metrics that fit an application at hand. For instance, \citet{owen2017shapley} take the value function as $v(S) := \Var\left[\E(\mu_0(X) \mid X_s)\right]$ in an attempt to circumvent some shortcomings of the ANOVA decomposition. Likewise, \citet{covert2020understanding} and \citet{williamson2023general} introduce SAGE and SPVIMs respectively, which take the value function to be $v(S) := \max_{f \in \calF_S}\calV(f; P_Z)$, where $\calV$ is some measure of function predictiveness ($R^2$, area under ROC curve, classification accuracy) and $\calF_S$ is a class of functions only depending on coordinates in $S$. 

The most commonly-used VIMs defined in terms of Shapley values, and the one considered within this paper, is the Shapley additive explanation, or SHAP, framework~\citep{lundberg2017unified}. In its modern instantiation, the value function (which depends on a target point $x \in \R^d$) is taken to be $v_x(S) := \E[\wh{\mu}(x_S, X_{-S})]$ in model-specific settings and $v_x(S) := \E[\mu_0(x_S, X_{-S})]$ in model-agnostic settings. In our model-agnostic setting, one can check that defining $\phi_a(x) := \phi_a(v_x)$ for the latter choice of $v_x$ is equivalent to defining the SHAP curve $\phi_a(x)$ in terms of Equation~\eqref{eq:shap}. Since the population regression $\mu_0$ is unknown and must be estimated from data, it is helpful to define a flexible score that can depend on any regression estimate $\mu : \calX \rightarrow \R$. As in Equation~\eqref{eq:shap}, we define $\psi_a(X, X'; \mu)$ by $\psi_a(X, X'; \mu) := \frac{1}{d}\sum_{\substack{S \subset [d] \setminus a}}\binom{d - 1}{|S|}\left\{\mu(X_{S \cup a}, X'_{-S \cup a}) - \mu_0(X_S, X'_{-S})\right\}$.
Note that the expectation of $\psi_a$ over $X'$ is precisely the SHAP curve $\phi_a(X)$, i.e.\ we have $\phi_a(X) = \E_{X'}[\psi_a(X, X'; \mu_0)]$.

\begin{remark}
\label{rmk:shap_defn}
We note that the above choice of value function used in defining the SHAP curve $\phi_a(X)$ is not the only one to have been considered in the broader literature. For instance, in the original development of SHAP, the value function for a target $x \in \R^d$ and $S \subset [d]$ was taken to be $v_x(S) := \E[\mu_0(x_S, X_{-S}) \mid X_S = x_s]$, i.e.\ a conditional expectation vs.\ a marginal one. However, in more recent works, the marginal perspective has proven dominant over the original, conditional one. \citet{janzing2020feature} note that the marginal expectation possess the desirable interpretation of counter-factually setting the variables $X_S$ to a value $x_S \in \R^{|S|}$ via a ``do operation'' in Pearl's do-calculus~\citep{pearl2014probabilistic}. Further, in \citet{lundberg2020local}, the authors of the original SHAP paper advocate adopted the marginal value function due to its causal interpretation. Further, the most commonly used open source software packages for computing SHAP leverage the marginal value function~\citep{R-shapley, lundberg2025shap}.  
\end{remark}


%% file: shap_powers.tex
\section{Inference on Power of SHAP}
\label{sec:shap}

We now present estimators based on de-biased U-statistics for the expected $p$th absolute power of the SHAP curve, $\theta_p := \E|\phi_a(X)|^p$. As noted in the introduction, reporting powers of SHAP provides a convenient way of converting an inherently local of measure of variable importance, the SHAP curve, into a global one, and thus are regularly used in practice.  While the approaches discussed below could naturally generalize to other functionals of the SHAP curve (e.g.\ expected SHAP, variance of SHAP) we restrict ourselves to powers of SHAP not only due to the aforementioned practical importance, but also to simplify exposition.

First, we consider the setting $p \geq 2$, for which we present an asymptotically linear estimator based on Neyman orthogonal U-statistics for nested regressions.
Next, we consider the more challenging setting of $1 \leq p < 2$, in which the target functional $\theta_p = \E|\phi_a(X)|^p$ is not twice differentiable at $\phi_a(X) = 0$ . To work around non-differentiability, we present a smoothed alternative to our U-statistic that replaces the $p$th power with an appropriately scaled hyperbolic tangent function. We show how to tune the smoothed parameter so that the smoothed estimate converges appropriately to the true, unsmooth power of SHAP. 


\subsection{Inference when $p \geq 2$}
\label{sec:shap:regular}

We start by describing a de-biased estimator for $\theta_p := \E|\phi_a(X)|^p$ in the setting $p \geq 2$. We motivate our estimator by discussing the failure of a naive, plug-estimate estimate for $\theta_p$. Suppose the learner is given a sample $Z_1, \dots, Z_n \sim P_Z$ of i.i.d.\ observations and an independent estimate $\wh{\phi}(x)$ of the SHAP curve $\phi_a(x)$ (say, constructed from the loss outlined in Section~\ref{sec:learning} below). Then, a seemingly reasonable estimate of $\theta_p$ would be the following plug-in estimate: $\wh{\theta}_n^{\naive} := \frac{1}{n}\sum_{i = 1}^n \big|\wh{\phi}(X)\big|^p$.

However, the naive plug-in estimate is highly sensitive to first-order errors in $\wh{\phi}$. More precisely, in order to show asymptotic normality of the scaled difference $\sqrt{n}(\wh{\theta}_n^{\naive} - \theta_p)$, one would need super-parametric convergence rates for the SHAP curve estimate, i.e.\ $\big\|\wh{\phi} - \phi_a\big\|_{L^2(P_X)} = o_\P(n^{-1/2})$. This rate of convergence is generally impossible even under parametric assumptions. The standard semi-parametric approach for converting naive plug-in estimators into more robust ones is to ``de-bias'' the plug by subtracting a linear correction term~\citep{chernozhukov2018double}. In particular, this correction term, when evaluated at the true nuisances, should exactly cancel the Gateaux derivative of the naive estimator with respect to the nuisance component (here $\phi_a$). This is formalized through the concept of \textit{Neyman Orthogonality}.
\begin{definition}
\label{def:neyman}
Let $\Theta$ and $\calG$ be spaces of functions (e.g.\ $L^\infty(P_X)$ or $L^2(P_X)$). Let $g_0 \in \calG$ denote a true nuisance function and $M : \calG \rightarrow \R$ denote a functional. We say $M$ is Neyman orthogonal if $D_g \E[M(g_0)](\Delta_g) = 0$ for all $\Delta_g \in \calG - g_0$.
\end{definition}

In our setting, because $\phi_a$ is specified as a nested regression involving nuisance (here, the nested nuisance is the regression $\mu_0$) the first-order correction with respect to $\phi$ is itself sensitive to misestimation $\mu_0$. Thus, we actually need two correction terms in order  to cancel all first order nuisance estimation errors. We arrive the following Neyman orthogonal two sample score for $\theta_p$. We prove the following in Appendix~\ref{app:shap_proofs}

\begin{prop}
\label{prop:ortho_regular}
Let $p \geq 2$ and $Z = (X, Y), Z'= (X', Y') \sim P_Z$ be independent samples obeying Assumption~\ref{ass:dgp}, and let $g = (\phi, \mu, \gamma, \alpha)$ denote a vector functions of $X$ where $\phi, \mu, \gamma \in L^\infty(P_X)$ and $\alpha \in L^2(P_X)$. Define the two-sample score $m_p(Z, Z'; g)$ by
\[
m_p(Z, Z'; g) := |\phi_p(X)|^p +   \gamma(X)\left\{\psi_0(X, X'; \mu) - \phi(X)\right\} + \alpha(X')\left\{Y' - \mu(X')\right\}.
\]
Further, define the population functional $M_p(g) := \E[m_p(Z, Z'; g)]$ with true nuisances  $g_p = (\phi_a, \mu_0, \gamma_p, \alpha_p)$, with  $\gamma_p(X) := p\sgn(\phi_a(X))|\phi_a(X)|^{p - 1}$ and  
\begin{align*}
\alpha_p(X') := \frac{1}{d}\sum_{S \subset [d]: a \notin S}\binom{d - 1}{|S|}^{-1}\Big\{\gamma_p^{S \cup a}(X_{S \cup a}') \omega_{S \cup a}(X') - \gamma_p^S(X_S')\omega_S(X')\Big\},
\end{align*}
where $\gamma_p^S(X_S) := \E(\gamma_p(X) \mid X_S)$ and $\omega_S$ is as above. Then, $\theta_p = M(g_0)$ and $M(g)$ is Neyman orthogonal.
\end{prop}

Note that since we assume $Y$ is bounded (Assumption~\ref{ass:dgp}), $\gamma_p, \mu_0$, and $\phi_a$ all live in $L^\infty(P_X)$. Likewise, $\alpha_p \in L^2(P_X)$ since each likelihood ratio $\omega_S(X)$ has finite $(2 + \epsilon)$th moment. Thus, we can restrict our attention for estimates $\mu, \phi, \gamma \in L^\infty(P_X)$ and $\alpha \in L^2(P_X)$. Proposition~\ref{prop:ortho_regular} suggests a natural estimator for $\theta_p$ in terms of the symmetrized score $
h_p(Z, Z'; g) := \frac{1}{2}\left\{m_p(Z, Z'; g) + m_p(Z', Z; g)\right\}$. 
Given an estimate $\wh{g}$ of the nuisance vector $g_p$ and an independent sample $Z_1, \dots, Z_n$, we compute the $U$-statistic
\begin{equation}
\label{eq:u_stat_regular}
\wh{\theta}_{p} := \binom{n}{2}^{-1}\sum_{i < j}h_p(Z_i, Z_j; \wh{g}).
\end{equation}
Under standard, non-parametric convergence rates for the nuisance estimates, we can show that the U-statistic $\wh{\theta}_p$ is asymptotically linear about $\theta_p$, thus implying asymptotic normality when the influence function is non-degenerate. This is formalized in the following theorem, which we prove in Appendix~\ref{app:shap_proofs}.

\begin{theorem}
\label{thm:regular}
Let $p \geq 2$, let $Z_1, \dots, Z_n \sim P_Z$ be i.i.d.\ satisfying Assumption~\ref{ass:dgp}, and suppose $\wh{\phi}_n, \wh{\mu}_n, \wh{\alpha}_n$, and $\wh{\gamma}_n$ are estimates of $\phi_a, \mu_0, \alpha_{p},$ and $\gamma_{p}$ that are independent of $Z_1, \dots, Z_n$. Suppose the following hold.
\begin{enumerate}
    \item \textit{(Nuisance Convergence)} All nuisance estimates are consistent in $L^2(P_X)$, i.e.\ that $\|\wh{g} - g_p\|_{L^2(P_X)} = o_\P(1)$. Further, we have $\|\wh{\phi}_n - \phi_a\|_{L^2(P_X)} = o_\P(n^{-1/4})$, $\|\wh{\alpha} - \alpha_{p}\|_{L^2(P_X)}\|\wh{\mu} - \mu_0\|_{L^2(P_X)} = o_\P(n^{-1/2})$, and $\|\wh{\phi} - \phi_a\|_{L^2(P_X)}\|\wh{\gamma} - \gamma_p\|_{L^2(P_X)} = o_\P(n^{-1/2})$.
    Lastly, for any $S \subset [d]$, defining $Q_{S, X} := P_{X_S} \otimes P_{X_{-S}}$, we assume $\|\wh{\gamma} - \gamma_p\|_{L^2(P_X)}\|\wh{\mu} - \mu_0\|_{L^2(Q_{S, X})} = o_\P(n^{-1/2})$.
    \item \textit{(Moment Assumptions)} We assume that there is an absolute constant $D > 0$  such that $|\wh{\mu}(X)|$, $|\wh{\phi}(X)|$, $|\wh{\gamma}(X)| \leq D$ almost surely. Further, we assume $\|\wh{\alpha}\|_{L^{2 + \epsilon}(P_X)}  \leq D$ as well.
\end{enumerate}
Then, then estimator $\wh{\theta}_n$ outlined in Equation~\eqref{eq:u_stat_regular} is asymptotically linear, i.e.\ we have 
\[
\sqrt{n}(\wh{\theta}_n - \theta_p) = \frac{1}{\sqrt{n}}\sum_{i = 1}^n\Big\{\underbrace{|\phi_a(X_i)|^p + \alpha_p(X_i)(Y_i - \mu_0(X_i)) + \Lambda_p(X_i) - (p + 1)\theta_p}_{=: \rho_{\theta_p}(Z_i)}\Big\} + o_\P(1),
\]
 where $\Lambda_p(x) := \E[\gamma_p(X')\psi_a(X', X; \mu_0) \mid X = x]$.
Further, if $\sigma^2_p := \Var[\rho_{\theta, p}(Z)] > 0$, then we have $\sqrt{n}(\wh{\theta}_n - \theta_p) \Rightarrow \calN(0, \sigma_p^2)$. Thus, if $\wh{\sigma}_n^2$ is a consistent estimate for $\sigma_p^2$ and $z_{\alpha/2}$ is the $1  - \frac{\alpha}{2}$ quantile of the standard normal distribution, we have that
\[
C_{1 - \alpha} := \left[\wh{\theta}_n - \frac{\wh{\sigma}_n}{\sqrt{n}}z_{\alpha/2}, \wh{\theta}_n + \frac{\wh{\sigma}_n}{\sqrt{n}}z_{\alpha/2}\right]
\]
is a asymptotic $1 - \alpha$ confidence interval for $\theta_p$.
\end{theorem}

\begin{remark}
\label{rmk:alternative_rates}
We note that, in our assumptions on nuisance estimation rates above, we assume asymptotic control of the $L^2(Q_{S, X})$-norm of the error $\wh{\mu} -\mu_0$, where $Q_{S, X} := P_{X_S} \otimes P_{X_{-S}}$ denotes the distribution of the vector $(X_S, X'_{-S})$. This assumption is natural, as the score $\psi_a(X, X'; \mu)$ naturally consists of terms of the form $\mu(X_S, X'_{-S})$. 
If we are willing to make the stronger assumption that the likelihood ratio $\omega_S(X_S, X_{-S}) = \frac{p_S(X_S)p_{-S}(X_{-S})}{p(X_S, X_{-S})}$ is almost surely bounded for each $S \subset [d]$, we can reduce our nuisance estimation rate to the standard $\|\wh{\gamma} - \gamma_p\|_{L^2(P_X)}\|\wh{\mu} - \mu_0\|_{L^2(P_X)} = o_\P(n^{-1/2})$. Alternatively, if one is willing to assume an existence of a $q$th moment for the likelihood ratio, i.e.\ $\|\omega_S(X)\|_{L^q(P_X)} < \infty$ and we let $h := \frac{1}{1 - \frac{1}{q}}$ denote the Holder conjugate of $q$, then another sufficient nuisance estimation rate is $\|\wh{\gamma} - \gamma_p\|_{L^2(P_X)}\|\wh{\mu} - \mu_0\|_{L^{2h}(P_X)} = o_\P(n^{-1/2})$. For instance, since we know $\|\omega_S\|_{L^2(P_X)} < \infty$ in Assumption~\ref{ass:dgp}, it is sufficient that $\|\wh{\gamma} - \gamma_0\|_{L^2(P_X)}\|\wh{\mu} - \mu_0\|_{L^4(P_X)} = o_\P(n^{-1/2})$.\footnote{Nuisance estimation assumptions in terms of $L^4$ norm are common in many causal works. See \citet{foster2023orthogonal} for the use of $L^4$ rates in the setting of statistical learning or \citet{morzywolek2025inference} for the use of such rates in performing inference on the projection of the SHAP curve onto balls in reproducing kernel Hilbert spaces (RKHS's).} 
\end{remark}

\subsection{Inference When $1 \leq p < 2$}
\label{sec:shap:irregular}

Next, we consider the setting where $p \in [1, 2)$. In this setting, the functional $\phi \mapsto \E|\phi(X)|^p$ is not twice Gateaux differentiable, and hence the conclusions of Theorem~\ref{thm:regular}  cannot be directly applied. Further, when $p = 1$ (the most common setting in practice), the functional is not even differentiable. If one were willing to assume a strict margin condition of the form $|\phi_a(X)| \geq \eta$  for some $\eta > 0$ (thus pushing the SHAP curve away from the region of non-differentiability), then it may be possible to salvage the theory from the preceding subsection. However, these assumptions would be unrealistic, as any given covariate may not be predictive for some members of the population, i.e.\ we may expect $\P(\phi_a(X) = 0) > 0$ in general.

Instead of making such a strong margin assumption, our goal is to define an estimator that yields asymptotic linearity (and hence, normality) while (a) allowing feature irrelevance, i.e.\ $\P(\phi_a(X) = 0) > 0$ and (b) making lightweight assumptions about the behavior of $\phi_a(X)$ in a neighborhood of zero. In particular we make the following density assumption, which has previously been leveraged in the literature on policy learning and performing inference on the value of optimal treatment regimes~\citep{whitehouse2025inference, luedtke2016statistical}. It assumes $|\phi_a(X)|$ admits a density in a small neighborhood near but excluding zero. In particular, it allows the density to tend towards infinity near zero.

\begin{ass}
\label{ass:margin}
We assume there are constants $\delta, H, c > 0$ such that $U: = |\phi_a(X)|$ admits a Lebesgue density $f_U(t)$ on $(0, c)$  satisfying $f_U(t) \leq H t^{\delta - 1}$ for all $t \in (0, c)$.
\end{ass}

To work around non-differentiability, we propose replacing the poorly-behaved function $x \mapsto |x|^p$ by a smoother, twice-differentiable alternative. Recalling that the hyperbolic tangent function is defined as $\tanh(x) := \frac{e^x - e^{-x}}{e^x + e^{-x}}$, for any  $\beta > 0$, we define the smoothed alternative $\theta_{p, \beta}$ to $\theta_p$ as 
\[
\theta_{p, \beta} := \E\left[\varphi_{p, \beta}(\phi_a(X))\right]\quad  \text{where}\quad\varphi_{p, \beta}(x) := |x|^p \tanh(\beta |x|^{2 - p}).
\]
Importantly, note that (a) $\varphi_{p, \beta}(0) = \varphi_p(0) = 0$, i.e.\ there is no bias at zero and (b) for a fixed $\beta > 0$, the quality of the approximation $\varphi_{p, \beta}(x) \approx \varphi_p(x)$ improves the larger $|x|$ gets. For any fixed $\beta > 0$, we can construct a Neyman orthogonal score, which we detail in the following proposition. In the following, we refer to the first derivative $\varphi_{p, \beta}'$ of $\varphi_{p, \beta}$, which is given as
\begin{equation}
\varphi_{p, \beta}'(u) = \sgn(u)\bigg[p|u|^{p - 1}\tanh(\beta |u|^{2- p}) + \beta(2 - p)|u|\sech(\beta |u|^{2 - p})\bigg].
\end{equation}

\begin{prop}
\label{prop:ortho_irregular}
Let $\beta > 0$ and $p \in [1, 2)$ be arbitrary, and suppose $Z, Z' \sim P_Z$ are two independent samples. Letting $g = (\phi, \mu, \gamma, \alpha)$, define the score $m_p^\beta(Z, Z'; g)$ by
\[
m_p^\beta(Z, Z'; g) := \varphi_{p, \beta}(\phi(X)) + \gamma(X)\{\psi_a(X, X'; \mu) - \phi(X)\} + \alpha(X')\{Y' - \mu(X')\}
\]
and $M_p^\beta(g) := \E[m_p^\beta(Z, Z'; g)]$. Define the representers by $\gamma_{p, \beta}(X) := \varphi'_{p, \beta}(\phi_a(X))$ and
\[
\alpha_{p, \beta}(X') := \frac{1}{d}\sum_{S \subset [d] : a \notin S}\binom{d - 1}{|S|}^{-1}\left\{\gamma_{p, \beta}^{S \cup a}(X'_{S \cup a})\omega_{S \cup a}(X') - \gamma_{p, \beta}^S(X'_S)\omega_S(X')\right\},
\]
where, for any $S \subset [d]$, $\gamma_{p, \beta}^S(x_S) := \E(\gamma_{p, \beta}(X) \mid X_S = x_S)$. Then, letting $g_{p, \beta} := (\phi_a, \mu_0, \gamma_{p, \beta}, \alpha_{p, \beta})$, we have (a) $\theta_{p, \beta} = M_p^\beta(g_{p, \beta})$ and
(b) $M_p^\beta(g)$ is Neyman orthogonal.
\end{prop}

As before, Proposition~\ref{prop:ortho_irregular} suggests a natural estimator for the smoothed parameter $\theta_{p, \beta}$. Namely, we again define the symmetrized score $h_p^\beta(Z, Z'; g) := \frac{1}{2}\left\{m_p^\beta(Z, Z'; g) + m_p^\beta(Z', Z; g)\right\}$ and construct the smoothed U-statistic as $\wh{\theta}_n :=  \binom{n}{2}^{-1}\sum_{i < j}h_p(Z_i, Z_j; \wh{g})$, where $Z_1, \dots, Z_n \sim P_Z$ is a sample of i.i.d.\ data and $\wh{g}$ is a nuisance estimate assumed to have been constructed on an independent sample.

However, if one naively uses the above U-statistic, under the true nuisances the estimator will be unbiased for \emph{the smoothed alternative} $\theta_{p, \beta} = \E\varphi_{p, \beta}(\phi_a(X))$, not  $\theta_p = \E|\phi_a(X)|^p$. In order to do inference on $\theta_p$, we need to study $\Bias(\beta) := |\theta_p - \theta_{p, \beta}|$ as a function of $\beta > 0$. Under the density assumption made above (Assumption~\ref{ass:margin}) we can show the following bound (see Appendix~\ref{app:shap_proofs} for a proof).

\begin{lemma}
\label{lem:bias}
Under Assumption~\ref{ass:margin}, for sufficiently large $\beta > 0$, we have $\Bias(\beta) := \big|\theta_{p, \beta} - \theta_p\big| \lesssim  \left(\frac{1}{\beta}\right)^{\frac{p + \delta}{2 - p}}$.
\end{lemma}
Given a sample size $n$, how should one pick $\beta_n$ in order to obtain asymptotic normality of $\wh{\theta}_n$ around the un-smoothed parameter $\theta_p$? Recalling that $\theta_{p, \beta} = M_p^\beta(g_{p, \beta})$, note that we can write 
\[
\sqrt{n}(\wh{\theta}_n - \theta_p) = \underbrace{\sqrt{n}(\wh{\theta}_n - M_p^\beta(\wh{g}))}_{\text{Asymptotically Linear Term}} + \underbrace{\sqrt{n}(M_p^\beta(\wh{g}) - M_p^\beta(g_{p, \beta}))}_{\text{Remainder Term}} + \underbrace{\sqrt{n}(\theta_{p, \beta} - \theta_p)}_{\text{Smoothing Bias Term}}
\]
it is clear to obtain asymptotic linearity the learner needs to ensure the smoothing bias term vanishes, i.e.\ that $\sqrt{n}\Bias(\beta_n) = \sqrt{n}|\theta_{p, \beta} - \theta_p| = o(1)$. For this to hold, per Lemma~\ref{lem:bias}, the learner must take $\beta_n = \omega\big(n^{\frac{2 - p}{2(p + \delta)}}\big)$. Further, it is also clear that the remainder term must vanish as well. One can show (Lemma~\ref{lem:smooth} in Appendix~\ref{app:shap_proofs}) that $\varphi_{p, \beta}''(u) \leq C_p \beta$, where $C_p$ is some absolute constant depending just on $p$. A second order Taylor expansion with mean value remainder yields that, for some $\wb{g}_n \in [\wh{g}_n, g_{p, \beta}]$:
\[
\sqrt{n}(M_p^\beta(\wh{g}) - M_p^\beta(g_{p, \beta})) = \underbrace{D_g M_p^\beta(g_{p, \beta})(\wh{g}_n - g_{p, \beta})}_{=0 \text{ by Proposition~\ref{prop:ortho_irregular}}} + D_g^2 M_p^\beta(\wb{g}_n)(\wh{g}_n - g_{p, \beta}, \wh{g}_n - g_{p, \beta}).
\]
We can show that the second derivative is generally bounded above by $\beta_n\|\wh{g}_n - g_{p,\beta}\|_{L^2(P_X)}$. Thus, to avoid assuming super-parametric convergence rates, one must also select $\beta_n = o(n^{1/2})$. This heuristic sketch is formalized by the following theorem and its proof, which can be found in Appendix~\ref{app:shap_proofs}.

\begin{theorem}
\label{thm:irregular} 
Let $1 \leq p < 2$, suppose $Z_1, \dots, Z_n \sim P_Z$ are i.i.d.\ satisfying Assumption~\ref{ass:dgp}, suppose $|\phi_a(X)|$ satisfies Assumption~\ref{ass:margin} with constant $\delta > 0$, and suppose $\beta_n = \omega\big(n^{\frac{2 - p}{2(p + \delta)}}\big)$. Further, suppose $\wh{\phi}_n, \wh{\mu}_n, \wh{\alpha}_n$, and $\wh{\gamma}_n$ are estimates of $\phi_a, \mu_0, \gamma_{p, \beta}$, and $\alpha_{p, \beta}$ that are independent of the sample. Assume the following hold:
\begin{enumerate}
    \item \textit{(Nuisance Convergence)} All nuisance estimates are consistent in $L^2(P_X)$, i.e.\ that $\|\wh{g} - g_p\|_{L^2(P_X)} = o_\P(1)$. Further, we have $\|\wh{\phi}_n - \phi_a\|_{L^2(P_X)} = o_\P(\beta_n^{-1/2}n^{-1/4})$, $\|\wh{\alpha} - \alpha_{p, \beta}\|_{L^2(P_X)}\|\wh{\mu} - \mu_0\|_{L^2(P_X)} = o_\P(n^{-1/2})$, and $\|\wh{\phi} - \phi_a\|_{L^2(P_X)}\|\wh{\gamma} - \gamma_{p,\beta}\|_{L^2(P_X)} = o_\P(n^{-1/2})$.
    Lastly, for any $S \subset [d]$, we assume $\|\wh{\gamma} - \gamma_{p, \beta}\|_{L^2(P_X)}\|\wh{\mu} - \mu_0\|_{L^2(Q_{S, X})} = o_\P(n^{-1/2})$.
    \item \textit{(Moment Assumptions)} We assume that there is an absolute constant $D > 0$  such that $|\wh{\mu}(X)|$, $|\wh{\phi}(X)|$, $|\wh{\gamma}(X)| \leq D$ almost surely. Further, we assume $\|\wh{\alpha}\|_{L^{2 + \epsilon}(P_X)} \leq D$ as well, where $\epsilon > 0$ is as in Assumption~\ref{ass:dgp}.
\end{enumerate}
Then, defining $h_p^\beta(z, z'; g) := \frac{1}{2}\left\{m_p^\beta(z, z'; g) + m_p^\beta(z', z; g)\right\}$ and setting $\wh{\theta}_n :=  \binom{n}{2}^{-1}\sum_{i < j}h_p^\beta(Z_i, Z_j; \wh{g}_n)$, we have
 $\sqrt{n}(\wh{\theta}_n - \theta_p) = \frac{1}{\sqrt{n}}\sum_{i = 1}^n\rho_{\theta_p}(Z_i) + o_\P(1)$, where $\rho_{\theta_p}(Z)$ is as outlined in Theorem~\ref{thm:regular}. Consequently, if $\sigma^2_p := \Var[\rho_{\theta, p}(Z)] > 0$, then the asymptotic normality result and $(1 - \alpha)$-confidence interval outlined in Theorem~\ref{thm:regular} are also valid when $1 \leq p < 2$ as well.
\end{theorem}

\begin{remark}
\label{rmk:irrelevant}
We note that, in the above theorem (as well as in Theorem~\ref{thm:regular}), we made the assumption that $\sigma_p^2 := \Var[\rho_{\theta_p}(Z)] > 0$. In the case the $a$th feature $X_a$ is irrelevant to $\mu_0(X)$, we will have \ $\phi_a(X) = 0$ (which follows by the \textit{null player property} of the Shapley value), and consequently $\sigma_p^2 = 0$. Some other works that consider hypothesis testing for zero importance (such as \citet{morzywolek2025inference}) leverage the bootstrap to construct tests in this degenerate setting, and others assume non-degeneracy  in order to obtain asymptotic normality. We note that by simply clipping the variance from below, one can still maintain valid (albeit sometimes conservative) coverage.

In more detail, while the asymptotic linearity result is still valid in the degenerate setting, when $\phi_a(X) = 0$ a.s., we have $\sqrt{n}(\wh{\theta} - \theta_p) = o_\P(1)$. Thus, asymptotic normality fails. Although one may still hope that the confidence interval outlined above is valid, any consistent variance estimate satisfies $\wh{\sigma}_n^2 = o_\P(1)$ as well, and thus arguing coverage is non-trivial. Instead, note that if one defines $\wt{\sigma}^2 := \wh{\sigma}^2 \land c$ for some $c > 0$ and consistent $\wh{\sigma}$, then $C_{1 - \alpha}^{\trunc} := \left[\wh{\theta}_p \pm n^{-1/2}\wt{\sigma}z_{\alpha/2}\right]$ is a valid $(1 - \alpha)$-confidence interval. If $\sigma^2_p \geq c$, this interval will asymptotically coincide with the one outlined in Theorem~\ref{thm:irregular}. If $0 < \sigma^2_p < c$, the interval will be conservative. If $\phi_a(X) = 0$ with probability one, we trivially get asymptotically-perfect coverage.
\end{remark}

%% file: learning_revised.tex
\section{Statistical Learning for the SHAP Curve}
\label{sec:learning}
In the previous sections, we analyzed an estimator based on Neyman orthogonal U-statistics for performing inference of the $p$th powers of SHAP. In our main theorems (Theorems~\ref{thm:regular} and \ref{thm:irregular}) we assumed that the learner had access to a high-quality estimate $\wh{\phi}$ of the SHAP curve $\phi_a(X) := \frac{1}{d}\sum_{a \notin S}\binom{d - 1}{|S|}^{-1}\{\E[\mu_0(x_{S \cup a}, X'_{-S \cup a})] - \E[\mu_0(x_S, X'_{-S})]\} = \E[\psi_a(X, X'; \mu_0) \mid X = x]$. Because $\phi_a$ naturally depends on the unknown regression $\mu_0$, it is unclear how to use off-the-shelf learners to produce estimates $\wh{\phi}$. 

In this section, we describe an approach for estimating $\phi_a(X)$ based on the orthogonal statistical learning framework of \citet{foster2023orthogonal}. In Proposition~\ref{prop:ortho_loss}, we describe a de-biased (i.e.\ Neyman orthogonal) loss function $L(\phi; g)$ for the SHAP curve $\phi_a(X)$, where $g$ represents a vector of nuisances that must be learned from data. Then, in Theorem~\ref{thm:excess_risk}, we show that by performing empirical risk minimization (ERM) with respect to $L$, the learner can produce an estimate $\wh{\phi}$ that enjoys favorable excess risk guarantees. We start by defining \textit{Neyman orthogonal loss functions}.

\begin{definition}[Assumption 1 of \citet{foster2023orthogonal}]
\label{def:neyman_ortho_loss}
Let $\Theta$ and $\calG$ be spaces of functions. Let $g_0 \in \calG$ denote a true nuisance function, and $\theta_0 = \arg\min_{\theta \in \Theta}L(\theta, g_0)$ a true loss minimzer. We say a population loss $L : \Theta \times \calG \rightarrow \R$ is \textit{Neyman orthogonal} if, for all $\Delta_\phi \in \Theta - \theta_0$ and $\Delta_g \in \calG - g_0$, $D_{g, \theta}L(\theta_0, g_0)(\Delta_\phi, \Delta_g) = 0$.
\end{definition}

Before presenting our de-biased loss, we discuss a naively-constructed loss that is inherently sensitive to misestimation in the regression $\mu_0(x) := \E[Y \mid X = x]$. Since we have the identity $\phi_a(X) = \E[\psi_a(X, X'; \mu_0) \mid X]$, perhaps the most immediate loss for $\phi_a$ is the two-sample loss $\ell^{\naive}(x, x'; \phi, \mu) := \frac{1}{2}\big(\phi(x) - \psi_a(x, x'; \mu)\big)^2$. One can check that $\phi_a = \arg\min_{\phi : \calX \rightarrow \R} L^{\naive}(\phi; \mu_0)$, where $L^{\naive}(\phi; \mu) = \E_{X, X'}\left[\ell^{\naive}(X, X'; \phi, \mu)\right]$. However, the population loss $L^{\naive}$ does not satisfy Definition~\ref{def:neyman_ortho_loss}. 

Since the naive loss is not orthogonal, the next best approach would be to compute the cross-derivative $D_{\mu, \phi}L^{\naive}(\phi_a; \mu_0)$ and then subtract off a linear correction term, thus ensuring a vanishing derivative. We now do a heuristic derivation of an orthogonal loss --- a rigorous derivation is provided in the proof of Proposition~\ref{prop:ortho_loss} in Appendix~\ref{app:learning}. If $\Delta_\phi = \phi - \phi_a$ and $\Delta_\mu = \mu - \mu_0$ for some $\phi, \mu : \calX \rightarrow \R$, we have
\begin{align*}
&D_{\phi, \mu}L^{\naive}(\phi_a; \mu_0)(\Delta_\phi, \Delta_\mu) = -\E\left[\Delta_\phi(X)\psi_a\big(X, X'; \Delta_\mu\big)\right] = -\sum_{S \subset [d]}w(S) \sigma(S)\E\left[\Delta_\phi(X)\Delta_\mu(X_S, X_{-S}')\right],
\end{align*}
where $w(S) := \frac{1}{d}\binom{d - 1}{|S - a|}$ and $\sigma(S) = 2\1\{a \in S\} - 1$. Now, defining $\zeta_0^S(X, X') := \frac{p(X_S, X_{-S}')}{p(X_S, X_{-S})}\frac{p_{-S}(X_{-S})}{p_{-S}(X'_{-S})}$ $\forall S$, we have:
\begin{align*}
&\E_{X, X'}\left[\Delta_\phi(X)\Delta_\mu(X_S, X'_{-S})\right] = \E_{X, X'}\left[\E\left(\Delta_\phi(X) \mid X_S\right) \cdot \E\left(\Delta_\mu(X_S, X'_{-S}) \mid X_S\right)\right] \\
&\qquad = \E_{X, X'}\left[\Delta_\phi(X_S, X'_{-S})\Delta_\mu(X)\zeta_0^S(X, X')\right] = \E_{X, X'}\left[\Delta_\phi(X_S, X'_{-S})\beta_0^S(X, X')\{\mu(X) - Y\}\right],
\end{align*}
where the first equality follows from the tower rule and independence of $X$ and $X'$ and the second inequality follows from a change of measure. Thus, to arrive at a de-biased loss, we should consider the corrected population loss $L^{\corr}(\phi; \mu, \zeta) := \E_{Z, Z'}\left[\ell^{\corr}(Z, Z'; \phi, \mu, \zeta)\right]$, where $\zeta = (\zeta^S : S \subset [d])$ and $\ell^{\corr}$ is given by
\begin{align*}
&\ell^{\corr}(Z, Z'; \phi, \mu, \zeta) := \frac{1}{2}\left(\phi(X) -  \psi_a(X, X'; \mu)\right)^2 - \left\{Y - \mu(X)\right\}\sum w(S)\sigma(S)\phi(X_S, X'_{-S})\zeta^{S}(X, X')
\end{align*}
The following Proposition formalizes the above intuition. The presented population loss $L(\phi; \mu, \beta)$ below is an equivalent re-writing of $L^{\corr}$ defined above. We prefer the formulation in the following proposition due to the fact that it explicitly decomposes $L^{\corr}$ over subsets $S \subset [d]$. In practice, this allows one to compute an approximate loss minimizer by randomly sampling subsets, which is more computationally efficient. We prove Proposition~\ref{prop:ortho_loss} in Appendix~\ref{app:learning}.

\begin{prop}
\label{prop:ortho_loss}
For $S \subset [d]$, define the partial loss function $\ell_S(Z, Z'; \phi, \mu, \zeta_S)$ by 
\[
\ell_S(Z, Z'; \phi, \mu, \zeta_S) := \Big(\tfrac{\phi(X)}{2}- \sigma(S)\mu(X_{S},X'_{-S})\Big)^2
- \sigma(S) \phi(X_{S},X'_{-S})\alpha_{S}(X,X')\{Y-\mu(X)\}
\]
where $\sigma(S) := 2\1\{a \in S\} - 1$ and let $L_S(\phi; \mu, \zeta_S) := \E_{Z, Z'}\left[\ell_S(Z, Z'; \phi, \mu, \zeta_S)\right]$. Further, letting $\zeta = (\zeta^S : S \subset [d])$, define the loss $\ell(Z, Z'; \phi, \mu, \zeta)$ by
\[
\ell(Z, Z'; \phi, \mu, \zeta) := \sum_{S \subset [d] : a \notin S}\frac{1}{d}\binom{d - 1}{|S|}^{-1}\Big(\ell_{S \cup a}(Z, Z'; \phi, \zeta_{S \cup a}, \mu) + \ell_S(Z, Z'; \phi, \zeta_S,  \mu)\Big)
\]
and let $L(\phi; \mu, \zeta) := \E_{Z, Z'}\left[\ell(Z, Z'; \phi, \mu, \zeta)\right]$. Then, letting $g_0 = (\mu_0, \zeta_0)$ and $\zeta_0 := (\zeta_0^S : S \subset [d])$, $L(\phi; \mu, \zeta)$ is Neyman orthogonal in the sense of Definition~\ref{def:neyman_ortho_loss}. Further, $\phi_a = \arg\min_{\phi : \calX \rightarrow \R}L(\phi; \mu_0, \zeta_0)$.
\end{prop}

The above proposition gives us a  loss for the SHAP curve $\phi_a(X)$ that is insensitive to misestimation in both the regression $\mu_0$ as well as the additional nuisance $\zeta_0$. We can exploit the above proposition to prove a generic bound on the $L^2$ error of any estimator $\wh{\phi}$ of $\phi_a$ in terms of the excess risk of $\wh{\phi}$ under arbitrary nuisance estimates $\wh{g}$ plus a penalty for misestimating $g_0 = (\mu_0, \zeta_0)$. 

\begin{ass}
\label{ass:reg_loss}
Letting $c, C > 0$ denote absolute constants, we assume the following. (B1) For every $S \subset [d]$, the density ratio
$\omega_{S}(x):=\frac{p_{S}(x_S)p_{-S}(x_{-S})}{p(x_S,x_{-S})}$ satisfies
$0< c \le \omega_S(X)\le C <\infty$ almost surely for some constants $c, C > 0$. (B2) $\sup_{\phi\in\Phi_a}\|\phi\|_{L^\infty(P_X)}\le C$. (B3) $\Phi - \phi_a$ is star-shaped\footnote{We say that a set $\calS$ is star-shaped around a point $s \in \calS$ if the interval $[s, s'] \subset \calS$ for all $s' \in \calS$.} and $\phi_a \in \Phi$ (B4) $|Y|\le C$ almost surely and 
$\sup_{S \subseteq [d]}\|\zeta^S_0\|_{L^\infty(P_{X, X'})}\le C$.(B5) All nuisance estimates are bounded: if $\wh{g}=(\wh{\mu},\wh{\zeta})$, then
$\|\wh{\mu}\|_{L^\infty(P_X)}\le C$ and $\max_{S \subseteq [d]}\|\wh{\zeta}^S\|_{L^\infty(P_{X, X'})}\le C$.
\end{ass}

\begin{lemma}\label{lem:loss_bound}
Assume (B1) and (B3) hold, and let $g_0 = (\mu_0, \zeta_0)$ denote the true, unknown nuisances. Let $g$ and $\phi$ be any (possibly random) estimates of $g_0$ and $\phi_a$, respectively. We have:
\[
\frac{1}{4}\|\phi-\phi_a\|_{L^2(P_X)}^2
\leq 
\Big(L(\phi; g)-L(\phi_a;g)\Big)
+ \frac{D}{4}\|g - g_0\|_{\mathcal G}^4,
\]
where $D > 0$ is an absolute constant and $\|g\|_{\calG} := \max_{S \subset [d]}\E\left[\mu(X)^2\zeta^S(X, X')^2\right]^{1/4}$.\footnote{By H\"older's inequality, we have
$\|g\|_{\calG}
\leq \bigl(\|\mu\|_{L^4(P_X)}\|\zeta\|_{\calA}\bigr)^{1/2}$,
where
$\|\zeta\|_{\calA}
:= \max_{S \subseteq [d]}
\|\zeta^S\|_{L^4(P_{X,X'})}$.}
\end{lemma}

The above follows from exploiting the orthogonality of $L$ alongside first-order optimality conditions for $\phi_a$. We provide a full proof in Appendix~\ref{app:learning}. We now present Theorem~\ref{thm:excess_risk}, which leverages Lemma~\ref{lem:loss_bound} to provide  bounds on excess risk based on the critical radius of the assumed function class. In more detail, for a star-shaped class of functions $\calH \subset L^\infty(P_X)$, define the $\calH(r) := \{h \in \calH : \|h\|_{L^2(P_X)} \leq r\}$, and let the \textit{localized Rademacher complexity} be
$\wt{\calR}_n(r, \calH) := \E_{\epsilon, X, X'}\Big[\max_{S \subset [d]} \sup_{h \in \calH(r)}\frac{1}{n}\sum_{i = 1}^n \epsilon_i h(X_S, X'_{-S})\Big]$.
Correspondingly, we also define the \textit{critical radius $r_n^\ast$} as the minimal solution to the inequality $\wt{\calR}_n(r_n, \calH) \leq r_n^2$. We note that our notions of localized complexity and critical radius are defined in terms of worst-case complexities of the function classes computed over all product distributions $P_{X_S} \otimes P_{X_{-S}}$ over covariates. This is only natural, as the SHAP curve is defined in terms of averages over such product distributions.

\begin{theorem}
\label{thm:excess_risk}
Let $Z_1, \dots, Z_n, Z_1', \dots, Z_n' \sim P_Z$ be i.i.d., and suppose $\wh{\mu}$ and $\wh{\zeta} = (\wh{\zeta}^S : S \subset [d])$ are estimates of $\mu_0$ and $\zeta_0 = (\zeta^S : S \subset [d])$ that are independent of the sample. Let $\Phi \subset L^\infty(P_X)$ be a class of bounded functions, and suppose points (B1)-(B5) in Assumption~\ref{ass:reg_loss} hold. Let $\wh{\phi}$ be defined as the empirical risk minimizer, i.e. $
\wh{\phi} := \arg\min_{\phi \in \Phi}\frac{1}{n}\sum_{i = 1}^n \ell(Z_i, Z_i'; \phi, \wh{\mu}, \wh{\zeta})$. Let $\wt{\calR}_n(r, \Phi - \phi_a)$ be the localized Rademacher complexity as defined above, let $r_n^\ast$ denote the critical radius, and let $r_n$ satisfy  $r_n\gtrsim \sqrt{\log n/n} \lor r_n^\ast$. 
Then
\[
\|\hat\phi-\phi_0\|_{L^2}^2 = O_\P(r_n^2) + O_\P\left(\|\hat g-g_0\|_{\mathcal G}^4\right).
\]
In particular, if $\|\hat g-g_0\|_{\mathcal G}=o_p(r_n^{1/2})$, then
$\|\hat\phi-\phi_0\|_{L^2}^2 = O_p(r_n^2)$.

\end{theorem}

We prove Theorem~\ref{thm:excess_risk} in Appendix~\ref{app:learning}. We state $O_\P$ rates above due to their alignment with our main inferential theorems, but one could analogously prove ``with high probability'' guarantees using the same techniques. If $\Phi$ is parametric, a VC-class, or a ball in an RKHS with exponential eigendecay (e.g.\ the Gaussian kernel), we have $r_n^\ast O(\sqrt{\tfrac{\log n}{n}})$. If $\Phi$ is the ball in a RKHS with polynomial eigendecay $\lambda_i = O(i^{-\nu})$ for some $\nu > 0$ (e.g.\ for the Matern family of kernels) we get $r_n^\ast = \wt{O}(n^{-\nu/2(1 + \nu)})$, where $\wt{O}(f(n))$ suppresses dependence on poly-logarithmic factors.

%% file: conclusion.tex
\section{Conclusion}
\label{sec:conclusion}

In this paper, we studied two intertwined problems. First, we considered the problem of constructing asymptotically-valid confidence intervals for $\theta_p := \E|\phi_a(X)|^p$, the $p$th powers of the SHAP curve. These powers are frequently leveraged in feature selection algorithms and are plotted by open source software, and thus providing confidence intervals is of practical importance. 
To complement our inferential results (which assumed access to an estimate $\wh{\phi}$ of the SHAP curve $\phi_a$), we also provided a Neyman orthogonal loss for learning the entire SHAP curve, and proved doing ERM with respect to our loss yields high-fidelity estimates $\wh{\phi}$ of $\phi_a$. 

There are a variety of interesting open problems related to the work considered in this paper. First, while we showed inferential results for $\theta_p := \E|\phi_a(X)|^p$, an interesting next step would be to produce and study estimators for arbitrary functionals $\E\left[\chi(X; \phi_a)\right]$. In the case $\chi(X; \phi_a)$ is twice Gateaux differentiable, applying our U-statistic to score
\[
m_p^\chi(Z, Z'; g) = \chi(X; \phi) + \gamma(X)\{\psi_a(X, X'; \mu_0) - \phi(X)\} - \alpha(X')\{Y' - \mu_0(X')\} 
\]
with $\gamma_p(X) := \frac{\partial}{\partial \phi_a(x)}\chi(X; \phi_a)$ should yield asymptotic linearity as in Theorem~\ref{thm:regular}. However, in the setting where $\E\left[\chi(X; \phi_a)\right]$ is not twice-differentiable, it is not clear how to directly extend our smoothing function. Another interesting direction is to develop hypothesis tests for feature relevance. \citet{morzywolek2025inference} do this for projections of $\phi_a(X)$ onto balls in RKHS's, but developing theory in fully non-parametric settings is an interesting direction. In sum, we believe our work provides a useful jumping off point for future works that use semi-parametric and statistical learning methods to understand variable importance metrics.

%% file: appendix/related_work.tex
\section{Extended Related Work}
\label{app:related}

\paragraph{Shapley Additive Explanations}

The SHAP framework for local variable importance was first introduced by \citet{lundberg2017unified}, where authors leveraged Shapley values from cooperative game theory~\citep{shapley1953value} to construct VIMs satisfying four important properties: efficiency, symmetry, linearity, and the null player property. 
In subsequent work, \citet{lundberg2018consistent} provide an efficient dynamic programming algorithm for computing the SHAP value $\phi_a(X)$ for additive tree models, which are widely used as black-box learners in practice. In \citet{lundberg2020local}, the authors discuss various marginal measures of variable importance that can be obtained from the SHAP framework, with the mean absolute SHAP being an important example. We note that, in these works, the authors do not provide valid confidence intervals for either marginal or local measures of variable importance.
We also emphasize that these results are all \textit{model-specific}, replacing the true regression $\mu_0(x) := \E[Y \mid X = x]$ in Equation~\eqref{eq:shap} with a learned approximation $\wh{\mu}$. Many other works (including the present one)~\citep{williamson2020efficient, williamson2021nonparametric, williamson2023general, miftachov2024shapley, owen2017shapley} instead take a \textit{model-agnostic} perspective, attempting to assign importance to features based on their impact on an outcome variable of interest $Y$. We discuss these works more in the paragraphs below.

We also emphasize that, in the SHAP literature, there is also a divide in works that define variable importance in \textit{conditional expectations} ($\E[\wh{\mu}(X) \mid X_S = x_S]$ or $\E[\mu_0(X) \mid X_S = x_S]$)~\citep{lundberg2017unified, lundberg2018consistent, lundberg2020local}  and marginal expectations ( $\E[\wh{\mu}(X_{-S}, x_S)]$ or $\E[\mu_0(X_{-S}, x_S)]$)~\citep{janzing2020feature, sundararajan2020many, williamson2023general}. 
We emphasize that most commonly leveraged software packages (such as \texttt{SHAP}~\citep{lundberg2025shap} and \texttt{shapley}~\citep{R-shapley}) leverage the marginal perspective. Further, authors such as \citet{sundararajan2020many} and \citet{janzing2020feature} argue that SHAP values defined in terms of conditional expectations lack certain consistency properties. In particular, the latter authors argue that defining SHAP values in terms of marginal expectations posseses a nice causal interpretation of setting feature values via a ``do'' operation in a causal graph where all features $X$ serve as a common cause of an outcome $Y$.
\citet{chen2023algorithms} consider both conditional and marginal SHAP values, and provide an extensive overview of various efficient approximation algorithms for computing both. We also mention the works of \citet{heskes2020causal, frye2020asymmetric}, which introduce causal Shapley values and asymmetric Shapley values respectively to to handle causal structure present in the covariates $X$. However, these approaches require exact knowledge of the underlying causal directed acyclic graph (DAG), which in practice may be unknown. In this work, we solely consider the \textit{marginal perspective}, and leave the problem of constructing confidence intervals for certain downstream functionals of conditional SHAP values for interesting future work.

Lastly, we discuss works related to SHAP that use Shapley values to define model-agnostic variable importance metrics. We note the contributions of \citet{miftachov2024shapley}, who use kernel smoothing to perform inference on the local SHAP value $\phi_a(x_0)$, where $x_0 \in \R^d$ is a target covariate of interest. However, the estimator's described by the author require either (a) the computation of an intractable number of regression estimates (one for each subset of variables) or (b) knowledge of the conditional densities for the feature distributions. Also related to our work are the contributions of \citet{morzywolek2025inference}, who consider inference on the \textit{projection} of the SHAP curve $\phi_a$ onto a ball in a reproducing kernel Hilbert space (RKHS). They also design hypothesis tests based on the boostrap for zero feature importance under the projected curve. We note that, in general, the projection of $\phi_a$ may be very far from the true SHAP curve, and thus insights derived from the projection may not translate into equivalent insights on the base data generating process. Most related to our work are the contributions of Williamson et al.~\citep{williamson2020efficient, williamson2021nonparametric, williamson2023general}. In particular, in \citet{williamson2023general}, the authors also use the game-theoretic concept of Shapley values to define Shapley Population Variable Importance Metrics (SPVIMs). At a high level, these metrics are obtained by replacing the expected regressions $\E_X[\mu_0(X_{-S}, x_{S})]$ with \textit{population variable importance metrics} of the form $v_{0, S} := \inf_{f \in \calF_s} V(f; P)$ where $P$ denotes the data generating distribution over $(X, Y)$, $\calF_s$ is a class of functions/models only depending on $X_S$, and $V$ is a function measuring predictiveness. While the SPVIM framework can be used to measure analogues of $R^2$, classification accuracy, ROC, and more, it cannot be used to recover functionals $\phi_a(X)$ such as $\E[\phi_a(X)]$, $\Var[\phi_a(X)]$, or $\E|\phi_a(X)|$. Furthermore, because the metrics $v_{0, S}$ as a minimizer of some function $V$, the naive plug-in estimates are actually robust to misestimation in the nuisance component $f^\ast_S := \arg\min_{f \in \calF_s} V(f, P)$. This is in stark contrast to our setting, where plug-in estimates will generally not be asymptotically normal, and so de-biasing and smoothing methods must be used. We lastly note the work of \citet{covert2020understanding}, who introduce SAGE, a model-agnostic global VIM that is equivalent to SPVIM defined above. While the authors in this work do show the consistency of natural estimators for their metrics, they only present heuristic arguments for constructing asymptotically valid confidence intervals.

\paragraph{Variable Importance Metrics}

Variable importance metrics (VIMs) can be grouped into two buckets: \textit{model-specific} metrics, which attribute feature importance for a learned ML model, and \textit{model-agnostic} metrics, which explain how individual features contribute to some observed outcome variable $Y$. We start by describing some of the former. Classically, \cite{breiman2001random} use decrease in predictive accuracy when a variable is permuted to measure variable importance for random forests. Subsequently, \citet{fisher2019all} define a a general, permutation-based measure of variable importance via \textit{model reliance (MR)}, and then generalize this metric to entire classes of functions via \textit{model class reliance (MCR)}. \citet{ribeiro2016should} present LIME, a generic framework for locally approximating a generic black-box model by a simpler, interpretable one and then using this latter model for assessing variable importance. Additionally, many methods have been developed for measuring feature importance relevance for deep learning models~\citep{bach2015pixel, shrikumar2017learning}. There are many more model-specific VIMs than could be discussed here, and we point the reader towards \citet{guidotti2018survey} for a thorough overview.

We now discuss existing model-agnostic approaches to variable importance. \citet{van2006statistical} describe a notion of variable importance that measures either the marginal or conditional increase in an outcome $Y$ under a perturbation of a coordinate of interest in covariates $X$. \citet{wang2023total} instead define a notion of variable importance based on the expected total variation distance between conditional laws of $\E[Y \mid V]$ and $\E[Y \mid X]$ where $V \subset X$ is a restricted subset of covariates. \citet{zhang2020floodgate} describe an approach for measuring the importance (based on squared error) of a variable in the presence of potentially high-dimensional confounding. Additionally, \citet{lei2018distribution} propose LOCO, a leave-one-out measure of variable importance based in conformal methods that trains a separate regression model for each possible omitted variable.

\paragraph{Causal Machine Learning}

Causal inference, and more broadly semi-parametric statistics, focuses on the estimation of low-dimensional structural parameters in the presence of infinite-dimensional nuisance parameters~\citep{kosorok2008introduction, bickel1993efficient, kennedy2016semiparametric, tsiatis2006semiparametric, van2006targeted}. Of particular importance to our work is the literature on double machine learning~\citep{chernozhukov2018double}, which presents a general cross-fitting estimator for solutions to Neyman orthogonal estimating equations involving nuisance. Also relevant is the work on semi-parametric estimation under covariate shifts, nested regressions, and nuisance functions that may lack closed-form representations~\citet{chernozhukov2022automatic, chernozhukov2022nested, chernozhukov2022riesznet, chernozhukov2023automatic}. Also directly related to our work is the literature that focuses on performing statistical learning for heterogeneous causal effects or, more broadly, statistical parameters that are minimizers of losses involving nuisance components~\citet{foster2023orthogonal, whitehouse2024orthogonal, oprescu2019orthogonal, wager2018estimation}. Using these latter sets of tools (particularly, the contributions of \citet{foster2023orthogonal}), we provide excess risk bounds on ML estimates of the SHAP curve $\phi_a$. One could also use these tools to construct confidence intervals (at non-parametric rates) for $\phi_a(x_0)$ for some target point $x_0 \in \R^d$, but we leave this for future work.

We also note several works in the causal inference and causal ML literature that leveraging smoothing in the estimation of irregular parameters. One of the most important examples of an irregular parameter is the \textit{value $V^\ast$ of the optimal treatment policy}, which can be specified under standard identifying assumptions as the maximum of two regression functions: $V^\ast = \E\left[\max\{\mu_0(0, X), \mu_0(1, X)\}\right]$ where $\mu_0(a, x) = \E[Y \mid X = x,  A = a]$~\citep{luedtke2016statistical}. In both parametric and non-parametric settings, authors have leveraged a mix of soft-plus~\citep{ goldberg2014comment} and soft-max~\citep{chen2023inference, whitehouse2025inference} smoothing to perform inference on $V^\ast$. Likewise, \citet{levis2023covariate} use soft-plus smoothing in the estimation of conditional Balke and Pearl bounds~\citep{balke1994counterfactual, balke1997bounds}, although the authors do not show how to select smoothing parameters as a function of the sample size to obtain normality around the true population quantities. Our smoothing strategy is somewhat different from the ones mentioned in the aforementioned papers, as we use hyperbolic tangent smoothing to construct twice continuously differentiable approximations of the maps $u \mapsto |u|^p$. However, we do inherit our key assumptions on the distribution of $\phi_a(X)$ from existing works~\citep{chen2023inference, luedtke2016statistical, whitehouse2025inference}. 

%% file: appendix/shap_powers_proofs.tex
\section{Proofs from Section~\ref{sec:shap}}
\label{app:shap_proofs}
In this appendix, we provide the proofs of our main inferential results from Section~\ref{sec:shap}. We start by stating and proving several lemmas relating to the regularity of the smoothing function $\varphi_{p, \beta}(u) := |u|^p \tanh(\beta|u|^{2 - p})$.  We then proceed to prove Neyman orthogonality of the scores $m_p(Z, Z; g)$ (from Proposition~\ref{prop:ortho_regular}) and $m_p^\beta(Z, Z'; g)$ (from Proposition~\ref{prop:ortho_irregular}). Lastly, we prove asymptotic linearity of the U-statistics in both the un-smoothed setting ($p \geq 2$) and then the smoothed setting ($1 \leq p < 2$).

We start by stating and proving a lemma that controls details the first and second derivatives of $\varphi_{p, \beta}$, and also provides a uniform upper bound on the second derivative $\varphi_{p, \beta}''(u)$ in terms of $\beta$.

\begin{lemma}
\label{lem:smooth}
Let $\beta > 0$ and $p \in [1, 2)$ be arbitrary. Then, the function $\varphi_{p, \beta}(u) := |u|^p\tanh(\beta|u|^{2- p})$ is twice continuously differentiable in $u$ with derivatives $\varphi_{p, \beta}'$ and $\varphi_{p, \beta}''$ given by
\begin{align*}
    \varphi_{p, \beta}'(u) &=
    \sgn(u)\left[p|u|^{p - 1}\tanh(\beta|u|^{2 - p}) + \beta(2 - p)|u|\sech^2(\beta|u|^{2 - p})\right] \\
    \varphi_{p, \beta}''(u) &= p(p - 1)|u|^{p - 2} \tanh(\beta|u|^{2 - p}) + \beta(2 - p)(p + 1)\sech^2(\beta|u|^{2- p}) \\
    &\qquad \qquad - 2\beta^2(2 - p)^2|u|^{2 - p}\sech^2(\beta|u|^{2 - p})\tanh(\beta|u|^{2 - p}) \qquad \qquad (\text{For } u \neq 0) \\
    \varphi''_{p, \beta}(0) &= 2\beta.
    \end{align*}
Further, there is an absolute constant $C_p > 0$ such that $|\varphi''_{p, \beta}(u)| \leq C_p \beta$ for all $u \in \R$.

\end{lemma}

\begin{proof}
We can directly compute the first derivative using the product rule for $u \neq 0$. Namely, for $u > 0$, we have 
\begin{align*}
\left(\varphi_{p, \beta}\right)'(u) &= \left(u^p\right)'\tanh(\beta u^{2- p}) + u^p\left(\tanh(\beta u^{2 - p})\right)' &(\text{Product rule}) \\
&= p u^{p - 1} \tanh(\beta u^{2 -  p}) + \beta (2 - p)u^{p}u^{1 - p}\sech(\beta u^{2 - p}) &(\text{Chain rule}) \\
&= pu^{p - 1}\tanh(\beta u^{2 - p}) + \beta(2 - p)u\sech(\beta u^{2-p}).
\end{align*}
Next, since $\varphi_{p, \beta}$ is an even function, it follows that $\varphi_{p, \beta}'(-u) = -\varphi_{p, \beta}'(u)$, and so for $u \neq 0$ we have 
\[
\varphi_{p, \beta}'(u) =
    \sgn(u)\left[p|u|^{p - 1}\tanh(\beta|u|^{2 - p}) + \beta(2 - p)|u|\sech(\beta|u|^{2 - p})\right]
\]
Lastly, since $\varphi_{p, \beta}$ is a continuous functions and $\lim_{u \downarrow 0}\varphi_{p, \beta}'(u) = 0 = \lim_{u \uparrow 0} \varphi'_{p, \beta}(u)$, we have $\varphi_{p, \beta}'(u) = 0$ and the first claim follows.

Next, we check the second derivative. Again, we start by analytically computing the derivative when $u \neq 0$. First, we have
\begin{align*}
\small
\left(p|u|^{p - 1}\tanh(\beta |u|^{2 - p})\right)' &= p \left(|u|^{p - 1}\right)' \tanh(\beta |u|^{2 - p}) + p|u|^{p  - 1}\left(\tanh(\beta |u|^{2 - p})\right)'\\
&= \sgn(u)\Big[p(p - 1)|u|^{p - 2}\tanh(\beta|u|^{2 - p}) + p (2 - p)|u|^{p - 1}|u|^{1 - p}\sech(\beta |u|^{2 - p})\Big] \\
&=\sgn(u)\Big[p(p - 1)|u|^{p - 2}\tanh(\beta|u|^{2 - p}) + p (2 - p)\sech(\beta |u|^{2 - p})\Big]
\end{align*}
Further, we have 
\begin{align*}
&\left(\beta(2 - p)|u|\sech^2(\beta |u|^{2 - p})\right)' = \beta(2- p)\left(|u|\right)'\sech^2(\beta |u|^{2 - p}) + \beta(2 - p)|u|\left(\sech^2(\beta|u|^{2 - p})\right)'\Big] \\
&\qquad = \sgn(u)\Big[\beta(2 - p)\sech^2(\beta |u|^{2 - p}) + \beta^2 (2 - p)^2|u|^{2 - p}\sech^2(\beta |u|^{2 - p})\tanh(\beta|u|^{2 - p})\Big].
\end{align*}
Consequently, for $u \neq 0$, we have 
\begin{align*}
\varphi_{p, \beta}''(u) &= \underbrace{p(p - 1)|u|^{p - 2} \tanh(\beta|u|^{2 - p})}_{=: T_1(u)} + \underbrace{\beta(2 - p)(p + 1)\sech^2(\beta|u|^{2- p})}_{=: T_2(u)} \\
&\qquad \qquad - \underbrace{2\beta^2(2 - p)^2|u|^{2 - p}\sech^2(\beta|u|^{2 - p})\tanh(\beta|u|^{2 - p})}_{=: T_3(u)}.
\end{align*}
Again, note that $\varphi'_{p, \beta}(u)$ is continuous on $\R \setminus \{0\}$. Since $\varphi''_{p, \beta} : \R \setminus \{0\} \rightarrow \R$ is even, it suffices to show that the right-hand limit is finite at zero, as the left-hand derivative must coincide with it. We have
\begin{align*}
\lim_{u \downarrow 0}T_1(u) &= \lim_{u \downarrow 0}p(p - 1)|u|^{p -2}\tanh(\beta|u|^{2 - p}) = \beta p(p - 1)\lim_{u \downarrow 0}\frac{1}{\beta|u|^{2-p}}\tanh(\beta |u|^{2 - p}) \\
&= \beta p (p - 1)\lim_{x \downarrow 0}\frac{1}{x}\tanh( x).
\end{align*}
Next, note that we can write $\tanh(x) = x + \Rem(x)$ where $\Rem(x) = o(x)$ for $|x| < \frac{\pi}{2}$. Hence, we have
\[
\lim_{x \downarrow 0}\frac{1}{x}\tanh(x) = \lim_{x \downarrow 0}\frac{1}{x}\{x + \Rem(x)\} = 1.
\]
Thus, we have shown that $\lim_{u \downarrow 0} T_1(u) = \beta p (p - 1)$. Next, we have
\[
\lim_{u \downarrow 0}T_2(u) = \beta(2 - p)(p + 1)\lim_{u \downarrow 0}\sech^2(\beta |u|^{2 - p}) = \beta(2 - p)(p + 1),
\]
since $\sech^2(0) = 1$. Lastly, we have 
\[
\lim_{u \downarrow 0}|u|^{2 - p} \sech^2(\beta|u|^{2 - p})\tanh(\beta |u|^{2 - p}) = \lim_{u \downarrow 0}|u|^{2 - p} \cdot \lim_{u \downarrow 0}\sech^2(\beta |u|^{2- p}) \cdot \lim_{u \downarrow 0}\tanh(\beta|u|^{2 -p}) = 0 \cdot 1 \cdot 0 = 0,
\]
and hence $\lim_{u \downarrow 0}T_3(u) = 0$. Thus, in aggergate, we have
\[
\lim_{u \downarrow 0}\varphi_{p, \beta}''(u) = \beta p (p - 1) + \beta (2 - p)(p + 1) = 2\beta,
\]
which is finite and hence the value of $\varphi''_{p, \beta}(u)$.

Lastly, we need to show that $\sup_{u \in \R}\varphi''_{p, \beta}(u) \leq C_p \beta$ where $C_p > 0$ is some constant only depending on $p$. To do this, we start by showing that $\sup_{u}|T_1(u)| \leq C_{1, p} \beta$. By the same analysis done earlier, we know that
\[
\sup_u \left|T_1(u)\right| = p (p - 1)\sup_u|u|^{p - 2}\tanh(\beta |u|^{2 - p}) =  \beta p (p - 1) \sup_{x \geq 0}x^{-1}\tanh(x).
\]
Now, we already know that $\lim_{x \downarrow 0}x^{-1}\tanh(x) = 1$, that $x^{-1}\tanh(x)$ is continuous on $[0, \infty)$, and that $\sup_{x \geq 1}x^{-1}\tanh(x) \leq \sup_{x \geq 1} x^{-1} \leq 1$. Hence, $\sup_{x \geq 0} x^{-1}\tanh(x) \leq C_1$ for some absolute constant $C_1$. Thus, $\sup_{u}T_1(u) \leq \beta C_{p, 1}$ for $C_{p, 1} = p (p - 1)C_1$. 
Now, we show $\sup_u |T_2(u)| \leq C_{2, p}\beta$ for some appropriate $p$-dependent constant $C_{2, p} > 0$. Since $\sup_x |\sech^2(x)| = 1$, we have 
\[
\sup_u |T_2(u)| = \sup_u \beta(2 - p)(p + 1)\sech^2(\beta |u|^{2- p}) = \beta(2 - p)(p + 1).
\]
Hence, $\sup_u |T_2(u)| = \beta C_{2, p}$, where $C_{2, p} = \beta(2 - p)(p + 1)$. 

Lastly, we show that $\sup_u |T_3(u)| \leq C_{p, 3}\beta$ some $C_{p, 3} > 0$. We have:
\begin{align*}
\sup_u |T_3(u)| &= \sup_u \left\{2\beta^2(2 - p)^2 |u|^{2 - p}\sech^2(\beta |u|^{2 - p})\left|\tanh(\beta |u|^{2 - p})\right|\right\} \\
&= 2\beta (2 - p)^2\sup_u \left\{ \beta|u|^{2 - p}\sech^2(\beta |u|^{2 - p})\left|\tanh(\beta |u|^{2 - p})\right|\right\} \\
&= 2 \beta (2 - p)^2 \sup_{x \geq 0}x\sech^2(x)\tanh(x).
\end{align*}
Now, since $f(x) := x\sech^2(x)\tanh(x)$ is continuous, $\lim_{x \downarrow 0}f(x) = 0$, and $\lim_{x \rightarrow \infty}f(x) = 0$, it follows that there exists some absolute constant $C_3 > 0$ such that $f(x) \leq C_3$ for all $x  > 0$. Consequently, it follows that $\sup_u |T_3(u)| \leq C_{p, 3}\beta$ where $C_{p, 3} := (2 - p)^2 C_3$. Thus, in aggregate, we have shown that
\[
\sup_{u}|\varphi''_{p, \beta}(u)| \leq \sup_u|T_1(u)| + \sup_u |T_2(u)| + \sup_u |T_3(u)| \leq C_p \beta,
\]
where $C_p := C_{p, 1} + C_{p, 2} + C_{p, 3}$. This thus completes the proof.
\end{proof}

Next, we prove Lemma~\ref{lem:bias} from Section~\ref{sec:shap:irregular}. The proof follows from combining ideas from the proof of Lemma A.3 in \citet{whitehouse2025inference} with the analytic properties of the smoothing function $\varphi_{p, \beta}$.
\newline
\begin{proof}[Proof of Lemma~\ref{lem:bias}]
First, note that for $x > 0$, we have the inequality
\[
1 - \tanh(x) = \frac{e^{x} + e^{-x}}{e^x + e^{-x}} - \frac{e^x - e^{-x}}{e^x + e^{-x}} = \frac{2e^{-x}}{e^x + e^{-x}} \leq 2e^{-x},
\]
where the final inequality follows since $e^x + e^{-x} \geq 1$ for all $x \in \R$. Thus, letting $U := |\phi_a(X)|$ and $P_U$ the distribution of $U$ for notational ease, we see that, for any $1 \leq p < 2$ and $\beta > 0$, we have
\begin{align*}
\Bias(\beta) &:= |\theta_{p} - \theta_{p, \beta}| = \left|\E\left[U^p\left(1 - \tanh(\beta U^{2- p})\right)\right]\right| \\
&\leq 2\E\left[U^p \exp\left(-\beta U^{2- p}\right)\right] \\
&= 2\underbrace{\int_0^c u^p \exp(-\beta u^{2- p})P_U(du)}_{=:T_1} + 2\underbrace{\int_c^\infty u^p\exp(-\beta u^{2 -p})P_U(du)}_{=:T_2}.
\end{align*}

We start by bounding the second term. In particular, note that the function $f : \R_{\geq 0} \rightarrow \R_{\geq 0}$ given by $f(x) := x^p \exp(- x^{2- p})$ is maximized at $x_0 = \left(\frac{p}{2 - p}\right)^{\frac{1}{2 - p}}$ and strictly decreasing on the interval $[x_0, \infty)$. We have
\begin{align*}
T_2 &= \int_c^\infty u^p\exp(-\beta u^{2- p})P_U(du) = \beta^{-\frac{p}{2 - p}}\int_c^\infty \left(\beta^{\frac{1}{2 - p}} u\right)^{p}\exp\left\{- \left(\beta^{\frac{1}{2 - p}}u\right)^{2 - p}\right\}P_U(du)  \\
&\leq \beta^{-\frac{p}{2 - p}}\sup_{ u \geq c}\left\{\left(\beta^{\frac{1}{2 - p}} u\right)^{p}\exp\left\{- \left(\beta^{\frac{1}{2 - p}}u\right)^{2 - p}\right\}\right\} \\
&\leq \beta^{-\frac{p}{2 - p}}\sup_{u \geq \left(\beta^{1/(2 - p)}c\right)^p}\left\{u^p \exp(-u^{2 - p})\right\} \\
&\leq \beta^{-\frac{p}{2 - p}}\left(\beta^{\frac{1}{2 - p}}c\right)^p \exp\left\{-\left(\beta^{\frac{1}{2 -p}} c\right)^{2 - p}\right\} \\
&= c^p\exp\left\{-\beta c^{2 - p}\right\}.
\end{align*}
where the final inequality holds for any $\beta$ large enough such that $\beta^{\frac{p}{2 - p}}c^p \geq x_0$. Next, to bound $T_1$ we observe that we have
\begin{align*}
T_1 &= \int_0^c u^p\exp(-\beta u^{2 - p})P_U(du) = \int_0^c u^p \exp(-\beta u^{2- p})f_U(u)du \\
&\leq H \int_0^\alpha u^{p + \delta - 1}\exp(-\beta u^{2 - p})du &(\text{Assumption~\ref{ass:margin}})\\
&\leq H\int_0^\infty u^{p + \delta - 1}\exp(-\beta u^{2 - p})d u \\
&= H\int_0^{\infty}\left(\frac{v}{\beta^{1/(2 - p)}}\right)^{p + \delta - 1}\exp(-v^{2 - p})\left(\frac{1}{\beta}\right)^{\frac{1}{2 - p}}dv &\left(\text{c.o.v. with } v := \beta^{\frac{1}{2 - p}}u\right) \\
&= H\left(\frac{1}{\beta}\right)^\frac{p + \delta}{2 - p}\int_0^{\infty}v^{p + \delta - 1}\exp(-v^{2 - p}) dv \\
&= H\frac{1}{2 - p}\Gamma\left(\frac{p + \delta}{2 - p}\right) \left(\frac{1}{\delta}\right)^\frac{p + \delta}{2 - p},
\end{align*}
where the final equality follows from the identity 
\[
\int_0^\infty x^{a - 1}e^{-x^b}dx = \frac{1}{b}\int_0^\infty x^{\frac{a}{b} - 1}e^{-x}dx = \frac{1}{b}\Gamma\left(\frac{a}{b}\right)
\]
and substituting $a = p + \delta$ and $b  = 2 - p$. Thus, in aggregate, we have
\[
\Bias(\beta) \leq 2 H \frac{1}{2 - p}\Gamma\left(\frac{p + \delta}{2 - p}\right)\left(\frac{1}{\delta}\right)^{\frac{p + \delta}{2 - p}} + 2c^p\exp\{-\beta c^{2- p}\}
\]
for sufficiently large $\beta > 0$. The result now follows by noting that $c^p\exp\{-\beta c^{2-p}\} \leq \left(\frac{1}{\beta}\right)^{\frac{p + \delta}{2 - p}}$ for large $\beta$.
\end{proof}

\subsection{Proofs of Orthogonality}
We now prove the main propositions regarding the orthogonality of our two-sample scores. We start by proving Proposition~\ref{prop:ortho_regular}, which gives a de-biased score $m_p(Z, Z'; g)$ for $\theta_p := \E|\phi_a(X)|^p$ when $p \geq 2$.
\newline
\begin{proof}[Proof of Proposition~\ref{prop:ortho_regular}]
First, we argue that $M_p(g_p) = \theta_p$. To do this, it suffices to argue that 
\[
\E[\gamma_p(X)\left\{\psi_a(X, X'; \mu)- \phi_a(X)\right\}]  = 0 \;\; \text{and} \;\;  \E[\alpha_p(X')\left\{Y' - \mu_0(X')\right\}] = 0.
\]
First, note that the tower rule for conditional expectations yields
\begin{align*}
\E[\gamma_p(X)\{\psi_a(X, X'; \mu_0) - \phi_a(X)] = \E\left[\gamma_p(X)\{\E(\psi_a(X, X'; \mu_0)\mid X) - \phi_a(X)\}\right] = 0,
\end{align*}
as desired. Likewise, again applying the tower rule, we have
\begin{align*}
\E[\alpha_p(X')\{Y' - \mu_0(X')\}] = \E[\alpha_p(X')\{\E(Y' \mid X') - \mu_0(X')\}] = 0.
\end{align*}
Now, we argue orthogonality. We start by checking the derivative with respect to $\phi$. Letting $\Delta_\phi := \phi - \phi_a$ where $\phi \in L^\infty(P_X)$ is a bounded function, we have
\begin{align*}
&D_\phi \E\left[|\phi_a(X)|^p\right](\Delta_\phi) := \frac{\partial}{\partial t}\E\left[|(\phi_a + t \Delta_\phi)(X)|^p\right]\Big|_{t =0}&(\text{Definition of Gateaux Derivative}) \\
&\quad = \lim_{\substack{t \rightarrow 0 \\ |t| \leq 1}}\E\left[\frac{1}{t}\left\{|(\phi_a + t \Delta_\phi)(X)|^p - |\phi_a(X)|^p\right\}\right]  &(\text{Definition of Partial Derivative})\\
&\quad = \lim_{\substack{t \rightarrow 0 \\ |t| \leq 1}}\E\left[p\sgn((\phi_a + \epsilon_{t, X}\Delta_\phi)(X))|(\phi_a + \epsilon_{t, X}\Delta_\phi)(X)|^{p - 1}\Delta_\phi(X)\right]&(\text{Mean Value Theorem})\\
&\quad = \E\left[\lim_{\substack{t \rightarrow 0 \\ |t| \leq 1}}p \sgn((\phi_a(X) + \epsilon_{t, X}\Delta_\phi)(X))|(\phi_a + \epsilon_{t, X}\Delta_\phi)(X)|^{p - 1}\Delta_\phi(X)\right] &(\text{Bounded Convergence Theorem})\\
&\quad = \E[p\sgn(\phi_a(X))|\phi_a(X)|^{p - 1}\Delta_\phi(X)] &(\text{Since } \epsilon_{t, X} \xrightarrow[t \rightarrow 0]{} 0), 
\end{align*}
where in the above $\epsilon_{t, X}$ is a measurable function of $X$ taking values in $[0, t]$.
Further, we trivially have 
\begin{align*}
D_\phi\E[\gamma_p(X)\{\psi_a(X, X'; \mu_0) - \phi_a(X)\}](\Delta_\phi) &= \frac{\partial}{\partial t}\E[\gamma_p(X)\{\psi_a(X, X'; \mu_0) - (\phi_a + t \Delta_\phi)(X)\}]\Big|_{t = 0} \\
&= -\frac{\partial}{\partial t}t\E[\gamma_p(X)\Delta_\phi(X)]\Big|_{t = 0} \\
&= -\E[p\sgn(\phi_a(X))|\phi_a(X)|^{p - 1}\Delta_\phi(X)],
\end{align*}
so the Gateaux derivative with respect to $\phi$ vanishes. 

Next, notice that, letting $\Delta_\mu := \mu - \mu_0$ for any bounded function $\mu \in L^\infty(P_X)$, we have
\begin{align*}
&D_\mu \E[\gamma_p(X) \psi_a(X, X'; \mu_0)](\Delta_\mu) = D_\mu \E_{X, X', S}[\gamma_p(X)\{\mu_0(X_{S \cup a}, X'_{-S\cup a}) - \mu_0(X_S, X'_{-S})\}](\Delta_\mu) \\
&\quad = \frac{\partial}{\partial t}\E_{X, X', S}\left[\gamma_p(X)((\mu_0 + t \Delta_\mu)(X_{S \cup a}, X'_{-S\cup a }) - (\mu_0 + t \Delta_\mu)(X_S, X'_{-S})\right]\Big|_{t = 0} \\
&\quad = \E_{X, X', S}[\gamma_p(X)\{\Delta_\mu(X_{S \cup a}, X'_{-S\cup a}) - \Delta_\mu(X_S, X'_{-S})\}] \\
&\quad = \E_{X, X', S}[\E(\gamma_p(X) \mid X_{S \cup a})\Delta_\mu(X_{S \cup a}, X'_{-S \cup a}) - \E(\gamma_p(X) \mid X_S)\Delta_\mu(X_S, X'_{-S})] \\
&\quad = \E_{X, X', S}\left[\gamma_p^{S \cup a}(X_{S \cup a})\Delta_\mu(X_{S \cup a}, X'_{-S \cup a}) - \gamma_p^S(X_S)\Delta_\mu(X_S, X'_{-S})\right]\tag{A}\label{align:continuation},
\end{align*}
where we explicitly write $\E_{X, X', S}$ to book-keep the explicit dependence of the expectation on the independent samples $X, X' \sim P_X$ and an independent, random subset $S$ drawn from the Shapley distribution on $[d]\setminus\{a\}$.\footnote{As noted in Section~\ref{sec:back}, we have $\psi_a(X, X'; \mu) = \E_S[\mu_0(X_{S \cup a}, X'_{-S \cup a}) - \mu_0(X_S, X'_{-S})]$, where $\E_S[\cdots]$ denotes the expectation with respect to a draw of $S\sim [d]\setminus \{a\}$ where $S$ is drawn with probability $\frac{1}{d}\binom{d - 1}{|S|}$. In particular, it is assumed that $S$ is independent of $X, X'$.}
In more detail, the first equality follows from writing $\psi_a(X, X'; \mu_0)$ as an expectation over subsets $S \subset [d] \setminus \{a\}$ drawn from the aforementioned distribution, and the second equality follows from the definition of Gateaux derivative. The second to last equality follows from taking a conditional expectation given $X_{S \cup a}$ and $X'_{-S \cup a}$ (respectively $X_S$ and $X'_{-S}$) and noting $X$ and $X'$ are independent. The final line follows from our earlier definition of the notation $\gamma_p^S(X_S) := \E(\gamma_p(X) \mid X_S)$ for any $S \subset [d]$. 

Next, note that for any subset $S \subset [d]$, under Assumption~\ref{ass:dgp}, we can perform a change of measure from the product distribution $P_{X_S} \otimes P_{X_{-S}'}$ to the marginal distribution $P_{X'}$ as follows:
\begin{align*}
\E_{X, X'}\left[\gamma_p^S(X_S)\Delta_\mu(X_S, X'_{-S})\right] &= \E_{X_S, X'_{-S}}\left[\gamma_p^S(X_S)\Delta_\mu(X_S, X'_{-S})\right] \\
&= \E_{X_S, X'_{-S}}\left[\gamma_p^S(X_S) \frac{p(X_S, X'_{-S})}{p(X_S, X'_{-S})}\Delta_\mu(X_S, X'_{-S})\right] \\
&= \E_{X'}\left[\gamma_p(X_S')\frac{p_S(X'_S)p_{-S}(X_{-S}')}{p(X_S', X_{-S}')}\Delta_\mu(X')\right] \\
&= \E_{X'}\left[\gamma_p(X_S')\omega_S(X')\Delta_\mu(X')\right],
\end{align*}
where as earlier we have defined $\omega_S(X) := \frac{p_S(X_S)p_{-S}(X_{-S})}{p(X_S, X_{-S})}$. Again, we have explicitly book-kept the dependence on $X$ and $X'$ in the subscript of the expectations to make the change of measure argument clearer.
Thus, continuing from the Line~\ref{align:continuation} above, we have
\begin{align*}
\mathrm{(A)} &= \E_{X', S}\left[\gamma_p^{S\cup a}(X_{S\cup a}')\omega_{S \cup a}(X')\Delta_\mu(X')\right] + \E_{X', S}\left[\gamma_p^{S}(X_S')\omega_S(X')\Delta_\mu(X')\right] \\
&=\E_{X'}\left[\left\{\frac{1}{d}\sum_{\substack{S \subset [d]\\a \notin S}}\binom{d - 1}{|S|}^{-1}\left(\gamma_p^{S \cup a}(X'_{S \cup a})\omega_{S \cup a}(X') - \gamma_p^{S}(X_S')\omega_S(X')\right)\right\}\Delta_\mu(X')\right] \\
&= \E_{X'}[\alpha_p(X')\Delta_\mu(X')].
\end{align*}
where the second equality follows from expanding out the expectation with respect to $S$ and the last equality follows from the definition of $\alpha_p(X)$. Given that the final value is precisely $-D_\mu\E_{X'}[\alpha_p(X')\{Y - \mu_0(X'))](\Delta_\mu),$ the Gateaux derivative with respect to $\mu$ also vanishes.

We now have to just check the Gateaux derivatives with respect to the representers $\gamma$ and $\alpha$. This is largely straightforward. First, for $\Delta_\gamma := \gamma - \gamma_p$ with $\gamma \in L^\infty(P_X)$, we have
\begin{align*}
&D_\gamma \E\left[\gamma_p(X)\left\{\psi_a(X, X'; \mu_0) - \phi_a(X)\right\}\right](\Delta_\gamma) = \frac{\partial}{\partial t}\E\left[(\gamma_p + t \Delta_\gamma)(X)\left\{\psi_a(X, X'; \mu_0) - \phi_a(X)\right\}\right]\Big|_{t = 0}\\
&\qquad = \E\left[\Delta_\gamma(X)\left\{\psi_a(X, X'; \mu_0) - \phi_a(X)\right\}\right] \\
&\qquad = \E\left[\Delta_\gamma(X)\left\{\E\left(\psi_a(X, X'; \mu_0) \mid X\right) - \phi_a(X)\right\}\right] \\
&\qquad = 0,
\end{align*}
where the second to last equality follows from the tower rule by taking a conditional expectation given $X$, and the final line follows since $\phi_a(X) = \E\left(\psi_a(X, X'; \mu_0) \mid X\right)$. Likewise, for square-integrable $\Delta_\alpha := \alpha - \alpha_p$ and $\alpha$ assumed square-integrable, we have
\begin{align*}
D_\alpha \E\left[\alpha_p(X')\left\{Y' - \mu_0(X')\right\}\right](\Delta_\alpha) &= \frac{\partial}{\partial t}\E\left[(\alpha_p + t\Delta_\alpha)(X')\left\{Y' - \mu_0(X')\right\}\right] \\
&= \E\left[\Delta_\alpha(X')\left\{Y' - \mu_0(X')\right\}\right] \\
&= \E\left[\Delta_\alpha(X')\left\{\E(Y' \mid X') - \mu_0(X')\right\}\right] \\
&= 0.
\end{align*}
Thus, the score is Neyman orthogonal, as desired.
\end{proof}

We now prove Proposition~\ref{prop:ortho_irregular}, which details the orthogonality of a de-biased two sample score $m^\beta_p(Z, Z'; g)$ for the smoothed quantity $\theta_{p, \beta} := \E\varphi_{p, \beta}(\phi_a(X))$. This proof is largely the same as the previous one, and we omit arguments that are exactly analogous. The main difference is studying the Gateaux derivative $D_\phi\E[\varphi_{p, \beta}(\phi_a(X))](\Delta_\phi)$ for suitable perturbations $\Delta_\phi$.
\newline
\begin{proof}[Proof of Proposition~\ref{prop:ortho_irregular}]
Seeing that $\theta_{p, \beta} := \E[\varphi_{p, \beta}(\phi_a(X))]$ satisfies $\theta_{p, \beta} = M_p^\beta(g_{p, \beta})$ is straightforward, as the tower rule for conditional expectations  yields that
\[
\E\left[\gamma_{p, \beta}(X)\{\psi(X, X'; \mu_0) - \phi_a(X)\}\right] = 0 \;\; \text{and} \;\; \E\left[\alpha_{p, \beta}(X')\{Y' - \mu_0(X')\}\right] = 0,
\]
and so $M_p^\beta(g_{p, \beta}) = \E[\varphi_{p, \beta}(\phi_a(X))]$. Next, we argue orthogonality. Orthogonality with respect to $\gamma, \alpha,$ and $\mu$ follows from the exact same argument used in the proof of Proposition~\ref{prop:ortho_regular} (replacing $\gamma_p$ with $\gamma_{p, \beta}$ and $\alpha_p$ with $\alpha_{p, \beta}$ in the arguments). Checking orthogonality with respect to $\phi$ is mostly the same as well, but we detail it for completeness. Again letting $\Delta_\phi := \phi - \phi_a$ for an arbitrary direction $\phi \in L^\infty(P_X)$, we have
\begin{align*}
&D_\phi \E[\varphi_{p, \beta}(\phi_a(X))](\Delta_\phi) := \frac{\partial}{\partial t}\E\left[\varphi_{p, \beta}\big((\phi_a + t \Delta_\phi)(X)\big)\right]\Big|_{t = 0} &(\text{Definition of Gateaux Derivative}) \\
&\qquad = \lim_{\substack{t \rightarrow 0 \\ |t| \leq 1}}\E\left[\frac{1}{t}\left\{\varphi_{p, \beta}\big(\phi_a + t \Delta_\phi)(X)\big) - \varphi_{p, \beta}\big(\phi_a(X)\big)\right\}\right] &(\text{Definition of Partial Derivative}) \\
&\qquad = \lim_{\substack{t \rightarrow 0 \\ |t| \leq 1}}\E\left[\varphi_{p, \beta}'\big((\phi_a + \epsilon_{t, X}\Delta_\phi)(X)\big)\right] &(\text{Mean Value Theorem}) \\
&\qquad = \E\left[\lim_{\substack{t \rightarrow 0 \\ |t| \leq 1}}\varphi_{p, \beta}'\big(\phi_a + \epsilon_{t, X}\Delta_\phi)(X)\big)\Delta(X)\right] &(\text{Bounded Convergence Theorem}) \\
&\qquad = \E\left[\varphi_{p, \beta}'\big(\phi_a(X)\big)\Delta_\phi(X)\right] &(\text{Since } \epsilon_{t, X} \xrightarrow[t \rightarrow 0]{} 0),
\end{align*}
where we are able to apply the bounded convergence theorem due to the fact that $Y$ is assumed to be bounded in Assumption~\ref{ass:dgp} (and hence $\phi_a(X)$ is bounded as well). We conclude by noting that, using the same argument as in the proof of Proposition~\ref{prop:ortho_regular}, we have
\begin{align*}
D_\phi\E\left[\gamma_{p, \beta}(X)\{\psi_a(X, X'; \mu_0) - \phi_a(X)\}\right](\Delta_\phi) &= - \E\left[\gamma_{p, \beta}(X)\Delta_\phi(X)\right] = - \E\left[\varphi_{p, \beta}'\big(\phi_a(X)\big)\Delta(X)\right],
\end{align*}
thus completing the proof of orthogonality.
\end{proof}

\subsection{Proofs of Asymptotic Linearity}

Now that we have proven the requisite orthogonality results, we can analyze the asymptotic properties of the U-statistics presented in Section~\ref{sec:shap}. We first consider the asymptotic linearity and normality of our un-smoothed U-statistic, which targets $\theta_p = \E|\phi_a(X)|^p$ for $p \geq 2$. In the sequel, given samples $Z_1, \dots, Z_n \sim P_Z$, we let $\P_{n, 2}f(Z, Z') := \binom{n}{2}^{-1}\sum_{i < j}f(Z_i, Z_j)$ for a function $f : \calZ^2 \rightarrow \R$, and we likewise define $\G_{n, 2} := \sqrt{n}(\P_{n, 2} - \E_{Z, Z'})$
\newline

\begin{proof}[Proof of Theorem~\ref{thm:regular}]
To prove Theorem~\ref{thm:regular}, it suffices to the check the conditions of Theorem~\ref{thm:smooth_clt}, our central limit for smoothed U-statistics. In particular, since the score doesn't not depend on a smoothing parameter and is thus the same across sample sizes, we do not need to check Condition~\ref{cond:score}. Throughout our proof, when we are working with random nuisance estimates $\wh{g}$ (which are assumed independent of the sample), we explicitly subscript our expectations with the random variables that the expectation is taken over, as detailed in Section~\ref{sec:back}. If an expectation lacks a subscript, it is implicitly over all sources of randomness.

\paragraph{Checking Theorem~\ref{thm:smooth_clt}, Condition~\ref{cond:neyman}}

We have already shown the Neyman orthogonality of $m_p(Z; g)$ in Proposition~\ref{prop:ortho_regular}. Thus, by our assumptions on the nuisance estimates $\wh{g} = (\wh{\phi}, \wh{\mu}, \wh{\gamma}, \wh{\alpha})$, we have 
\[
D_g \E_{Z, Z'}\left[m_p(Z, Z'; g^\beta)\right](\wh{g} - g_p) = 0.
\]
Consequently, we have
\begin{align*}
D_g \E_{Z, Z'}\left[h_p(Z, Z'; g_p)\right](\wh{g} - g_p) &= \frac{1}{2}D_g\E_{Z, Z'}\left[m_p(Z, Z'; g_p)\right](\wh{g} - g_p) \\
&\qquad + \frac{1}{2}D_g\E_{Z, Z'}\left[m_p(Z', Z; g_p)\right](\wh{g} - g_p) \\
&= 0,
\end{align*}
which follows from the linearity of the Gateaux derivative.

\paragraph{Checking Theorem~\ref{thm:smooth_clt}, Condition~\ref{cond:hessian}}
Since $h_p(Z, Z'; g) = \frac{1}{2}\left\{m_p(Z, Z'; g) + m_p(Z', Z; g)\right\}$, by symmetry, it suffices to just bound $D_g^2\E[m^\beta_p(Z, Z'; \wb{g})](\wh{g} - g_p, \wh{g} - g_p)$ for any $\wb{g} \in [\wh{g}, g_p]$. We have
\begin{align*}
D_g^2\E_{Z, Z'}[m_p(Z, Z'; \wb{g})](\wh{g} - g_p, \wh{g} - g_p) &= D_\phi^2\E_X[|\wb{\phi}(X)|^p](\wh{\phi} - \phi_a, \wh{\phi} - \phi_a) \\
&\qquad + D_{\gamma, \mu}\E_{X, X'}[\wb{\gamma}(X)\psi_a(X, X'; \wb{\mu})](\wh{\gamma} - \gamma_{p}, \wh{\mu} - \mu_0) \\
&\qquad- D_{\gamma, \phi}\E_{X}[\wb{\gamma}(X)\wb{\phi}(X)](\wh{\gamma} - \gamma_{p}, \wh{\phi} - \phi_a) \\
&\qquad - D_{\alpha, \mu}\E_{X'}[\wb{\alpha}(X')\wb{\mu}(X')](\wh{\alpha} - \alpha_{p}, \wh{\mu} - \mu_0).
\end{align*}

We now proceed to bound each of the terms of the right hand side of the above display separately. First, letting $\wh{\Delta}_\phi := \wh{\phi} - \phi_a$, observe that
\begin{align*}
&D_\phi^2 \E_X[|\wb{\phi}(X)|^p](\wh{\Delta}_\phi) = \frac{\partial^2}{\partial t^2}\E_X\left[\left|(\wb{\phi} + t\wh{\Delta}_\phi)(X)\right|^p\right]\Big|_{t  = 0} &(\text{Definition of Gateaux Derivative})\\
&\qquad = \lim_{\substack{t \rightarrow 0 \\ |t| \leq 1}}\E_X\left[\frac{\big|(\wb{\phi} + t\wh{\Delta}_\phi)(X)\big|^p - 2\big|\wb{\phi}(X)\big|^p + \big|(\wb{\phi} - t\wh{\Delta}_\phi)(X)\big|^p}{t^2}\right] &(\text{Definition of Second Partial}) \\
&\qquad = \lim_{\substack{t \rightarrow 0 \\ |t| \leq 1}}\E_X\left[\wh{\Delta}_\phi(X)^2 p(p - 1)\big|(\wb{\phi} + \epsilon_{t, X}\wh{\Delta}_{\phi})(X)\big|^{p - 2}\right] &(\text{Mean Value Theorem})\\
&\qquad = p(p - 1)\E_X\left[\wh{\Delta}_\phi(X)^2 \lim_{\substack{t \rightarrow 0 \\ |t| \leq 1}}\big|(\wb{\phi} + \epsilon_{t, X}\wh{\Delta}_\phi)(X)\big|^{p - 2}\right] &(\text{Bounded Convergence Theorem}) \\
&\qquad = p(p - 1)\E_X\left[\wh{\Delta}_\phi(X)^2 \left|\wb{\phi}(X)\right|^{p - 2}\right] \\
&\qquad \leq p(p - 1)\left\||\wb{\phi}(X)|^{p - 2}\right\|_{L^\infty(P_X)}\E_X\left[(\wh{\phi}(X) - \phi_a(X))^2\right] &(\text{Holder's Inequality, Definition of } \wh{\Delta}_\phi)\\
&\qquad \lesssim  \|\wh{\phi} - \phi_a\|_{L^2(P_X)}^2 &(\text{Boundedness of } \wb{\phi}) \\
&\qquad = o_\P(n^{-1/2}) &(\text{Nuisance Convergence Rates}).
\end{align*}
In a bit more detail, the second equality above follows as $\frac{\partial^2}{\partial t^2}f(x + th)\big|_{t = 0} = \lim_{t \rightarrow 0}\frac{f(x + th) - 2f(x) + f(x - th)}{t^2}$ for any twice differentiable function $f$. Further, we apply the mean value theorem for second derivatives (applied point-wise), which guarantees the existence of some $X$-dependent value $\epsilon_{t, X} \in [-t, t]$ such that
\[
\big|(\wb{\phi} + t\wh{\Delta}_\phi)(X)\big|^p - 2\big|\wb{\phi}(X)\big|^p + \big|(\wb{\phi} - t\wh{\Delta}_\phi)(X)\big|^p = t^2 \wh{\Delta}_\phi(X)^2p(p-1)|(\wb{\phi} + \epsilon_{t, X}\wh{\Delta}_\phi)(X)|^{p - 2}.
\]
We can apply the bounded convergence theorem above as $\wh{\Delta}_\phi$, $\wb{\phi}$, and $\epsilon_{t, X}$ are all bounded.\footnote{$\wb{\phi}$ is bounded as we have $\wb{\phi} \in [\wh{\phi}, \phi_a]$ and we have assumed $\wh{\phi}$ and $\phi_a$ are bounded.}
The third to last line comes from applying the $L^\infty/L^1$ Holder's inequality. The second to last line follows from the fact $\wb{\phi}(X)$ is almost surely bounded. The final line follows as we have assumed $\|\wh{\phi} - \phi_a\|_{L^2(P_X)} = o_\P(n^{-1/4})$.

\noindent Next, we can directly compute that
\begin{align*}
\left|D_{\gamma, \mu}\E_{X}[\wb{\gamma}(X)\wb{\phi}(X)](\wh{\gamma} - \gamma_{p}, \wh{\phi} - \phi_a)\right| &= \left|\E_{X}[(\wh{\phi} - \phi_a)(X)(\wh{\gamma} - \gamma_{p})(X)]\right| \\
&\leq \|\wh{\gamma} - \gamma_{p}\|_{L^2(P_X)}\|\wh{\phi} - \phi_a\|_{L^2(P_X)} \\
&= o_\P(n^{-1/2}),
\end{align*}
where the first equality follows from the definition of the cross Gateaux derivative, the first inequality follows from Cauchy-Schwarz, and the final line follows as we have assumed $\|\wh{\gamma}-\gamma_{p}\|_{L^2(P_X)}\|\wh{\phi} - \phi_a\|_{L^2(P_X)} = o_\P(n^{-1/2})$.
An exactly analogous argument gives $\left|D_{\alpha, \mu}\E_{X'}[\wb{\alpha}(X')\wb{\mu}(X')](\wh{\alpha} - \alpha_{p}, \wh{\mu} - \mu_0)\right| = o_\P(n^{-1/2})$. Lastly, letting $\wh{\Delta}_\mu := \wh{\mu} - \mu_0$ and $\wh{\Delta}_\gamma := \wh{\gamma} - \gamma_p$, we have
\begin{align*}
&\left|D_{\gamma, \mu}\E_{X, X'}[\wb{\gamma}(X)\psi_a(X, X'; \wb{\mu})](\wh{\Delta}_\gamma, \wh{\Delta}_\mu)\right| \\
&\qquad = \left|D_{\gamma, \mu}\E_{X, X', S}\left[\wb{\gamma}(X)\left\{\wb{\mu}(X_{S \cup a}, X'_{-S \cup a}) - \wb{\mu}(X_S, X'_{-S})\right\}\right](\wh{\Delta}_\gamma, \wh{\Delta}_\mu)\right| \\
&\qquad = \left|\E_{X, X', S}\left[\wh{\Delta}_\gamma(X)\left\{\wh{\Delta}_\mu(X_{S \cup a}, X'_{-S \cup a}) - \wh{\Delta}_\mu(X_S, X'_{-S})\right\}\right]\right| \\
&\qquad \leq  \left(\E_{X}\wh{\Delta}_\mu(X)^2\right)^{1/2}\left\{\left(\E_{X, X', S}\wh{\Delta}_\mu(X_{S \cup a}, X'_{-S \cup a})^2\right)^{1/2} + \left(\E_{X, X', S}\wh{\Delta}_\mu(X_{S}, X'_{-S})^2\right)^{1/2} \right\} \\
&\qquad \leq 2 \left(\E_X \wh{\Delta}_\gamma(X)^2\right)^{1/2}\max_{S \subset [d]}\left(\E_{X, X'} \wh{\Delta}_\mu(X_S, X'_{-S})^2\right)^{1/2} \\
&\qquad = 2 \|\wh{\gamma} - \gamma_{p}\|_{L^2(P_X)}\max_{S \subset [d]}\|\wh{\mu} - \mu_0\|_{L^2(Q_{S, X})} \\
&\qquad = o_\P(n^{-1/2}),
\end{align*}
where the first equality follows from the definition of $\psi_a(X, X'; \mu)$, the first inequality follows from Cauchy-Schwarz, the second inequality follows from replacing the average over subsets by the maximum, and the final line follows from the assumption that, for any $S \subset [d]$, $\|\wh{\gamma} - \gamma_{p, \beta}\|_{L^2(P_X)}\|\wh{\mu} - \mu_0\|_{L^2(Q_{S, X})} = o_\P(n^{-1/2})$. 

Lastly, we check that the equivalent nuisance estimation rates outlined in Remark~\ref{rmk:alternative_rates} also guarantee that the cross-derivative with respect to $\gamma$ and $\mu$ vanishes sufficiently quickly. To do this, it suffices to show in the case of a bounded likelihood ratio $\omega_S(X_S, X_{-S})$ that $\|\wh{\mu} - \mu_0\|_{L^2(Q_{S, X})} \lesssim \|\wh{\mu} - \mu_0\|_{L^2(P_X)}$. To see this, observe that
\begin{align*}
\|\wh{\mu} - \mu_0\|_{L^2(Q_{S, X})} &= \E_{X, X'}\left[\big(\wh{\mu}(X_S, X'_{-S}) - \mu_0(X_S, X'_{-S})\big)^2\right]^{1/2} \\
&= \E_{X}\left[\omega_S(X_S, X_{-S})\big(\wh{\mu}(X) - \mu_0(X)\big)^2\right]^{1/2} \\
&\leq \|\omega_S\|_{L^\infty(P_X)}^{1/2}\|\wh{\mu} - \mu_0\|_{L^2(P_X)},
\end{align*}
where the final inequality follows from an application of the $L^\infty/L^1$ Holder inequality. Lastly, note that in the final line if one instead assume $\omega_S$ has a finite $q$th moment, the $L^q/L^h$ Holder inequality (where $\frac{1}{h} + \frac{1}{q} = 1$) implies
\[
\E_{X}\left[\omega_S(X_S, X_{-S})\big(\wh{\mu}(X) - \mu_0(X)\big)^2\right]^{1/2} \leq \|\omega_{S}\|_{L^q(P_X)}^{1/2}\|\wh{\mu} - \mu_0\|_{L^{2h}(P_X)}.
\]

This thus completes checking Condition~\ref{cond:hessian}

\paragraph{Checking Theorem~\ref{thm:smooth_clt}, Condition~\ref{cond:equi}}

We now proceed to check stochastic equicontinuity. For simplicity, for $g = (\phi, \mu, \gamma, \alpha)$ and $g' = (\phi', \mu', \gamma', \alpha')$, where we assume $\phi, \phi', \mu, \mu', \gamma,$ and $\gamma'$ are bounded and $\alpha, \alpha'$ are square-integrable. Define the functions $\delta_{p}(Z, Z'; g, g') := h_p(Z, Z'; g) - h_p(Z, Z'; g') - \theta_{p}(g) - \theta_{p}(g')$, where we have defined $\theta_{p}(g) := \E[h_p(Z, Z'; g)]$ for any $g$. By construction, $\delta_p$ is symmetric in $Z$ and $Z'$ and has $\E[\delta_p(Z, Z'; g, g')] = 0$. Further define the functions $f_{1, p}(Z; g, g')$ and $f_{2, p}(Z, Z'; g, g')$ by
\begin{align*}
f_{1, p}(z; g, g') &:= \E_{Z, Z'}[\delta_p(Z, Z'; g, g') \mid Z = z] = \E_{Z'}[\delta_p(z, Z'; g, g')] \\
f_{2, p}(z, z'; g, g') &:= \delta_p(z, z'; g, g') - f_{1, p}(z; g, g') - f_{1, p}(z'; g, g').
\end{align*}
We discuss several properties of $f_{1, p}$ and $f_{2, p}$. First, observe that $\E[f_{1, p}(Z; g, g')] = \E[\E(\delta_p(Z, Z'; g, g') \mid Z)] = \E[\delta_p(Z, Z'; g, g')] = 0$. Further, we have
\begin{align*}
\E\left[f_{2, p}(Z, Z'; g, g') \mid Z\right] &= \E\left[\delta_p(Z, Z'; g, g') \mid Z\right] - f_{1, p}(Z; g, g') - \E[f_{1, p}(Z'; g, g')] \\
&= f_{1, p}(Z; g, g') - f_{1, p}(Z; g, g') - 0 \\
&= 0.
\end{align*}
Further, by construction, we have 
\[
\delta_p(z, z'; g, g') = f_{1, p}(z; g, g') + f_{1, p}(z'; g, g') + f_{2, p}(z, z'; g, g').
\]
Thus, defining $\P_{n, 2} f(X_1, X_2) := \binom{n}{2}^{-1}\sum_{i < j} f(X_i, X_j)$, we  note the identity
\begin{align*}
\P_{n, 2}\delta_p(Z, Z'; g, g') &:= \binom{n}{2}^{-1}\sum_{i < j}\delta_p(Z_i, Z_j; g, g') \\
&= \binom{n}{2}^{-1}\sum_{i < j}\left\{f_{1, p}(Z_i; g, g') + f_{1, p}(Z_j; g, g') + f_{2, p}(Z_i, Z_j; g, g')\right\} \\
&= \frac{2}{n(n - 1)}(n - 1)\sum_{i = 1}^n f_{1, p}(Z_i; g, g') + \binom{n}{2}^{-1}\sum_{i < j} f_{2, p}(Z_i, Z_j; g, g') \\
&= \frac{2}{n}\sum_{i = 1}^nf_{1, p}(Z_i; g, g') + \binom{n}{2}^{-1}\sum_{i < j}f_{2, p}(Z_i, Z_j; g, g').
\end{align*}
Putting the aforementioned identities together, we can upper bound $\E\left[\left(\P_{n, 2}\delta_p(Z, Z'; g, g')\right)^2\right]$ as follows:
\begin{align*}
&\E\left[\left(\P_{n, 2} \delta_p(Z, Z'; g, g')\right)^2\right] = \E\left[\left(\frac{2}{n}\sum_{i = 1}^nf_{1, p}(Z_i; g, g') + \binom{n}{2}^{-1}\sum_{i < j}f_{2, p}(Z_i, Z_j; g, g')\right)^2\right] \\
&\qquad = \frac{4}{n}\E\left[f_{1, p}^\beta(Z)^2\right] + \binom{n}{2}^{-1}\E\left[f_{2, p}^\beta(Z, Z'; g, g')^2\right]  \\
&\qquad \qquad + \frac{2}{n}\binom{n}{2}^{-1}\E\left[\left(\sum_{i = 1}^n f_{1, p}(Z_i; g, g')\right)\left(\sum_{j < k}f_{2, p}(Z_i, Z_j; g, g')\right)\right]\\ 
&\qquad = \frac{4}{n^2}\E\left[\left(\sum_{i = 1}^n f_{1, p}(Z_i; g, g')\right)^2\right] + \binom{n}{2}^{-2}\E\left[\left(\sum_{i < j}f_{2, p}(Z_i, Z_j; g, g')\right)^2\right] \\ 
&\qquad = \frac{4}{n}\E\left[f_{1, p}(Z)^2\right] + \binom{n}{2}^{-1}\E\left[f_{2, p}(Z, Z'; g, g')^2\right] \\ 
&\qquad = \frac{4}{n}\E\left[\E\left(\delta_p(Z, Z'; g, g') \mid Z\right)^2\right] + \binom{n}{2}^{-1}\E\left[\left(\delta_p(Z, Z'; g, g') -f_{1, p}(Z; g, g') - f_{1, p}(Z'; g, g')\right)^2\right] \\
&\qquad \leq  \frac{4}{n}\E[\delta_p(Z, Z'; g, g')^2] + \binom{n}{2}^{-1}\E\left\{\E[\delta_p(Z, Z'; g, g')^2] - 2\E[f_{1, p}(Z; g, g')^2]\right\} \\
&\qquad \leq \left(\frac{4}{n} + \frac{2}{n(n - 1)}\right)\E\left[\delta_p(Z, Z'; g, g')^2\right].
\end{align*}
In the above, the first equality follows from the decomposition shown in the above display and the second follows from expanding the square. The third equality follows from the fact $\E[f_{2, p}(Z_i, Z_j; g, g')f_{1, p}(Z_k; g, g')] = 0$ for any $k \in [n]$. The fourth equality follows from the fact that $f_{1, p}(Z)$ is mean zero and that $\E[f_{2, p}(Z, Z'; g, g') \mid Z] = 0$. The fifth equality follows from expanding the definitions of $f_{1, p}$ and $f_{2, p}$. The first inequality follows from an application of Jensen's inequality for conditional expectations along with the fact that
\begin{align*}
\E\left[\left(\delta_p(Z, Z'; g, g') - f_{1, p}(Z; g, g') - f_{1, p}(Z'; g, g')\right)^2\right] &= \E\left[\delta_p(Z, Z'; g, g')^2\right] + 2 \E[f_{1, p}(Z; g, g')^2] \\
&\qquad - 4\E[f_{1, p}(Z; g, g')f_{2, p}(Z, Z'; g, g')] \\
&=\E\left[\delta_p(Z, Z'; g, g')^2\right] - 2\E\left[f_{1, p}(Z; g, g')^2\right].
\end{align*}
The final inequality follows from dropping the term attached to $-2\E[f_{1, p}(Z; g,g')^2]$. With this upper bound in mind, we note that:
\begin{align*}
\E\left[\delta_p(Z, Z'; g, g')^2\right] &= \E\left[\left(h_p(Z, Z'; g) - h_p(Z, Z'; g') - \theta_{p}(g) - \theta_{p}(g')\right)^2\right] \\
&= \Var\left[h_p(Z, Z'; g) - h_p(Z, Z'; g')\right] \\
&\leq \E\left[(h_p(Z, Z'; g) - h_p(Z, Z'; g'))^2\right].
\end{align*}
In aggregate, we have the bound
\begin{align*}
\E\left[\left(\P_{n, 2} h_p(Z, Z'; g) - \P_{n, 2} h_p(Z, Z'; g')\right)^2\right] &\leq \left(\frac{4}{n} + \frac{2}{n(n - 1)}\right)\E\left[(h_p(Z, Z'; g) - h_p(Z, Z'; g'))^2\right].
\end{align*}
Multiplying both sides by $n$ and noting $\G_{n, 2} := \sqrt{n}(\P_{n, 2} - \E_{Z, Z'})$ yields
\begin{align*}
\E\left[\left(\G_{n, 2} h_p(Z, Z'; g) - \G_{n, 2} h_p(Z, Z'; g')\right)^2\right] &\leq \left(4 + \frac{2}{n - 1}\right)\E\left[(h_p(Z, Z'; g) - h_p(Z, Z'; g'))^2\right].
\end{align*}
Next, note that upper bounding $\E[(h_p(Z, Z'; g) - h_p(Z, Z'; g'))^2]$ is equivalent to upper bounding $\E[(m_p(Z, Z'; g) - m_p(Z, Z'; g'))^2]$ up to absolute, multiplicative constants, so we just need to upper bound the latter. Observe that the parallelogram inequality yields that
\begin{align*}
&\E\left[(m_p(Z, Z'; g) - m_p(Z, Z'; g'))^2\right] \lesssim \E\left[\left(\big|\phi(X)\big|^p - \big|\phi'(X)\big|^p\right)^2\right] \\
&\qquad + \E\left[\left(\gamma(X)(\psi_a(X, X'; \mu) - \phi(X)) - \gamma'(X)(\psi_a(X, X'; \mu') - \phi'(X))\right)^2\right] \\
&\qquad + \E\left[\left(\alpha(X')\left\{Y - \mu(X')\right\} - \alpha'(X')\left\{Y - \mu'(X')\right\}\right)^2\right]
\end{align*}

We provide upper bounds for each of the terms on the right hand side of the above expression. First, we have by a first order Taylor expansion with mean value theorem remainder that
\begin{align*}
\E\left[\left(\big|\phi(X)\big|^p - \big|\phi'(X)\big|^p\right)^2\right] &= \E\left[\left(p\big|\wb{\phi}(X)\big|^{p - 1}(\phi(X) - \phi'(X))\right)^2\right] \\
&\leq p^2\Big\||\wb{\phi}(X)|^{2(p - 1)}\Big\|_{L^\infty(P_X)}\|\phi - \phi'\|_{L^2(P_X)}^2 \\
&\lesssim \|\phi - \phi'\|_{L^2(P_X)}^2.
\end{align*}
In the above, the firs inequality follows from applying Holder's inequality for $L^\infty/L^1$ norms, and the final inequality follows from the the fact $\wb{\phi} \in [\phi, \phi']$ is almost surely bounded, which in turn follows from the assumption that $\phi, \phi'$ are absolutely bounded.

Likewise, we have via the parallelogram inequality that
\begin{align*}
&\E\left[\big(\gamma(X)\{\psi_a(X, X'; \mu) - \phi(X)\} - \gamma'(X)\{\psi_a(X, X'; \mu') - \phi'(X)\}\big)^2\right] \\
&\qquad \lesssim \underbrace{\E\left[\left(\gamma(X)\psi_a(X, X'; \mu) - \gamma'(X)\psi_a(X, X'; \mu')\right)^2\right]}_{=:T_1} + \underbrace{\E\left[\left(\gamma(X)\phi(X) - \gamma'(X)\phi'(X)\right)^2\right]}_{=:T_2}.
\end{align*}
First, we can control $T_2$ by noting that
\begin{align*}
T_2 &= \E\left[\big(\gamma(X)\phi(X) \pm \gamma'(X)\phi(X) - \gamma'(X)\phi'(X)\big)^2\right] \\
&\lesssim \E\left[\big(\gamma(X) - \gamma'(X)\big)^2\phi(X)^2\right] + \E\left[\big(\phi(X) - \phi'(X)\big)^2 \gamma'(X)^2\right] \\
&\lesssim \|\gamma - \gamma'\|_{L^2(P_X)}^2 + \|\phi - \phi'\|_{L^2(P_X)}^2,
\end{align*}
where the first equality follows from adding and subtracting $\gamma'(X)\phi(X)$, the the first inequality follows again from the parallelogram inequality, and the final inequality follows since we have assumed $\phi(X)$ and $\gamma'(X)$ are bounded almost surely by absolute constants. 

Next, we can bound $T_1$ similarly. In particular, we have
\begin{align*}
T_1 &= \E\left[\big(\gamma(X)\psi_a(X, X'; \mu) \pm \gamma'(X)\psi(X, X'; \mu) - \gamma'(X)\psi(X, X'; \mu')\big)^2\right] \\
&\lesssim \E\left[\big(\gamma(X) - \gamma'(X)\big)^2 \psi_a(X, X';\mu)^2\right] + \E\left[\gamma'(X)^2\big(\psi_a(X, X'; \mu) - \psi(X, X'; \mu')\big)\right]\\
&\lesssim \|\gamma - \gamma'\|_{L^2(P_X)}^2 + \left\|\psi_a(X, X'; \mu) - \psi_a(X, X'; \mu')\right\|_{L^2(P_{X, X'})}.
\end{align*}
What remains is to upper bound the final term. We have:
\begin{align*}
&\|\psi_a(X, X'; \mu) - \psi_a(X, X'; \mu')\|_{L^2}^2 = \E\left[(\psi_a(X, X'; \mu - \mu'))^2\right] \\
&\qquad = \E\left[\left(\E_S\left[(\mu - \mu')(X_{S \cup a}, X'_{-S \cup a}) - (\mu - \mu')(X_S, X'_{-S})\right]\right)^2\right] \\
&\qquad \leq \E\left[\big((\mu - \mu')(X_{S \cup a}, X'_{-S \cup a}) - (\mu - \mu')(X_S, X'_{-S})\big)^2\right] \\
&\qquad \lesssim \E\left[(\mu - \mu')(X_{S \cup a}, X'_{-S \cup a})^2\right] + \E\left[(\mu - \mu')(X_S, X'_{-S})^2\right]\\
&\qquad = \E\left[\frac{p_{S \cup a}(X_{S \cup a})p_{-S\cup a}(X_{-S \cup a})}{p(X)}(\mu - \mu')(X)^2\right] + \E\left[\frac{p_S(X_S)p_{-S}(X_{-S})}{p(X)}(\mu - \mu')(X)^2\right] \\
&\qquad \leq 2\left(\max_{S \subset [d]}\E\left[\left(\frac{p_S(X_S)p_{-S}( X_{-S})}{p(X)}\right)^2\right]\right)^{1/2}\E\left[(\mu - \mu')(X)^4\right]^{1/2} \\
&\qquad \lesssim \left(\max_{S \subset [d]}\E\left[\left(\frac{p_S(X_S)P_{-S}(X_{-S})}{p(X)}\right)^2\right]\right)^{1/2}\E\left[(\mu - \mu')(X)^2\right]^{1/2} \\
&\qquad \lesssim \|\mu - \mu'\|_{L^2(P_X)}.
\end{align*}
The first equality above follows from the linearity of $\psi_a$ and the second equality from the definition of $\psi_a$. The first inequality follows from Jensen's inequality (pushing the square inside $\E_S$) and the second inequality comes from the parallelogram inequality. The next line comes from a change of measure, and the following from Cauchy-Schwarz and then taking the max over all subsets. The second to last line follows since we have assumed $\mu$ and $\mu'$ are absolutely bounded. The final line follows since we have assumed for any $S \subset [d]$ that $\E\left[\left(\frac{p_S(X_S)p_{-S}(X_{-S})}{p(X)}\right)^2\right] < \infty$.

Thus, we have shown that
\begin{align*}
&\E\left[\big(\G_{n, 2}h_p(Z, Z'; g) - \G_{n, 2}h_p(Z, Z'; g')\big)^2\right] \\
&\qquad \lesssim \underbrace{\|\phi - \phi'\|_{L^2(P_X)}^2 + \|\mu - \mu'\|_{L^2(P_X)}^2  + \|\alpha - \alpha'\|_{L^2(P_X)}^2 + \|\gamma  - \gamma'\|_{L^2(P_X)}^2 + \|\mu - \mu'\|_{L^2(P_X)}}_{=:\Delta(g, g')}
\end{align*}
Further, note that the assumption of nuisance consistency implies that $\Delta(\wh{g}, g_p) = o_\P(1)$. Likewise, since $\wh{\phi}, \wh{\mu}, \wh{\gamma}, \phi_a, \mu_0, \gamma_p$ are bounded by an absolute constant that is independent of the sample size and $\wh{\alpha}, \alpha_p$ has bounded $L^{2 + \epsilon}(P_X)$ norm independent of the sample size, the collections of random variables $(\Delta(\wh{g}_n, g_p))_{n \geq 1}$ is uniformly integrable. Consequently, for any $\epsilon > 0$, we have 
\begin{align*}
&\lim_{n \rightarrow \infty}\P\left(\big|\G_{n, 2}h_p(Z, Z'; \wh{g}) - \G_{n, 2}h_p(Z, Z'; g_p)\big| \geq \epsilon \right) \tag{B}\label{tag:chebyshev}\\
&\qquad = \lim_{n \rightarrow \infty} \E\left[\P_{Z_1, \dots, Z_n}\left(\big|\G_{n, 2}h_p(Z, Z'; g) - \G_{n, 2}h_p(Z, Z'; g')\big|^2 \geq \epsilon \right)\right] \\
&\qquad \leq \frac{1}{\epsilon^2}\lim_{n \rightarrow \infty}E\left[\E_{Z_1, \dots, Z_n}\left[\big(\G_{n, 2}h_p(Z, Z'; g) - \G_{n, 2}h_p(Z, Z'; g')\big)^2 \right]\right] \\
&\qquad \lesssim \lim_{n \rightarrow \infty}\E\left[\Delta(\wh{g}, g_p)\right] \\
&\qquad = 0,
\end{align*}
where the final equality follows from an application of Vitali's theorem (Proposition~\ref{prop:vitali}). Thus, since $\epsilon > 0$, we have shown
\[
\big|\G_{n, 2}h_p(Z, Z'; \wh{g}) - \G_{n, 2}h_p(Z, Z'; g_p)\big| = o_\P(1),
\]
and hence we have stochastic equicontinuity.

\paragraph{Checking Theorem~\ref{thm:smooth_clt}, Condition~\ref{cond:reg}}
Let $\epsilon > 0$ be the constant noted in Assumption~\ref{ass:dgp} and the statement of Theorem~\ref{thm:regular}. Observe that
\begin{align*}
\|h_p(Z, Z'; g_p)\|_{L^{2 + \epsilon}(P_X)} &\leq \left\||\phi_a(X)|^p\right\|_{L^{2 + \epsilon}(P_X)} + \left\|\gamma_p(X)\{\psi_a(X, X'; \mu_0) - \phi_a(X)\}\right\|_{L^{2 + \epsilon}(P_X)} \\
&\qquad + \left\|\alpha_p(X')\{Y'  - \mu_0(X)\}\right\|_{L^{2 + \epsilon}(P_X)} \\
&\leq \||\phi_a(X)|^p\|_{L^\infty(P_X)} + \|\gamma_p(X)\{\psi_a(X, X'; \mu_0) - \phi_a(X)\}\|_{L^\infty(P_X)} \\
&\qquad + \|Y' - \mu_0(X')\|_{L^\infty(P_X)}\|\alpha_p(X')\|_{L^{2 + \epsilon}(P_X)} \\
&\leq D^p + 2(p-1)D^p +  2D^2,
\end{align*}
where again $D > 0$ is the uniform bound on $L^{2 + \epsilon}$ norm of $\alpha_p(X)$ and the $L^\infty$ norm of $Y$ (and hence $\phi_a$ and $\mu_0$). Clearly, since $\gamma_p(X) = (p - 1)\sgn(\phi_a(X))|\phi_a(X)|^{p - 1}$, we have $\|\gamma_p\|_{L^\infty(P_X)} \leq (p - 1)D^{p - 1}$. Thus, the $L^{2 + \epsilon}$ norm of $h_p(Z, Z'; g_p)$ exists, and we have finished checking the condition.

\paragraph{Computing the Influence Function}
Since we have already verified the conditions of Theorem~\ref{thm:smooth_clt}, we know that $\sqrt{n}(\wh{\theta}_n - \theta_p) = \frac{1}{\sqrt{n}}\sum_{i = 1}^n \rho_{\theta_p}(Z_i) + o_\P(1)$ for some limiting influence function $\rho_{\theta_p}(z)$. What remains is to derive said influence function.

Also per the conclusion of Theorem~\ref{thm:smooth_clt}, we know that $\rho_{\theta_p}(z) = 2\left\{\E\left[h_p(Z, Z'; g_p) \mid Z = z\right] - \theta_p\right\}$ (since we are dealing with an order-2 U-statistic). We just need to compute the conditional expectation in order to complete the proof. By linearity of expectation, we have
\[
\E\left[h_p(Z, Z'; g_p) \mid Z = z\right] = \frac{1}{2}\underbrace{\E\left[m_p(Z, Z'; g_p) \mid Z = z\right]}_{=: T_1} + \frac{1}{2}\underbrace{\E\left[m_p(Z', Z; g_p) \mid Z = z\right]}_{=:T_2}
\]

Since $Z$ and $Z'$ are assumed independent, we can write 
\begin{align*}
T_1 &= \E\left[m_p(z, Z'; g_p)\right] \\
&= |\phi_a(x)|^p + p\sgn(\phi_a(x))\left\{\E[\psi_a(x, X'; \mu_0)] - \phi_a(X)\right\} + \E\left[\alpha_p(X')\{Y' - \mu_0(X')\}\right] \\
&= |\phi_a(x)|^p,
\end{align*}
where the second equality comes from expanding the definition of $m_p(Z, Z'; g_p)$ and the third equality follows from the fact that (a) $\E[\psi_a(x, X'; \mu_0)] = \phi_a(x)$ and (b) $\E\left[\alpha_p(X')\{Y' - \mu_0(X')\}\right] = 0$.

Likewise, we have
\begin{align*}
T_2 &= \E\left[m_p(Z, z'; g_p)\right] \\
&= \E\left|\phi_a(X')\right|^p + \E\left[\gamma_p(X')\left\{\psi_a(X', x; \mu_0) - \phi_a(X')\right\}\right] + \alpha_p(x)\{y - \mu_0(x)\} \\
&= \theta_p + \Lambda_p(x) - p\theta_p + \alpha_p(x)\{y - \mu_0(x)\},
\end{align*}
where $\Lambda_p(x) := \E\left[\gamma_p(X')\psi_a(X', X; \mu_0) \mid X = x\right]$. Putting things together, it follows that the influence function is given by
\[
\rho_{\theta_p}(z) = |\phi_a(x)|^p  + \Lambda_P(x) + \alpha_p(x)\{y - \mu_0(x)\} - (p + 1)\theta_p,
\]
which is precisely the claimed influence function in the theorem statement. Thus, we have completed the proof.
\end{proof}

Now, we prove the convergence of our smoothed estimator to the population $p$th SHAP power $\theta_{p}$. The structure of this proof is largely similar to the proof of Theorem~\ref{thm:regular} given immediately above. However, there are two key deviations. First, in the proof above, we studied an estimator that directly targeted the population $p$th SHAP power $\theta_p$. Here, we target the smoothed surrogate $\theta_{p, \beta} := \E[\varphi_{p, \beta}(\phi_a(X))]$. Thus, we must use the bias lemma (Lemma~\ref{lem:bias}) to show that distance between $\theta_{p, \beta}$ and $\theta_p$ given the choice of smoothing parameter is sufficiently large. Next, given the second derivative of the smoothed surrogate grows with $\beta$, we must ensure $\beta$ is not chosen to be too large in terms of the sample in order to obtain asymptotic linearity/normality.

\begin{proof}[Proof of Theorem~\ref{thm:irregular}]
Recalling that $\theta_{p, \beta}$ is defined as $\theta_{p, \beta} := \E\left|\varphi_{p, \beta}\big(\phi_a(X)\big)\right|$, we can write
\[
\sqrt{n}(\wh{\theta} - \theta_{p}) = \sqrt{n}(\wh{\theta} - \theta_{p, \beta}) + \sqrt{n}(\theta_{p, \beta} - \theta_p).
\]
We show that $\sqrt{n}(\theta_{p, \beta} - \theta_p) = o(1)$ under the assumptions of Theorem~\ref{thm:irregular}. In particular, using the bias bound outlined in Lemma~\ref{lem:bias} and the assumed nuisance estimation rates, we have
\begin{align*}
\sqrt{n}(\theta_{p, \beta} - \theta_p) &\lesssim \sqrt{n}\left(\frac{1}{\beta}\right)^{\frac{p + \delta}{2 - p}} &(\text{Lemma~\ref{lem:bias}})\\
&\lesssim \sqrt{n}\left(n^{-\frac{2 - p}{2(p + \delta)}}\right)^{\frac{p + \delta}{2 - p}} &\Big(\beta_n = \omega\big(n^{\frac{2 - p}{2(p + \delta)}}\big)\Big)\\
&= o(1) 
\end{align*}
Thus, we have 
\[
\sqrt{n}(\wh{\theta} - \theta_p) = \sqrt{n}(\wh{\theta} - \theta_{p, \beta}) + o(1).
\]
To complete the proof of the theorem, we just need to check the asymptotic linearity of $\sqrt{n}(\wh{\theta} - \theta_{p, \beta})$. We do this by checking the conditions of Theorem~\ref{thm:smooth_clt}.

\paragraph{Checking Theorem~\ref{thm:smooth_clt}, Condition~\ref{cond:neyman}}

We have already shown that $D_g\E[m_p^\beta(Z, Z'; g_\beta)](\wh{g} - g_{p, \beta}) = 0$ in Proposition~\ref{prop:ortho_irregular}. The orthogonality of $\E[h_p^\beta(Z, Z'; g_{p, \beta})]$ consequently follows from linearity.


\paragraph{Checking Theorem~\ref{thm:smooth_clt}, Condition~\ref{cond:hessian}}

As in the proof of Theorem~\ref{thm:regular} above, since $h_p^\beta(Z, Z'; g)$ is given by $h_p^\beta(Z, Z'; g) = \frac{1}{2}\left\{m_p^\beta(Z, Z'; g) + m_p^\beta(Z', Z; g)\right\}$, it suffices to just bound $D_g\E[m^\beta_p(Z, Z'; \wb{g})](\wh{g} - g_\beta)$ for $\wb{g} \in [\wh{g}, g_\beta]$. We have
\begin{align*}
D_g^2\E_{Z, Z'}[m^\beta_p(Z, Z'; \wb{g})](\wh{g} - g_\beta) &= D_\phi^2\E[\varphi_{p, \beta}(\wb{\phi}(X))](\wh{\phi} - \phi_0) \\
&\qquad + D_{\gamma, \mu}\E_{X, X'}[\wb{\gamma}(X)\psi_a(X, X'; \wb{\mu})](\wh{\gamma} - \gamma_{p, \beta}, \wh{\mu} - \mu_0) \\
&\qquad- D_{\gamma, \phi}\E_{X}[\wb{\gamma}(X)\wb{\phi}(X)](\wh{\gamma} - \gamma_{p, \beta}, \wh{\phi} - \phi_0) \\
&\qquad - D_{\alpha, \mu}\E_{X'}[\wb{\alpha}(X')\wb{\mu}(X')](\wh{\alpha} - \alpha_{p, \beta}, \wh{\mu} - \mu_0)
\end{align*}
We note that the final three terms (the cross-derivatives) are all $o_\P(n^{-1/2})$ by exactly the same argument used in the proof of Theorem~\ref{thm:regular}. What remains is to show that $D_\phi^2\E_X\left[\varphi_{p, \beta}\big(\wb{\phi}(X)\big)\right](\wh{\Delta}_\phi) = o_\P(n^{-1/2})$, where we have defined $\wh{\Delta}_\phi := \wh{\phi} - \phi_a$ for ease of notation. We have:
\begin{align*}
&D_\phi^2 \E_X\left[\varphi_{p, \beta}\big(\wb{\phi}(X)\big)\right](\wh{\Delta}_\phi) = \frac{\partial^2}{\partial t^2}\E_X\left[\varphi_{p,\beta}\big((\wb{\phi} + t\wh{\Delta}_\phi)(X)\big)\right]\Big|_{t  = 0} &(\text{Def.\ of Gateaux Derivative})\\
&\; = \lim_{\substack{t \rightarrow 0 \\ |t| \leq 1}}\E_X\left[\frac{\varphi_{p, \beta}\big((\wb{\phi} + t\wh{\Delta}_\phi)(X)\big) - 2\varphi_{p, \beta}\big(\wb{\phi}(X)\big) + \varphi_{p, \beta}\big((\wb{\phi} - t\wh{\Delta}_\phi)(X)\big)}{t^2}\right] &(\text{Definition of Second Partial}) \\
&\; = \lim_{\substack{t \rightarrow 0 \\ |t| \leq 1}}\E_X\left[\wh{\Delta}_\phi(X)^2 \varphi_{p, \beta}''\big((\wb{\phi} + \epsilon_{t, X}\wh{\Delta}_{\phi})(X)\big)\right] &(\text{Mean Value Theorem})\\
&\; \leq \sup_{u \in \R}|\varphi_{p, \beta}''(u)|\E_X\left[\left(\wh{\phi}(X) - \phi_a(X)\right)^2\right] &(L^\infty/L^1\text{ Holder's Inequality}) \\
&\; \lesssim \beta \E_X\left[\left(\wh{\phi}(X) - \phi_a(X)\right)^2\right] &(\text{Lemma~\ref{lem:smooth}}) \\
&\; = \beta \|\wh{\phi} - \phi_a\|_{L^2(P_X)}^2  \\
&\; = o_\P(n^{-1/2}) &(\text{Nuisance  Rates}).
\end{align*}
Thus, we see that $D_g\E_{Z, Z'}\left[h^\beta_p(Z, Z'; \wb{g})\right](\wh{g} - g_{p, \beta}) = o_\P(n^{-1/2})$, so we have completed checking this condition.

\paragraph{Checking Theorem~\ref{thm:smooth_clt}, Condition~\ref{cond:score}}
Again by the symmetry of $h_p^\beta(Z, Z'; g)$, it suffices to show that $m_p^\beta(Z, Z'; g_\beta) \xrightarrow[n \rightarrow \infty]{\P} m_p^\ast(Z, Z'; g_0)$. Since
\[
m^\beta_p(Z, Z'; g_{p,\beta}) = \varphi_{p, \beta}(\phi_a(X)) + \gamma_{p, \beta}(X)\left\{\psi_a(X, X'; \mu_0) - \phi_a(X)\right\} + \alpha_{p, \beta}(X')\left\{Y' - \mu_0(X')\right\},
\]
and 
\[
m^\ast(Z, Z'; g_p) = |\phi_a(X)|^p + \gamma_{p}(X)\left\{\psi_a(X, X'; \mu_0) - \phi_a(X)\right\} + \alpha_p(X')\left\{Y' - \mu_0(X')\right\}
\]
if suffices to show that each of the three terms on the right hand side of the first expression converge in probability to their respective counterparts on the right hand side of the second expression.

First, note that we have $\varphi_{p, \beta_n}\big(\phi_a(X)\big) \xrightarrow[n \rightarrow \infty]{a.s.} |\phi_a(X)|$, as on the event $\{\phi_a(X) \neq 0\}$ we have
\begin{align*}
\lim_{n \rightarrow \infty}\varphi_{p, \beta_n}\big(\phi_a(X)\big) &= \lim_{n \rightarrow \infty}|\phi_a(X)|^p\tanh(\beta_n|\phi_0(X)|^{2 - p}) \\
&= |\phi_a(X)|^p.
\end{align*}
Likewise, on the event $\{\phi_a(X) = 0\}$ we trivially have $\varphi_{p, \beta_n}\big(\phi_a(X)\big) = |\phi_a(X)|^p$ for all $n \geq 1$. Next, note that on $\{\phi_a(X) \neq 0\}$ we have 
\begin{align*}
    \lim_{n \rightarrow \infty}\gamma_{p, \beta}(X) &= \lim_{n \rightarrow \infty}\varphi_{p, \beta}'(\phi_a(X)) \\
    &=\lim_{n \rightarrow \infty}\sgn(\phi_a(X))(2 - p)\beta|\phi_a(X)|\sech(\beta|\phi_a(X)|^{2 - p}) \\
    &\qquad + \lim_{n \rightarrow \infty}\sgn(\phi_a(X))p|\phi_a(X)|^{p - 1}\tanh(\beta|\phi_a(X)|^{2- p}) \\
    &= 0 + \lim_{n \rightarrow \infty}\sgn(\phi_a(X))p|\phi_a(X)|^{p - 1}\tanh(\beta|\phi_a(X)|^{2- p}) \\
    &= p\sgn(\phi_a(X))|\phi_a(X)|^{p - 1}
\end{align*}
and on $\{\phi_a(X) = 0\}$ we have $\varphi'_{p, \beta}(\phi_a(X)) = 0 = p\sgn(\phi_a(X))|\phi_a(X)|^{p - 1}$ (under the convention that $\sgn(0) = 0$, which is relevant when $p = 1$). Thus, we have 
\[
\gamma_{p, \beta_n}(X)\{\psi_a(X, X'; \mu_0) - \phi_a(X)\} \xrightarrow[n \rightarrow \infty]{a.s.} \gamma_p(X)\{\psi_a(X, X'; \mu_0) - \phi_a(X)\}.
\]
What remains is to argue that $\alpha_{p, \beta_n}(X')\{Y' - \mu_0(X')\} \xrightarrow[n \rightarrow \infty]{\P} \alpha_p(X')\{Y' - \mu_0(X')\}$. To do this, it clearly suffices to argue that $\alpha_{p, \beta_n}\big(X') \xrightarrow[n \rightarrow \infty]{\P} \alpha_p(X')$. Recall for any $1 \leq p < 2$ and $\beta > 0$, we have the identities
\begin{align*}
\alpha_{p, \beta}(X') &= \E_S\left[\gamma_{p, \beta}^{S \cup a}(X_{S \cup a}')\omega_{S \cup a}(X') - \gamma_{p, \beta}^S(X_S')\omega_S(X')\right] \\
\alpha_p(X') &= \E_S\left[\gamma_{p}^{S \cup a}(X_{S \cup a}')\omega_{S \cup a}(X') - \gamma_{p}^S(X_S)\omega_S(X')\right],
\end{align*}
where for any $S \subset [d]$, $\gamma_{p, \beta}^S(X_S) := \E\left(\gamma_{p, \beta}(X) \mid X_S\right)$, $\gamma_p^S(X_S) := \E(\gamma_p(X) \mid X_S)$, and $\omega_S(X) = \frac{p_S(X_S)p_{-S}(X_{-S})}{p(X)}$. Thus, it suffices to argue that $\gamma_{p, \beta_n}^S(X_S)  = \gamma_p^S(X_S) + o_\P(1)$ for all $S \subset [d]$. We show the stronger result that $\gamma_{p, \beta_n}^S(X'_S) \xrightarrow[n \rightarrow \infty]{L^1} \gamma_p^S(X'_S)$.

Since $\gamma_{p, \beta_n}$ and $\gamma_p$ are both absolutely bounded (which follows from the assumption that $Y$ is bounded), we have via the bounded convergence theorem that
\begin{align*}
\lim_{n \rightarrow \infty}\E\left|\gamma_{p, \beta_n}^S(X'_S) - \gamma_p^S(X'_S)\right| &= \lim_{n \rightarrow \infty}\E\left|\E(\gamma_{p, \beta}(X) \mid X_S) - \E(\gamma_p(X) \mid X_S)\right| &(\text{Def. } \gamma_{p, \beta_n}^S, \gamma_p^S)\\
&= \lim_{n \rightarrow \infty}\E\left|\E(\gamma_{p, \beta}(X) - \gamma_p(X) \mid X_S))\right|  &(\text{Linearity of Expectation})\\
&\leq \lim_{n \rightarrow \infty}\E\left[\E\left(\left|\gamma_{p, \beta}(X) - \gamma_p(X)\right| \mid X_S\right)\right] &(\text{Jensen's Inequality})\\
&= \lim_{n \rightarrow \infty}\E\left|\gamma_{p, \beta}(X) - \gamma_p(X)\right| &(\text{Tower Rule})\\
&= 0 &(\text{Bounded Convergence}).
\end{align*}
This thus completes the proof.

\paragraph{Checking Theorem~\ref{thm:smooth_clt}, Condition~\ref{cond:equi}}

Checking this condition is largely analogous to checking this condition in the proof of Theorem~\ref{thm:regular}. 
Using the same argument as used in checking Condition~\ref{cond:equi} in the proof of Theorem~\ref{thm:regular}, we have the bound
\begin{align*}
&\E\left[\left(\G_{n, 2}h_p^\beta(Z, Z'; g) - \G_{n, 2}h_p^\beta(Z, Z'; g')\right)^2\right] \leq \left(4 + \frac{2}{n - 1}\right)\E\left[\left(h_p^\beta(Z, Z'; g) - h_p^\beta(Z, Z'; g')\right)^2\right] \\
&\qquad \lesssim \E\left[(m_p^\beta(Z, Z'; g) - m_p^\beta(Z, Z'; g'))^2\right] \\
&\qquad \lesssim \E\left[\left(\varphi_{p, \beta}\big(\phi(X)\big) - \varphi_{p, \beta}\big(\phi'(X)\big)\right)^2\right] \\
&\qquad \qquad + \E\left[\left(\gamma(X)(\psi_a(X, X'; \mu) - \phi(X)) - \gamma'(X)(\psi_a(X, X'; \mu') - \phi'(X))\right)^2\right] \\
&\qquad \qquad + \E\left[\left(\alpha(X')\left\{Y - \mu(X')\right\} - \alpha'(X')\left\{Y - \mu'(X')\right\}\right)^2\right]
\end{align*}
where the second line follows from the parallelogram inequality and the fact that $4 + \frac{2}{n - 1} \leq 6$ for $n \geq 2$, and the third line again follows from a repeated application of the parallelogram inequality.

Using the same argument as in the proof of Theorem~\ref{thm:regular}, it follows that
\begin{align*}
&\E\left[\left(\gamma(X)(\psi_a(X, X'; \mu) - \phi(X)) - \gamma'(X)(\psi_a(X, X'; \mu') - \phi'(X))\right)^2\right] \\
&\qquad \lesssim \|\gamma - \gamma'\|_{L^2(P_X)}^2 + \|\phi - \phi'\|_{L^2(P_X)}^2 + \|\mu - \mu'\|_{L^2(P_X)}
\end{align*}
and 
\[
\E\left[\left(\alpha(X')\left\{Y - \mu(X')\right\} - \alpha'(X')\left\{Y - \mu'(X')\right\}\right)^2\right] \lesssim \|\alpha - \alpha'\|_{L^2(P_X)}^2  + \|\mu - \mu'\|_{L^2(P_X)}^2.
\]
Similarly, applying a first order Taylor expansion with mean value theorem remainder as before, we have that:
\begin{align*}
\E\left[\left(\varphi_{p, \beta}\big(\phi(X)\big) - \varphi_{p, \beta}\big(\phi'(X)\big)\right)^2\right]  &= \E\left[\left(\varphi_{p, \beta}'\big(\wb{\phi}(X)\big)(\phi(X) - \phi'(X))\right)^2\right] \\
&\leq \Big\|\varphi_{p, \beta}'\big(\wb{\phi}(X)\big)\Big\|_{L^\infty(P_X)}\|\phi - \phi'\|_{L^2(P_X)}^2 \\
&\lesssim \|\phi - \phi'\|_{L^2(P_X)}^2,
\end{align*}
where the final inequality follows as $\wb{\phi} \in [\phi, \phi']$ is almost surely bounded. Thus, we again get
\begin{align*}
&\E\left[\big(\G_{n, 2}h_p^\beta(Z, Z'; g) - \G_{n, 2}h_p^\beta(Z, Z'; g')\big)^2\right] \\
&\qquad \lesssim \underbrace{\|\phi - \phi'\|_{L^2(P_X)}^2 + \|\mu - \mu'\|_{L^2(P_X)}^2  + \|\alpha - \alpha'\|_{L^2(P_X)}^2 + \|\gamma  - \gamma'\|_{L^2(P_X)}^2 + \|\mu - \mu'\|_{L^2(P_X)}}_{=:\Delta(g, g')}.
\end{align*}
Setting $g = \wh{g}$ and $g' = g_{p, \beta}$ and applying Chebyshev's inequality conditionally as in Line~\ref{tag:chebyshev} yields that
\[
\left|\G_{n, 2}h_p^\beta(Z, Z'; \wh{g}) - \G_{n, 2}h_p^\beta(Z, Z'; g_{p, \beta})\right| = o_\P(1).
\]

\paragraph{Checking Theorem~\ref{thm:smooth_clt}, Condition~\ref{cond:reg}}
Showing $\left\|m_p^\beta(Z, Z'; g_{p, \beta}\right\|_{L^{2 + \epsilon}(P_Z)} < \infty$ where $\epsilon > 0$ is as in Assumption~\ref{ass:dgp} is exactly analogous to the argument made while Checking~\ref{cond:reg} in the proof of Theorem~\ref{thm:regular} above.

\paragraph{Computing the Influence Function}
Given that $m^\beta_p(Z, Z'; g_{p, \beta}) \xrightarrow[\beta \rightarrow \infty]{\P} m_p(Z, Z'; g_p)$, where $m_p(Z, Z'; g_p)$ is as given in Section~\ref{sec:shap:regular}, the derivation of the influence function is exactly analogous to the derivation presented at the end of the proof of Theorem~\ref{thm:regular} a few paragraphs above. Thus, since there is nothing new to show here, we have completed the proof.
\end{proof}

%% file: appendix/smooth_clt.tex
\section{A Smooth CLT for U-Statistics}
\label{app:clt}
In this appendix, we present a central limit theorem for smoothed U-statistics.  We allow these U-statistics to depend on infinite-dimensional nuisances parameters that must be estimated from the data, but assume for any for any smoothing fixed smoothing parameter that they are Neyman orthogonal in the sense of Definition~\ref{def:neyman}. We show that, so long as these scores are sufficiently regular and converge in probability to some limiting score, then asymptotically it appears as if the limiting score had been used to construct the U-statistic in the first place. 

More concretely, we let $Z \in \calZ$ denote a generic observation, where $\calZ$ is some measurable space. We assume $Z$ has distribution $P_Z$. For any $\beta > 0$, there is some score $m^\beta(Z_1, \dots, Z_r; g)$ that depends on $r$ samples and a nuisance component $g \in L^2(P_W)$, where $W \subset Z$ takes values in a measurable space $\calW \subset \calZ$. We denote the distribution of $W$ as $P_W$. Our goal is to perform inference on the quantity $\theta_\beta := \E\left[m^\beta(Z_1, \dots, Z_r; g_\beta)\right]$ where $g_\beta \in L^2(P_W)$ denotes some true nuisance associated with $m^\beta$ and $Z_1, \dots, Z_r$ are assumed to be i.i.d.\ from $P_Z$. More specifically, we want to show that if $m^\beta(Z_1,\dots, Z_r; g^\beta) \xrightarrow[\beta \rightarrow \infty]{\P} m^\ast(Z_1, \dots, Z_r; g_0)$ for some limiting score $m^\ast$ and limiting nuisance $g_0 \in L^2(P_W)$, then we $\sqrt{n}(\wh{\theta}_n - \theta_\beta) \Rightarrow \calN(0, \Sigma^\ast)$, where $\Sigma^\ast > 0$ is some appropriate asymptotic variance dependent on $m^\ast$ and $\wh{\theta}_n$ is the plug-in U-statistic. The following theorem, which combines ideas from the asymptotic study of U-statistics~\citep{van2000asymptotic} with similar results on smoothed M-estimation~\citep{whitehouse2025inference, chen2023inference}, accomplishes a stronger result, proving asymptotic linearity instead.
\begin{theorem}
\label{thm:smooth_clt}
Suppose, for any $\beta > 0$, there is a score $m^\beta(z_1, \dots, z_r; g)$ depending on observations $z_1, \dots, z_r \in \calZ$ and function $g : \calW \rightarrow \R^k$ with $\calW \subset \calZ$. Further, suppose $m^\beta(z_1, \dots, z_r; g)$ is symmetric in $z_1, \dots, z_r$ for any $g$. Define $\theta_\beta := \E\left[m^\beta(Z_1, \dots, Z_r; g^{\beta})\right]$, where the expectation is over i.i.d.\ samples $Z_1, \dots, Z_r \sim P_Z$ and $g_\beta \in L^2(P_W)$ is some true nuisance parameter. 

Let $Z_1, \dots, Z_n$ be i.i.d.\ samples from $P_Z$, $\beta_n = \omega(1)$ a smoothing parameter, and $\wh{g}_n$ be a nuisance estimate that is independent of $Z_1, \dots, Z_n$. Define the plug-in U-statistic
\[
\wh{\theta}_n := \binom{n}{r}^{-1}\sum_{i \in \calC_r} m^{\beta_n}(Z_{i_1}, \dots, Z_{i_r}; \wh{g}_n).
\]
Suppose the following hold:
\begin{enumerate}
    \item \textit{(Neyman Orthogonality)} For any $\beta > 0$, the score $m^\beta$ is Neyman orthogonal in the sense that
    \[
    D_g \E_{Z_1,\dots,Z_r}\left[m^\beta(Z_1, \dots, Z_r; g_\beta)\right](\wh{g} - g_\beta) = 0.
    \]
   \label{cond:neyman}
    \item \textit{(Second Order Gateaux Derivatives)} The second-order Gateaux derivatives exist and shrink sufficiently quickly, i.e.\ $\forall \wb{g} \in [g^{\beta_n}, \wh{g}_n]$
    \[
    D_g^2\E_{Z_1, \dots, Z_r}\left[m^\beta(Z_1, \dots, Z_r; \wb{g})\right](\wh{g}_n - g_{\beta_n}, \wh{g}_n - g_{\beta_n}) = o_\P(n^{-1/2}). 
    \]
    \label{cond:hessian}
    \item \textit{(Convergence of Score)} We have, as $\beta \rightarrow \infty$,
    \[
    m^{\beta}(Z_1, \dots, Z_r; g_\beta) \xrightarrow[]{\P} m^\ast(Z_1, \dots, Z_r; g_0)
    \]
    for some limiting score $m_\ast$ and some nuisance $g_0$.\label{cond:score}
    \item \textit{(Stochastic Equicontinuity)} We have 
    \[
    |\G_{n,r} m^{\beta_n}(Z_1, \dots,Z_r; \wh{g}_n) - \G_{n,r} m^{\beta_n}(Z_1,\dots, Z_r; g_{\beta_n})| = o_\P(1),
    \]
    where $\G_{n,r} := \sqrt{n}(\P_{n, r} - \E_{Z_1, \dots, Z_r})$, $\P_{n, r}f(Z_1, \dots, Z_r) := \binom{n}{r}^{-1}\sum_{i \in \calC_r}f(Z_{i_1}, \dots, Z_{i_r})$, and $\calC_r := \binom{[n]}{r}$ is the collection of all subsets of $[n]$ of size $r$.\label{cond:equi}
    \item \textit{(Regularity)} There exists $D > 0$ and $\epsilon > 0$ such that, for any $\beta > 0$, we have
    \[
    \left\|m^\beta(Z_1, \dots, Z_r; g_\beta)\right\|_{L^{2 + \epsilon}(P_Z^{\otimes r})} \leq D.
    \]\label{cond:reg}
    Further, we assume $\left\|m^\ast(Z_1, \dot, Z_r; g^\ast)\right\|_{L^{2 + \epsilon}(P_Z^{\otimes r})} \leq D$ as well.
\end{enumerate}
Then, letting $\theta_0 := \E\left[m^\ast(Z_1, \dots, Z_r; g_0)\right]$, we have the following asymptotic linearity result:
\[
\sqrt{n}(\wh{\theta}_n - \theta_{\beta_n}) = \frac{1}{\sqrt{n}}\sum_{i = 1}^n \rho_\theta(Z_i)  + o_\P(1),
\]
where $\rho_\theta(z) := r\left\{\E[m^\ast(Z_1, Z_2, \dots, Z_r; g_0) \mid Z_1 = z] - \theta_0\right\}$. In particular, this implies
\[
\sqrt{n}(\wh{\theta}_n - \theta_{\beta_n}) \Rightarrow \calN(0, \Sigma^\ast)
\]
whenever $\Sigma^\ast := r^2\Var[\rho_\theta(Z)]$ satisfies $\Sigma^\ast > 0$.
\end{theorem}

\begin{proof}[Proof of Theorem~\ref{thm:smooth_clt}]
First, we note that we have the following decomposition.
\begin{align*}
\sqrt{n}(\wh{\theta} - \theta_{\beta}) &= \sqrt{n}\left\{\P_{n, r} m^{\beta}(Z_1, \dots, Z_r; \wh{g}) - \E_{Z_1, \dots, Z_r} m^{\beta}(Z_1, \dots, Z_r; g_\beta)\right\} \\
&= \sqrt{n}\Big\{\P_{n, r} m^{\beta}(Z_{1}, \dots, Z_{r}; \wh{g}) \pm \P_{n, r} m^{\beta}(Z_{1},\dots, Z_{r}; g_{\beta}) \pm \E_{Z_1, \dots, Z_r} m^{\beta}(Z_1, \dots, Z_r; \wh{g}) \\
&\qquad\qquad  \pm \E_{Z_1, \dots, Z_r} m^{\beta}(Z_1, \dots, Z_r; g_{\beta}) - \E_{Z_1, \dots, Z_r} m^{\beta}(Z_1, \dots, Z_r; g^{\beta})\Big\} \\
&= \sqrt{n}\Big\{\P_{n ,r} m^{\beta}(Z_{1}, \dots, Z_{r}; \wh{g}) - \E_{Z_1, \dots, Z_r} m^{\beta}(Z_1, \dots, Z_r; \wh{g})  \\
&\qquad \qquad - \P_{n, r} m^{\beta}(Z_{1}, \dots, Z_{r}; g_{\beta}) + \E_{Z_1, \dots, Z_r} m^{\beta}(Z_1, \dots, Z_r; g_\beta)\Big\} \\
&\qquad\qquad + \sqrt{n}\Big\{\P_{n, r} m^{\beta}(Z_{1}, \dots, Z_{r}; g_{\beta}) - \E_{Z_1, \dots, Z_r} m^{\beta}(Z_1, \dots, Z_r; g_{\beta})\Big\} \\
&\qquad \qquad + \sqrt{n}\left\{\E_{Z_1, \dots, Z_r} m^{\beta}(Z_1, \dots, Z_r; \wh{g}) - \E_{Z_1, \dots, Z_r} m^{\beta}(Z_1, \dots, Z_r; g_{\beta})\right\} \\
&= \underbrace{\G_{n, r} m^{\beta}(Z_{1}, \dots, Z_{r}; \wh{g}) - \G_{n,r} m^{\beta}(Z_{1}, \dots, Z_{r}; g_{\beta})}_{T_1} + \underbrace{\G_{n,r} m^{\beta}(Z_{1}, \dots, Z_{r}; g_{\beta})}_{T_2} \\
&\qquad \qquad +\underbrace{\sqrt{n}\left(\E_Z m^{\beta}(Z_1, \dots, Z_r; \wh{g}) - \E_{Z_1, \dots, Z_r} m^{\beta}(Z_1, \dots, Z_r; g_{\beta})\right)}_{T_3}.
\end{align*}
We handle each of the three terms separately. Handling the first term follows by assumption, as we have by Condition~\ref{cond:equi}
\[
T_1 = \G_{n, r} m^{\beta}(Z_{1}, \dots, Z_{r}; \wh{g}) - \G_{n,r} m^{\beta}(Z_{1}, \dots, Z_{r}; g_{\beta}) = o_\P(1).
\]
Handling the second term follows from the assumptions regarding Neyman orthogonality and bounded second order errors. In particular, performing a second-order 
\begin{align*}
\frac{1}{\sqrt{n}}T_3 &= \E_{Z_1, \dots, Z_r} m^{\beta}(Z_1, \dots, Z_r; \wh{g}) - \E_{Z_1, \dots, Z_r} m^{\beta}(Z_1, \dots, Z_r; g_{\beta}) \\
&= D_{g}\E_{Z_1, \dots, Z_r}[m^{\beta}(Z_1, \dots, Z_r; g_{\beta})](\wh{g} - g_{\beta}) + \frac{1}{2}D^2_g\E_{Z_1,\dots, Z_r}[m^{\beta}(Z_1, \dots, Z_r; \wb{g})](\wh{g} - g_{\beta}, \wh{g} - g_\beta) \\
&= \frac{1}{2}D^2_g\E_{Z_1, \dots, Z_r}[m^{\beta}(Z_1, \dots, Z_r; \wb{g})](\wh{g} - g_{\beta}) \qquad \qquad \qquad (\text{Condition~\ref{cond:neyman}}) \\
&= o_\P(n^{-1/2}) \quad \qquad \qquad \qquad \qquad \qquad \qquad \qquad \qquad \qquad (\text{Condition~\ref{cond:hessian}}),
\end{align*}
so multiplying everything by $\sqrt{n}$ yields that $T_3 = o_\P(1)$.  The last thing to do is analyze $T_2$. 
We can rewrite $T_2$ as follows:
\[
T_2 = \G_{n,r} m^\ast(Z_{1}, \dots, Z_{r}; g_0) + \underbrace{\left\{\G_{n,r} m^{\beta}(Z_{1}, \dots, Z_{r}; g_{\beta}) - \G_{n,r} m^\ast(Z_{1}, \dots, Z_{r}; g_0)\right\}}_{T_2'}.
\]
First, note $\G_{n, r} m^\ast(Z_1, \dots, Z_r; g_0)$ is simply a U-statistic without nuisance, so Theorems 11.2 and 12.3 from \citet{van2000asymptotic} imply that
\begin{align*}
\G_{n,r} m^\ast(Z_1, \dots, Z_r; g_0) &= \frac{1}{\sqrt{n}}\sum_{i = 1}^n \rho_\theta(Z_i) + o_\P(1),
\end{align*}
where again $\rho_\theta(z) := r\left\{\E\left[m^\ast(Z_1, \dots, Z_r; g_0) \mid Z_1 = z\right]\right\}$. Thus, to complete the proof, it suffices to argue that $T_2' = o_\P(1)$. To do this, we start by bounding the expected square of $T_2$, which will be used down the line in an application of Chebyshev's inequality. We have
\begin{align*}
&\E_{Z_1, \dots, Z_r}\left[\left(\G_{n, r}\left\{ m^\beta(Z_1, \dots, Z_r; g_\beta) - m^\ast(Z_{1}, \dots, Z_{r}; g_0)\right\}\right)^2\right] \\
&\qquad = n\binom{n}{r}^{-2}\E\left[\left(\sum_{i \in \calC_r}\left\{m^\beta(Z_{i_1}, \dots, Z_{i_r}; g_\beta) - \theta_\beta  - m^\ast(Z_{i_1}, \dots, Z_{i_r}; g_0) + \theta_0\right\}\right)^2\right] \\
&\qquad = n\binom{n}{r}^{-2}\E\Bigg[\sum_{i, i' \in \calC_r}\left\{m^\beta(Z_{i_1}, \dots, Z_{i_r}; g_\beta) - \theta_\beta - m^\ast(Z_{i_1}, \dots, Z_{i_r}; g_0) + \theta_0\right\} \\
& \qquad \qquad \times \left\{m^\beta(Z_{i_1'}, \dots, Z_{i_r'}; g_\beta) - \theta_\beta - m^\ast(Z_{i_1'}, \dots, Z_{i_r'}; g_0) + \theta_0\right\}\Bigg] \\
&\qquad = \underbrace{n\binom{n}{r}^{-2}\sum_{c = 0}^r\binom{n}{r}\binom{r}{c}\binom{n - r}{n - c}\xi^\beta_c}_{:=(\star)}
\end{align*}
where the final inequality follows from an analogous argument to the one used in the proof of Theorem 12.3 in \citet{van2000asymptotic}, where, letting $Z_1', \dots, Z_r'$ be a set of $r$ i.i.d.\ samples from $P_Z$ that are independent from $Z_1, \dots, Z_r$, we have defined
\begin{align*}
\xi^\beta_c &:= \Cov\Big(m^\beta(Z_1, \dots, Z_r; g_\beta) -m^\ast(Z_1, \dots, Z_r; g_0), m^\beta(Z_1, \dots, Z_c, Z_{c + 1}', \dots, Z_r'; g_\beta)  \\
&\qquad - m^\ast(Z_1, \dots, Z_c, Z_{c + 1}', \dots, Z_r'; g_0)\Big).
\end{align*}
Note that $\xi^\beta_0 = 0$ (since $Z_1, \dots, Z_r$ and $Z_1', \dots, Z_r'$ are independent) and, for any $c \in \{1, \dots, r\}$, we have the bound
\begin{align*}
\xi^\beta_c &\leq \Var\Big(m^\beta(Z_1, \dots, Z_r; g_\beta) - m^\ast(Z_1, \dots, Z_r; g_0)\Big) &(\text{Cauchy-Schwarz)}\\
&\leq \E\left[\left(m^\beta(Z_1, \dots, Z_r; g_\beta) - m^\ast(Z_1, \dots, Z_r; g_0)\right)^2\right] &(\Var[X] \leq \E X^2)\\
&\leq 2\E\left[m^\beta(Z_1, \dots, Z_r; g_\beta)^2\right] + 2\E\left[m^\ast(Z_1, \dots, Z_r; g_0)^2\right] &(\text{Parallelogram Inequality})\\
&\leq 2\left\|m^\beta(Z_1, \dots, Z_r; g_\beta)\right\|_{L^{2 + \epsilon}(P_Z^{\otimes r})}^{\frac{2}{2 + \epsilon}} + 2\left\|m^\ast(Z_1, \dots, Z_r; g_0)\right\|_{L^{2 + \epsilon}(P_Z^{\otimes r})}^{\frac{2}{2 + \epsilon}} \\
&\leq 2D^{\frac{2}{2 + \epsilon}} &(\text{Condition~\ref{cond:reg}}),
\end{align*}
where the second to last inequality follows from the monotonicty of $L^p$ norms. Thus, the $\xi_c^\beta$ are bounded above by a constant independent of $\beta$. We consequently have
\begin{align*}
(\star) &= n\binom{n}{r}^{-2}\sum_{c = 0}^r\binom{n}{r}\binom{r}{c}\binom{n - r}{n - c}\xi^\beta_c \\
&=n\sum_{c = 1}^r\frac{r!^2}{c!(r - c)!^2}\frac{(n - r)_{r - c}}{(n)_{r}}\xi^\beta_c \\
&\leq n \left(r^2 \frac{(n - r)_{r - 1}}{(n)_r}\xi^\beta_1 + 2\sum_{c = 2}^r\frac{r!^2}{c!(r - c)!^2}\frac{c^c}{c!}D^2\right) &(\text{Lemma~\ref{lem:falling-factorial}, stated below})\\
&\leq C_r\left(\xi_1^\beta + \sum_{c = 2}^r \frac{1}{n^{c - 1}}\right)  &(\text{Since } (n - r)_{r - 1}/(n)_r \leq 1)\\
&= C_r \xi_1^\beta + o(1) \\
&\leq C_r \Var\left(m^\beta(Z_1, \dots, Z_r; g_\beta) - m^\ast(Z_1, \dots, Z_r; g_0)\right) + o(1) &(\text{Again, Cauchy-Schwarz})\\
&\leq C_r\E\left[\left(m^\beta(Z_1, \dots, Z_r; g_\beta) - m^\ast(Z_1, \dots, Z_r; g_0)\right)^2\right] + o(1) &(\text{Again, } \Var[X] \leq \E X^2).
\end{align*}
In the above, $C_r > 0$ is some absolute constant that does not depend on sample size $n \geq r$. 

To make the above bound useful when coupled with Chebyshev's inequality, we need to show now that $\lim_{n \rightarrow \infty}\xi^\beta_1 = 0$. Since we already known $m^\beta(Z_1, \dots, Z_r; g_\beta) - m^\ast(Z_1, \dots, Z_r; g_0) = o_\P(1)$ (Condition~\ref{cond:score}), by Vitali's Theorem (Proposition~\ref{prop:vitali}), it suffices to show that the collection of random variables $(V_\beta)_{n \geq 1}$ are uniformly integrable, where $V_\beta := \left(m^\beta(Z_1, \dots, Z_r; g_\beta) - m^\ast(Z_1, \dots, Z_r; g_0)\right)^2$. To show uniform integrability of this collection, by Proposition~\ref{prop:pous}, it suffices to show that $\sup_{n \geq 1}\E[\Phi(V_\beta)] < \infty$ for $\Phi(t) := t^{1 + \frac{\epsilon}{2}}$. We have
\begin{align*}
\sup_{n \geq 1}\E[\Phi(V_\beta)] &= \sup_{n \geq 1}\E\left[\left|m^\beta(Z_1, \dots, Z_r; g_\beta) - m^\ast(Z_1, \dots, Z_r; g_0)\right|^{2 + \epsilon}\right] \\
&\leq 2^{2 + \epsilon}\left\{\E\left|m^\beta(Z_1, \dots, Z_r; g_\beta)\right|^{2 + \epsilon} + \E\left|m^\ast(Z_1, \dots, Z_r; g_0)\right|^{2 + \epsilon}\right\} \\
&\leq 2^{2 + \epsilon} D.
\end{align*}
Thus, the collection $(V_\beta)_{n \geq r}$ is uniformly integrable, and hence 
\[
\E V_\beta = \E\left[\left(m^\beta(Z_1, \dots, Z_r; g_\beta) - m^\ast(Z_1, \dots, Z_r; g_0)\right)^2\right] = o(1).
\]
We now complete the proof by showing $T_2' = o_\P(1)$ using Chebyshev's inequality. In particular, for any $\delta > 0$, we have that
\begin{align*}
\P\left(T_2' \geq \delta\right) &\leq \frac{1}{\delta^2}\E_{Z_1, \dots, Z_r}\left[\left\{\G_{n,r} m^\beta(Z_{1}, \dots, Z_{r}; g_\beta) - m^\ast(Z_{1}, \dots, Z_{r}; g_0)\right\}^2\right] \\
&\leq \frac{1}{\delta^2}\left\{C_r\E\left[\left(m^\beta(Z_1, \dots, Z_r; g_\beta) - m^\ast(Z_1, \dots, Z_r; g_0)\right)^2\right] + o(1)\right\} \\
&= o(1).
\end{align*}
This thus completes the proof of asymptotic linearity. The claim of asymptotic normality now immediately follows under the assumption that $\Sigma^\ast := \Var[\rho_\theta(Z)] >0$.

\end{proof}

We conclude by stating several lemmas/propositions that were needed in the above proof.

\begin{lemma}
\label{lem:falling-factorial}
For any $n \geq 1$ and $r < n$, define the falling factorial $(n)_r := n(n - 1)\cdots(n - r)$, with the convention that $(n)_0 = 1$. For any $c \in \{1, \dots, r\}$, we have the bound
\[
\frac{(n - r)_{r - c}}{(n)_r} \leq \frac{c^c}{c!}\frac{1}{n^c} \leq \frac{r^r}{r!}\frac{1}{n^c}.
\]
\end{lemma}

\begin{proof}
First, observe that we can write
\[
\frac{(n - r)_{r - c}}{(n)_r} = \frac{1}{(n)_c}\frac{(n - r)_{r - c}}{(n - c)_{r - c}} = \frac{1}{(n)_c}\prod_{j = 0}^{n - c -1}\frac{n - r - j}{n - c - j} \leq \frac{1}{(n)_c}
\]
Next, note that we have
\[
\frac{1}{(n)_c} = \prod_{j = 0}^{c - 1}\frac{1}{n - j} = \frac{1}{n^c}\prod_{j = 0}^{c - 1}\frac{n}{n - j} \leq \frac{1}{n^c}\prod_{j = 0}^{c - 1}\frac{c}{c - j} = \frac{1}{n^c}\frac{c^c}{c!}.
\]
This proves the first claim. The second claim follows from the observation that the map $\frac{k^k}{k!}$ is strictly increasing in $k \in \N$.
\end{proof}

\begin{prop}[Vitali's Theorem, see \citet{bogachev2007measure}]
\label{prop:vitali}
Suppose $(X_n)_{n \geq 1}$ are a sequence of random variables and $X$ is a random variable in a measure space $(\Omega, \calF, P)$ such $X_n \in L^p(P)$ for each $n$ and $X \in L^p(P)$. The following are equivalent.
\begin{enumerate}
    \item $\|X_n - X\|_{L^p(P)} \xrightarrow[n \rightarrow \infty]{} 0$, i.e.\ $(X_n)_{n \geq 0}$ converge in $L^p$ to $X$.
    \item $X_n \xrightarrow[n \rightarrow \infty]{\P} X$ and $(|X_n|^p)_{n \geq 1}$ is uniformly integrable.\footnote{We say a collection of random variables $(X_\alpha)_{\alpha \in A}$ is uniformly integrable if, for any $\delta > 0$, there exists a constant $K_\delta > 0$ such that
    \[
    \sup_{\alpha \in A}\E\left[|X_\alpha|\mathbbm{1}\left\{|X_\alpha| > K_\delta\right\}\right] \leq \delta.
    \]}
\end{enumerate}
\end{prop}

\begin{prop}[\citet{poussin1915integrale,dellacherie2011probabilities}]
\label{prop:pous}
A sequence of integrable random variables $(X_n)_{n \geq 1}$ is uniformly integrable if and only if there is an increasing convex function $\Phi : \R_{\geq 0} \rightarrow \R_{\geq 0}$ such that
\[
\lim_{t \rightarrow \infty}\frac{\Phi(t)}{t} = \infty \quad \text{and} \quad \sup_n \E\left[\Phi(|X_n|)\right] < \infty.
\]

\end{prop}

%% file: appendix/learning.tex
\section{Proofs from Section~\ref{sec:learning}}\label{app:learning}

In this appendix, we prove our results from Section~\ref{sec:learning}. We start by showing the loss $L(\phi; \mu, \zeta)$ is Neyman orthogonal in the sense of Definition~\ref{def:neyman_ortho_loss}. Before presenting our proof, we state a change of measure lemma that will be useful throughout this section.

\begin{lemma}[Reweighting / change-of-measure]\label{lem:reweighting}
Fix $S\subset[d]$ and let $\zeta_0^S$ be as above.
For any integrable functions $f,g:\mathcal X\to\mathbb R$,
\begin{equation}\label{eq:reweighting}
\E \left[f(X)  g(X_S,X'_{-S})\right]
=
\E \left[f(X_S,X'_{-S})  \zeta_0^S(X,X')  g(X)\right].
\end{equation}
\end{lemma}

\begin{proof}
Let $\nu$ be the base measure noted before Assumption~\ref{ass:dgp} in Section~\ref{sec:back}. We can write the left-hand side as an integral (with $x=(x_S,x_{-S})$ and $u$ a value for $X'_{-S}$):
\begin{align*}
&\E \left[f(X)  g(X_S,X'_{-S})\right]
=\int\int f(x)g(x_s, x'_{-S})p(x)p(x')\nu(dx)\nu(dx') \\
&\qquad = \int f(x) \left(\int g(x_s, x_{-s}')p_{-S}
(x'_{-S})p_{S}(x'_S \mid x'_{-S})\nu(dx')\right) p(x)\nu(dx) \\
&\qquad = \int f(x) \left\{\int g(x_S, x'_{-S}) \left(\int p_S(x_S' \mid x'_{-S})\nu_S(dx'_{S})\right)p_{-S}(x'_{-S})\nu_{-S}(dx'_{-S}) \right\}p(x)\nu(dx) \\
&\qquad = \int f(x)\Big(\int g(x_S,x'_{-S})  p_{-S}(x'_{-S})  \nu_{-S}(dx'_{-S})\Big)p(x)  \nu(dx).
\end{align*}
For the right-hand side, expand using independence of $X$ and $X'$: 
\begin{align*}
&\E \left[f(X_S,X'_{-S})  \zeta_0^S(X,X')  g(X)\right] =
\int  \int f(x_S,x'_{-S})  
\frac{p(x_S,x'_{-S})}{p(x_S,x_{-S})}\frac{p_{-S}(x_{-S})}{p_{-S}(x'_{-S})}  
g(x)  p(x)  p_{-S}(x'_{-S})  \nu_{-S}(dx'_{-S})  \nu(dx)\\
&\qquad =
\int  \int f(x_S,x'_{-S})  p(x_S,x'_{-S})  g(x)  p_{-S}(x_{-S})  \nu_{-S}(dx'_{-S})  \nu(dx)\\
&\qquad =
\int \Big(\int f(x_S,x'_{-S})  p(x_S,x'_{-S})  \nu_{-S}(dx'_{-S})\Big)\Big(\int g(x_S,x_{-S})  p_{-S}(x_{-S})  \nu_{-S}(dx_{-S})\Big) \nu_S( dx_S),
\end{align*}
which equals the first expression after rearranging terms (Fubini/Tonelli).
\end{proof}


Now, again before providing our proofs, we quickly recall some notation. Throughout, write $Z:=(X,Y)$ and let $Z'$ be an independent copy.
For any subset $S\subset[d]$, define the sign
\[
\sigma(S) := 2\1\{a\in S\}-1 \in \{-1,+1\},
\]
and recall the SHAP weights (for subsets not containing $a$)
\[
w(S) := \frac{1}{d}\binom{d-1}{|S|}^{-1},\qquad S\subset[d]\setminus\{a\}.
\]
We also recall the density-ratio nuisance
\[
\zeta_0^S(X,X')
:= \frac{p(X_S, X'_{-S})}{p(X_S, X_{-S})}\cdot \frac{p_{-S}(X_{-S})}{p_{-S}(X'_{-S})},\qquad S\subset[d].
\]
Finally, recall the pointwise losses from Proposition~\ref{prop:ortho_loss}: for each $S\subset[d]$,
\[
\ell_S(Z,Z';\phi,\mu,\zeta^S)
:=
\Big(\tfrac{\phi(X)}{2}-\sigma(S)\mu(X_S,X'_{-S})\Big)^2
 - \sigma(S)  \phi(X_S,X'_{-S})  \zeta^S(X,X')\{Y-\mu(X)\},
\]
and the aggregated pointwise loss
\[
\ell(Z,Z';\phi,\mu,\zeta)
:=
\sum_{S\subset[d]\setminus\{a\}} w(S)\Big(\ell_{S\cup a}(Z,Z';\phi,\mu,\zeta^{S\cup a})
+\ell_S(Z,Z';\phi,\mu,\zeta^S)\Big),
\]
with corresponding population loss $L(\phi;\mu,\zeta):=\E[\ell(Z,Z';\phi,\mu,\zeta)]$.

\subsection{Assessing Orthogonality}

\begin{proof}[Proof of Proposition~\ref{prop:ortho_loss}]
First, we show consistency of the loss, i.e. that $\phi_a$ is the unique minimizer of the map $\phi \mapsto L(\phi; \mu_0, \zeta_0)$. Then, we prove we prove that the loss $L(\phi; \mu, \zeta)$ is Neyman orthogonal.

\paragraph{Consistency.}
Let $g_0:=(\mu_0,\zeta_0)$, where $\mu_0(x)=\E[Y\mid X=x]$ and $\zeta_0=(\zeta_0^S)_{S\subset[d]}$.
Fix $S\subset[d]\setminus\{a\}$ and write $T:=S\cup a$.
Under $\mu=\mu_0$, we have $\E[Y-\mu_0(X)\mid X]=0$ and $X'\perp (X,Y)$, hence
\[
\E \left[\phi(X_U,X'_{-U})  \zeta_0^U(X,X')  (Y-\mu_0(X))\right]=0
\qquad\text{for every }U\subset[d].
\]
Therefore, when evaluating the population loss at $g_0$, the correction terms drop, and we have
\begin{align*}
L(\phi;g_0)
&=
\E \left[\sum_{S\subset[d]\setminus\{a\}} w(S)\Big\{
\big(\tfrac{\phi(X)}{2}-\mu_0(X_{S \cup a},X'_{-S \cup a})\big)^2
+
\big(\tfrac{\phi(X)}{2}+\mu_0(X_S,X'_{-S})\big)^2
\Big\}\right]
.
\end{align*}
Let $\mathrm{const}$ collect all terms not depending on $\phi$, namely
\[
\mathrm{const}
:=
\mathbb E \left[\sum_{S\subset[d]\setminus\{a\}} w(S)\Big(
\mu_0(X_{S\cup a},X'_{-S\cup a})^2
+
\mu_0(X_S,X'_{-S})^2
\Big)\right].
\]
Expanding the squares and discarding terms that do not depend on $\phi$ yields
\[
L(\phi;g_0)
=
\E \left[
\frac12\Big(\sum_{S:a\notin S} w(S)\Big)\phi(X)^2
+
\phi(X)\sum_{S:a\notin S} w(S)\Big(\mu_0(X_S,X'_{-S})-\mu_0(X_{S\cup a},X'_{-S\cup a})\Big)
\right]
+\text{const}.
\]
We have 
$
\sum_{S\subset[d]\setminus\{a\}} w(S)
=1,
$
 and the conditional objective (given $X=x$) is a strictly convex quadratic in the scalar $\phi(x)$,
minimized at
\begin{align*}
\phi^\star(x)
&=
\E_{X'} \left[\sum_{S:a\notin S} w(S)\Big(\mu_0(x_{S\cup a},X'_{-S\cup a})-\mu_0(x_S,X'_{-S})\Big)\right]
=
\phi_a(x).
\end{align*}
Thus $\phi_a=\arg\min_\phi L(\phi;g_0)$.

\paragraph{Neyman orthogonality.}
We show $D_{g,\phi}L(\phi_a,g_0)(\Delta_\phi,\Delta_g)=0$ for all directions
$\Delta_\phi = \phi - \phi_a$ and $\Delta_g= g - g_0$, where $\Delta_g = (\Delta_\mu,\Delta_\zeta) = (\mu - \mu_0, \zeta - \zeta_0)$.

It suffices to verify orthogonality \emph{term-by-term} because $L(\phi;g)$ is a finite weighted sum of expectations
of pointwise losses.
Fix $S\subset[d]$ and consider the partial pointwise loss
\[
\ell_S(Z,Z';\phi,\mu,\zeta^S)
:=
\Big(\tfrac{\phi(X)}{2}-\sigma(S)\mu(X_S,X'_{-S})\Big)^2
-\sigma(S)  \phi(X_S,X'_{-S})  \zeta^S(X,X')\{Y-\mu(X)\}.
\]
Let $L_S(\phi;\mu,\zeta^S):=\E[\ell_S(Z,Z';\phi,\mu,\zeta^S)]$.
We will show that for every direction $\Delta_\phi :\mathcal X\to \R$ of the form $\Delta_\phi  = \phi - \phi_a$, we have
\[
D_{\mu, \phi}L_S(\phi;\mu_0,\zeta_0^S)(\Delta_\phi, \Delta_\mu)=0,
\qquad
D_{\zeta^S,\phi}L_S(\phi;\mu_0,\zeta_0^S)(\Delta_\phi, \Delta_{\zeta^S})=0,
\]
which implies the same for the weighted sum defining $L$.

\smallskip
\noindent\emph{Step 1: compute $D_{\phi}L_S$.}
A direct Gateaux differentiation gives
\begin{align}
D_{\phi}L_S(\phi;\mu,\zeta^S)(\Delta_\phi)
&=
\E \left[
\Big(\tfrac{\phi(X)}{2}-\sigma(S)\mu(X_S,X'_{-S})\Big)  \Delta_\phi(X)
\right]
 \\
 &\qquad -\sigma(S)  \E \left[\Delta_\phi(X_S,X'_{-S})  \zeta^S(X,X')  (Y-\mu(X))\right].
\label{eq:Dphi}
\end{align}

\smallskip
\noindent\emph{Step 2: cross derivative w.r.t. $\zeta^S$.}
Differentiating \eqref{eq:Dphi} in direction $\Delta_{\zeta^S}$ yields
\[
D_{\zeta, \phi}L_S(\phi;\mu,\zeta^S)(\Delta_\phi, \Delta_{\zeta^S})
=
-\sigma(S)  \E \left[\Delta_\phi(X_S,X'_{-S})  \Delta_{\zeta^S}(X,X')  (Y-\mu(X))\right].
\]
At $\mu=\mu_0$, $\E[Y-\mu_0(X)\mid X, X'] = \E[Y - \mu_0(X) \mid X] =0$, hence the above expectation is $0$ by iterated expectation,
so $D_{\zeta,\phi}L_S(\phi;\mu_0,\zeta_0^S)(\Delta_\phi, \Delta_{\zeta^S})=0$.

\smallskip
\noindent\emph{Step 3: cross derivative w.r.t. $\mu$.}
Differentiating \eqref{eq:Dphi} in direction $\Delta_\mu$ yields
\begin{align*}
D_{\mu, \phi}L_S(\phi;\mu,\zeta^S)(\Delta_\phi, \Delta_\mu)
&=
-\sigma(S)  \E \left[\Delta_\phi(X)  \Delta_\mu(X_S,X'_{-S})\right]
+\sigma(S)  \E \left[\Delta_\phi(X_S,X'_{-S})  \zeta^S(X,X')  \Delta_\mu(X)\right].
\end{align*}
Setting $\zeta^S=\zeta_0^S$ and applying Lemma~\ref{lem:reweighting} with
$f=\Delta_\phi$ and $g=\Delta_\mu$ gives equality of the two expectations, hence they cancel, proving
$D_{\mu, \phi}L_S(\phi;\mu_0,\zeta_0^S)(\Delta_\phi, \Delta_\mu)=0$.

\smallskip
\noindent Since $L(\phi;\mu,\zeta)$ is a finite weighted sum of the $L_S$'s, the same cancellation implies
$D_{g,\phi}L(\phi_a,g_0)(\Delta_\phi,\Delta_g)=0$, i.e. $L$ is Neyman orthogonal.
This completes the proof.

\end{proof}

\subsection{Learning rates for the SHAP curve}
\label{app:learning_rates}
We work under Assumption~\ref{ass:reg_loss}, using the same absolute constants $c,C>0$ throughout.
Recall the population loss  $L(\phi;g)=\E[\ell(Z,Z';\phi,g)]$ from Proposition~\ref{prop:ortho_loss}, where
$g=(\mu,\zeta)$ and $\zeta=(\zeta^S:S\subset[d])$. Let $\phi_a\in\arg\min_{\phi\in\Phi}L(\phi;g_0)$ with
$g_0=(\mu_0,\zeta_0)$, and let $\wh{g}=(\hat\mu,\hat\zeta)$ be any  estimate of $g_0$ that is
independent of the second-stage sample used to compute the ERM
\[
\wh{\phi}\in\arg\min_{\phi\in\Phi} \frac1n\sum_{i=1}^n \ell(Z_i,Z_i';\phi,\wh{g}).
\]
We use the notation
\[
\Delta_\phi := \phi-\phi_a,\qquad \wh{\Delta}_{ \phi}:=\wh{\phi}-\phi_a,\qquad
\Delta_g := g-g_0,\qquad \wh{\Delta}_{g}:=\wh{g}-g_0,
\]
and abbreviate $\|\cdot\|_2:=\|\cdot\|_{L^2(P_X)}$. Define $\calH := \Phi - \phi_a$ (which we recall is star-shaped by Assumption~\ref{ass:reg_loss}, and further recall that we have defined $\calH(r):=\{\Delta\in\calH:\|\Delta\|_2\le r\}$.

For each $S\subset[d]$, define the mixed covariate $U_S:=(X_S,X'_{-S})$ (so $U_{[d]}=X$ and $U_{\emptyset} = X'$).
We also use the nuisance error norm
\[
\|g\|_{\calG}:=\max_{S\subset[d]}\Big(\E[\mu(X)^2\zeta^S(X,X')^2]\Big)^{1/4},
\qquad
\|\wh{\Delta}_{g}\|_{\calG}:=\|\wh{g}-g_0\|_{\calG}.
\]

Before we prove Lemma~\ref{lem:loss_bound}We first state the intermediate lemmas used in the proof.

\begin{lemma}[Mixed-input $L^2$ comparability]\label{lem:mix-L2}
Assume \textup{(B1)}. Then for any function $h:\calX\to\R$ and any $S\subset[d]$,
\[
\E[h(U_S)^2] \le C\E[h(X)^2] = C\|h\|_2^2,
\]
where $C>0$ is the upper bound constant in \textup{(B1)}.
\end{lemma}

\begin{proof}
Under \textup{(B1)}, $\omega_S(x):=\frac{p_S(x_S)p_{-S}(x_{-S})}{p(x)}$ satisfies $\omega_S(x)\le C$ a.s., hence
$p_S(x_S)p_{-S}(x_{-S})\le Cp(x)$. Therefore
\[
\E[h(U_S)^2]=\E_X[\omega_S(X)h(X)^2]\le C\E_X[h(X)^2]=C\|h\|_2^2.
\]
\end{proof}

\begin{lemma}[First-order optimality]\label{lem:assump2-weighted}
Assume \textup{(B3)}. Then
\[
D_\phi L(\phi_a;g_0)(\phi-\phi_a) \ge 0
\qquad \forall \phi\in\Phi.
\]
\end{lemma}

\begin{proof}
For fixed $g_0$, the map $\phi\mapsto L(\phi;g_0)$ is convex (each term is a square in $\phi(X)$ plus terms linear in
$\phi(U_S)$). Since $\phi_a$ minimizes $L(\cdot;g_0)$ over the star-shaped set $\Phi$, the directional derivative at
$\phi_a$ towards any $\phi\in\Phi$ is nonnegative.
\end{proof}

\begin{lemma}[Curvature in $\phi$]\label{lem:curvature}
For any fixed nuisances $g$ and any direction $\Delta_\phi = \phi -\phi_a$ and $\phi \in \Phi$:, we have
\[
D_\phi^2 L(\wb\phi;g)(\Delta_\phi,\Delta_\phi) = \|\Delta_\phi\|_2^2,
\]
for every convex combination $\wb{\phi} \in [\phi_a, \phi]$.
\end{lemma}

\begin{proof}
Note that we have
\begin{align*}
&D_\phi^2L(\wb{\phi}; g)(\Delta_\phi, \Delta_\phi) \\
&\qquad = D_\phi^2\E \left[
\frac12\Big(\sum_{S:a\notin S} w(S)\Big)\phi(X)^2
+
\phi(X)\sum_{S:a\notin S} w(S)\Big(\mu_0(X_S,X'_{-S})-\mu_0(X_{S\cup a},X'_{-S\cup a})\Big)
\right](\Delta_\phi, \Delta_\phi) \\
&\qquad = D_\phi^2\E\Bigg[\frac12\underbrace{\Big(\sum_{S:a\notin S} w(S)\Big)}_{= 1}\phi(X)^2\Bigg](\Delta_\phi, \Delta_\phi) \\
&\qquad \qquad + D_\phi^2\E\left[\phi(X)\sum_{S:a\notin S} w(S)\Big(\mu_0(X_S,X'_{-S})-\mu_0(X_{S\cup a},X'_{-S\cup a})\Big)
\right](\Delta_\phi, \Delta_\phi)\\
&\qquad = D_\phi^2\E\left[\frac12\phi(X)^2\right](\Delta_\phi, \Delta_\phi)  + 0  =\E[\Delta_\phi(X)^2],
\end{align*}
which proves the result.
\end{proof}

\begin{lemma}[Second-order sensitivity in nuisances]\label{lem:assump3b-weighted}
Let $\Delta_g=(\Delta_\mu,\Delta_\zeta)$ with $\Delta_\mu=\mu-\mu_0$ and $\Delta_{\zeta^S}=\zeta^S-\zeta_0^S$.
Then for all $\Delta_\phi = \phi - \phi_a$ with $\phi \in \Phi$ and all $\wb{g} \in [g_0, g]$, we have
\[
\Big|D_g^2D_\phi L(\phi_a;\wb g)(\Delta_\phi,\Delta_g,\Delta_g)\Big|
 \le C  \|\Delta_\phi\|_2  \|\Delta_g\|_{\calG}^2,
\]
for an absolute constant $C>0$.
\end{lemma}

\begin{proof}
A direct differentiation shows the only nonzero second derivative in $g$ of $D_\phi L$ arises from one derivative in
$\mu$ and one derivative in $\zeta^S$ inside the correction term
$\zeta^S(X,X')\{Y-\mu(X)\}$. Indeed, for each $S\subset[d]$, the pointwise linear correction term is
\[
\ell^{\mathrm{lin}}_S(Z,Z';\phi,g)
:= -\sigma(S)\,\phi(U_S)\,\zeta^S(X,X')\{Y-\mu(X)\},
\qquad U_S:=(X_S,X'_{-S}),
\]
so its $\phi$-directional derivative at $\phi_a$ in direction $\Delta_\phi$ equals
\[
D_\phi \E[\ell^{\mathrm{lin}}_S(Z,Z';\phi_a,g)](\Delta_\phi)
=-\E\big[\sigma(S)\,\Delta_\phi(U_S)\,\zeta^S(X,X')\{Y-\mu(X)\}\big].
\]
Differentiating once in the $\mu$-direction $\Delta_\mu$ gives
\[
D_\mu D_\phi \E[\ell^{\mathrm{lin}}_S (Z,Z';\phi_a,g)](\Delta_\phi, \Delta_\mu)
= \E\big[\sigma(S)\,\Delta_\phi(U_S)\,\zeta^S(X,X')\,\Delta_\mu(X)\big],
\]
and differentiating this in the $\zeta$-direction $\Delta_{\zeta^S}$ yields the mixed second derivative
\[
D_{\zeta}D_{\mu}D_\phi \E[\ell^{\mathrm{lin}}_S(Z,Z';\phi_a,g)](\Delta_\phi, \Delta_{\zeta^S},\Delta_\mu)
= \E\big[\sigma(S)\,\Delta_\phi(U_S)\,\Delta\zeta^S(X,X')\,\Delta_\mu(X)\big].
\]
All other second derivatives in $g$ vanish: the squared part of the loss does not depend on $\zeta$, so any derivative
in $\zeta$ is zero there; and the linear correction is affine in $\mu$ and in $\zeta^S$ separately, so the only
nonzero second derivatives are the two mixed $(\mu,\zeta^S)$ derivatives. Summing over the $S$'s appearing in $L$ gives us
\begin{align*}
D_g^2D_\phi L(\phi_a;\wb g)(\Delta_\phi,\Delta_g,\Delta_g)
=
2\sum_{S:a\notin S} w(S)\Big(
&\E\big[\Delta_\phi(U_{S\cup a})  \Delta\zeta^{S\cup a}(X,X')  \Delta\mu(X)\big]\\
&-\E\big[\Delta_\phi(U_S)  \Delta\zeta^{S}(X,X')  \Delta\mu(X)\big]
\Big).
\end{align*}

We now apply Cauchy--Schwarz to each expectation and then Lemma~\ref{lem:mix-L2} to bound
$\|\Delta_\phi(U_S)\|_{L^2}\le \sqrt{C}\|\Delta_\phi\|_2$. Using $\sum_{S:a\notin S}w(S)=1$ we get
\[
\big|D_g^2D_\phi L(\phi_a;\wb g)(\Delta_\phi,\Delta_g,\Delta_g)\big|
\le
C  \|\Delta_\phi\|_2  
\max_{S\subset[d]}\Big(\E[\Delta\mu(X)^2  \Delta\zeta^S(X,X')^2]\Big)^{1/2}
=
C  \|\Delta_\phi\|_2  \|\Delta_g\|_{\calG}^2,
\]
absorbing numerical factors and the density-ratio bound into the absolute constant $C$. This proves the lemma.
\end{proof}
With the above lemma, we can now prove our generic bound on the $L^2$ norm of any $\wh{\phi} \in \Phi$. 
\newline
\begin{proof}[Proof of Lemma~\ref{lem:loss_bound}]
Let $g$ and $\phi \in \Phi$ be (potentially random) bounded nuisances and estimates of $\phi_a$, respectively. Again, write $\Delta_{\phi}:=\phi-\phi_a$ and $\Delta_{g}:=g-g_0$.

A second-order Taylor expansion of $\phi\mapsto L(\phi;g_0)$ with mean-value theorem remainder yields that, for some
convex combination $\wb{\phi} \in [\phi_a, \phi]$,
\[
L(\phi;g_0)-L(\phi_a;g_0)
=
D_\phi L(\phi_a;g_0)(\Delta_\phi) + \frac12D_\phi^2L(\wb\phi;g_0)(\Delta_{\phi},\Delta_{\phi}).
\]
By Lemma~\ref{lem:assump2-weighted} (first-order optimality), $D_\phi L(\phi_a;g_0)(\Delta_{\phi})\ge 0$, and by Lemma~\ref{lem:curvature}, we have 
$D_\phi^2L(\wb\phi;g_0)(\Delta_{\phi},\Delta_{\phi})=\|\Delta_{\phi}\|_2^2$. Hence
\begin{equation}\label{eq:l2-basic-app}
\frac12\|\Delta_{\phi}\|_2^2 \le L(\phi;g_0)-L(\phi_a;g_0).
\end{equation}

\noindent Next we obtain an upper bound on $L(\phi; g_0) - L(\phi_a; g_0)$ by adding and subtracting the difference $L(\phi; g) - L(\phi_a; g)$, i.e. we have the equality:
\[
L(\phi;g_0)-L(\phi_a;g_0)
=
\underbrace{L(\phi;g)-L(\phi_a;g)}_{=:A}
+
\underbrace{\big(L(\phi;g_0)-L(\phi_a;g_0)\big)-\big(L(\phi;g)-L(\phi_a;g)\big)}_{=:B}.
\]
For any $\phi', g'$, let $\Delta L(\phi',g'):=L(\phi';g')-L(\phi_a;g')$, so $B=\Delta L(\phi, g_0)-\Delta L( \phi, g)$.

Now we upper bound B using two Taylor expansions. For some convex combination  $\wb{\phi}$ of $\phi_a$ and $\phi$, 
\[
\Delta L(\phi,g_0) =  D_\phi L(\phi_a;g_0)(\Delta_{\phi}) +  \frac12 D_\phi^2 L(\wb\phi;g_0)(\Delta_{\phi},\Delta_{\phi})
\]
Similarly for some convex combination $\wb{\phi}'$ of $\phi_a$ and $\phi$, 
\[
\Delta L(\phi,g) =  D_\phi L(\phi_a;g)(\Delta_{\phi}) +  \frac12 D_\phi^2 L(\wb\phi';g)(\Delta_{\phi},\Delta_{\phi})
\]
Subtracting the two equation we get 
\[
\Delta L(\phi,g_0)-\Delta L(\phi,g)
=
\Big(D_\phi L(\phi_a;g_0)-D_\phi L(\phi_a;g)\Big)(\Delta_{\phi})
+\frac12\Big(D_\phi^2 L(\wb\phi;g_0)-D_\phi^2 L(\wb\phi';g)\Big)(\Delta_{\phi},\Delta_{\phi}).
\]
Since the $\phi$-curvature does not depend on $g$ or $\wb{\phi}$.
The second-order term is zero by Lemma~\ref{lem:curvature}, hence
\[
\Delta L(\phi;g_0)-\Delta L(\phi;g)
=
\Big(D_\phi L(\phi_a;g_0)-D_\phi L(\phi_a;g)\Big)(\Delta_{\phi}).
\]
Now expand the Gateaux derivative in $g$ around $g_0$, for some convex combination $\wb{g} \in [g, g_0]$,
\begin{align*}
\Big(D_\phi L(\phi_a;g)-D_\phi L(\phi_a;g_0)\Big)(\Delta_{\phi})
&=
D_{g, \phi}L(\phi_a;g_0)(\Delta_{\phi},\Delta_{g})
+\frac12 D_g^2D_\phi L(\phi_a;\wb g)(\Delta_{\phi},\Delta_{g},\Delta_{g}).
\end{align*}

By Neyman orthogonality from Proposition~\ref{prop:ortho_loss}, $D_gD_\phi L(\phi_a;g_0)(\Delta_\phi, \Delta_g)=0$, hence the first-order term
vanishes. The remaining term satisfies
\[
|B|
=
\frac12\big|D_g^2D_\phi L(\phi_a;\wb g)(\Delta_{\phi},\Delta_{g},\Delta_{g})\big|
 \le \frac{C}{2}  \|\Delta_{\phi}\|_2  \|\Delta_{g}\|_{\calG}^2,
\]
by Lemma~\ref{lem:assump3b-weighted}. Combining \eqref{eq:l2-basic-app} with the decomposition and the bound on $B$:
\[
\frac12\|\Delta_{\phi}\|_2^2
 \le A+\frac{C}{2}\|\Delta_{\phi}\|_2  \|\Delta_{g}\|_{\calG}^2.
\]
Apply Young's inequality:
\[
uv = \left(\tfrac{u}{\sqrt{2}}\cdot \sqrt{2}v\right) \leq  \tfrac{1}{4}(\tfrac{u}{\sqrt{2}} + \sqrt{2}v)^2 \leq \tfrac12\left(\tfrac{u^2}{2}+2v^2\right) \leq \tfrac{u^2}{4} + v^2,
\]
with $u=\|\Delta_{\phi}\|_2$ and
$v=\tfrac{C}{2}\|\Delta_{g}\|_{\calG}^2$ to get
\[
\frac14\|\Delta_{\phi}\|_2^2 \le A+\frac{C^2}{4}\|\Delta_{g}\|_{\calG}^4,
\]
which is Lemma~\ref{lem:loss_bound} (absorbing constants into the absolute constant $D$ in the statement).
\end{proof}

Now, we conclude by proving our main excess risk guarantee, Theorem~\ref{thm:excess_risk}, which instantiates Lemma~\ref{lem:loss_bound} in terms of bounds on the criticial radius of the underlying function class used for ERM. We begin by stating some empirical process lemmas used in the proof and defining some notions of localized complexity.

For $\calH\subset L^\infty(P_X)$ star-shaped, define the \emph{conditional} localized Rademacher complexity
\[
\calR_n(r,\calH)
:= \E_{\varepsilon}\Big[ \max_{S\subset[d]} \sup_{\Delta\in\calH(r)}\frac1n\sum_{i=1}^n\varepsilon_i  \Delta(X_{i,S},X'_{i,-S})
 \Big| (X_i,X_i')_{i=1}^n\Big],
\]
and the \emph{population} localized complexity
\[
\wt{\calR}_n(r,\calH):=\E[\calR_n(r,\calH)].
\]

The next two lemmas are standard tools in localized empirical process theory.
Lemma~\ref{lem:emp-pop-rad} (a specialization of \citet[Lemma~A.4]{bartlett2005local}) relates the
\emph{empirical} localized Rademacher complexity $\calR_n(r,\calH)$ to its \emph{population} counterpart.
Lemma~\ref{lem:rad-scale} is the elementary scaling property enjoyed by star-shaped classes and is the main step that turns the critical-radius condition $\wt{\calR}_n(r_0,\calH)\le r_0^2$ into the self-bounding inequality
$\wt{\calR}_n(r,\calH)\lesssim r r_0+r_0^2$ for all $r\ge 0$.
Together, these lemmas allow us to control the localized deviation of the empirical process by a quadratic function of the radius, which is exactly what is needed to analyze the ERM via the standard localization argument.

\begin{lemma}[Lemma A.4 in \citet{bartlett2005local}]\label{lem:emp-pop-rad}
Assume \textup{(B2)}. Then for every fixed $r\ge 0$ and every $\delta\in(0,1)$, with probability at least $1-\delta$,
\[
\calR_n(r,\calH)
 \lesssim 
\wt{\calR}_n(r,\calH) + r\sqrt{\frac{\log(1/\delta)}{n}}+\frac{\log(1/\delta)}{n}
.
\]
\end{lemma}

\begin{lemma}\label{lem:rad-scale}
If $\calH$ is star-shaped and $r_0>0$, then for all $r\ge 0$,
\[
\wt{\calR}_n(r,\calH) \le \Big(\frac{r}{r_0}+1\Big)\wt{\calR}_n(r_0,\calH).
\]
In particular, if $\wt{\calR}_n(r_0,\calH)\le r_0^2$, then for all $r\ge 0$,
\[
\wt{\calR}_n(r,\calH)\ \le\ r\,r_0+r_0^2.
\]
\end{lemma}

\begin{proof}
We consider two cases, 

\emph{Case 1: $r\ge r_0$.}
Take any $ \Delta_\phi\in\calH(r)$, so $\| \Delta_\phi\|_2\le r$.
Define $ \Delta_\phi':=(r_0/r) \Delta_\phi$.
Since $r_0/r\in[0,1]$ and $\calH$ is star-shaped, we have $ \Delta_\phi'\in\calH$.
Moreover $\| \Delta_\phi'\|_2=(r_0/r)\| \Delta_\phi\|_2\le r_0$, hence $ \Delta_\phi'\in\calH(r_0)$.
Thus $ \Delta_\phi=(r/r_0) \Delta_\phi'$ with $ \Delta_\phi'\in\calH(r_0)$, which implies
\[
\sup_{ \Delta_\phi\in\calH(r)}\frac1n\sum_{i=1}^n \varepsilon_i  \Delta_\phi(U_{i,S})
\ \le\
\frac{r}{r_0}\sup_{ \Delta_\phi'\in\calH(r_0)}\frac1n\sum_{i=1}^n \varepsilon_i  \Delta_\phi'(U_{i,S})
\]
Taking the maximum over $S$ and then expectations yields $\wt{\calR}_n(r,\calH)\le (r/r_0)\wt{\calR}_n(r_0,\calH)$.

\emph{Case 2: $r\le r_0$.}
Then $\calH(r)\subset \calH(r_0)$, so the supremum over $\calH(r)$ is at most the supremum over $\calH(r_0)$, and hence $\wt{\calR}_n(r,\calH)\le \wt{\calR}_n(r_0,\calH)$.

Combining the two cases gives
\[
\wt{\calR}_n(r,\calH)\ \le\ \Big(\frac{r}{r_0}\vee 1\Big)\wt{\calR}_n(r_0,\calH)
\ \le\ \Big(\frac{r}{r_0}+1\Big)\wt{\calR}_n(r_0,\calH).
\]
For the final claim, assume $\wt{\calR}_n(r_0,\calH)\le r_0^2$ and plug this into the previous inequality:
\[
\wt{\calR}_n(r,\calH)\ \le\ \Big(\frac{r}{r_0}+1\Big)r_0^2 \;=\; r\,r_0+r_0^2,
\]
which holds for all $r\ge 0$.
\end{proof}

\begin{proof}[Proof of Theorem~\ref{thm:excess_risk}]
For ease of notation, throughout this proof, let $P$ denote expectation over the pair $(Z,Z')$ and $\P_n$ the empirical average over $(Z_i,Z_i')_{i=1}^n$. Further, let $\wh{g} = (\wh{\mu}, \wh{\zeta})$ and $g_0 = (\mu_0, \zeta_0)$.
Throughout the empirical process argument, we condition on the independent nuisance estimates $\wh{g}$. Unconditional expectations below are over the second-stage sample and, when present, the Rademacher signs.
Define the excess point-wise loss at $\wh{g}$ as:
\[
\Delta\ell(Z,Z';\phi, \wh g)
:=\ell(Z,Z';\phi,\wh{g})-\ell(Z,Z';\phi_a,\wh{g}),
\qquad
A:=P  \Delta\ell(Z,Z';\wh\phi, \wh g)=L(\wh\phi;\wh{g})-L(\phi_a;\wh{g}).
\]
Since $\wh\phi$ is an ERM, $\P_n\Delta\ell(Z,Z';\wh\phi, \wh g)\le 0$, hence
\begin{equation}\label{eq:A-basic}
A \le (P-\P_n)\Delta\ell(Z,Z';\wh\phi, \wh g).
\end{equation}

Recall that the loss has is sum over squared terms and linear correction terms. We will work with each term separately. We define the squared term in the loss by
\[
\small
\ell^{\text{sq}}(Z, Z';\phi, \hat g)=
\sum_{S:a\notin S} w(S)\Big[
\Big(\tfrac{\phi(X)}{2}-\sigma(S\cup a)\hat\mu(X_{S\cup a},X'_{-S\cup a})\Big)^2
+
\Big(\tfrac{\phi(X)}{2}-\sigma(S)\hat\mu(X_{S},X'_{-S})\Big)^2
\Big],
\]

Next we define the linear correction for each $S\subset[d]$
\begin{align*}
    \ell^{\text{lin}}_S(Z,Z';\phi, \hat g) = -\sigma(S) \phi(X_{S},X'_{-S}) \hat\zeta^{S}(X,X')\{Y-\hat\mu(X)\}
\end{align*}

Hence we can write 
\[
\ell =  \ell^{\text{sq}} + \sum_{S:a\notin S} w(S) \left(\ell^{\text{lin}}_S + \ell^{\text{lin}}_{S\cup a}\right)
\]

so that $\ell= \ell^{\text{sq}}+\ell^{\text{lin}}_{\hat g}$ and hence
\[
\Delta\ell=\Delta \ell^{\text{sq}}+\sum_{S:a\notin S} w(S) \left(\Delta \ell^{\text{lin}}_S + \Delta \ell^{\text{lin}}_{S\cup a}\right),
\]

Recall for some $r\ge 0$ we define $\calH(r):=\{\Delta_\phi = \phi - \phi_a \in\calH:\|\Delta\|_2\le r\}$.
Since $\Delta\ell=\Delta \ell^{\text{sq}}+\sum_{S:a\notin S}w(S)\big(\Delta\ell^{\mathrm{lin}}_{S,\hat g}
+\Delta\ell^{\mathrm{lin}}_{S\cup a}\big)$ and $(P-\P_n)$ is linear, we have, via triangle inequality, for every $\phi$ such that $\Delta_\phi \in\calH(r)$
\begin{align*}
\Big|(P-\P_n)&\Delta\ell(Z,Z';\phi,\hat g)\Big|
\le
\Big|(P-\P_n)\Delta \ell^{\text{sq}}(Z,Z';\phi)\Big|
\\
&+\sum_{S:a\notin S}w(S)\Big(
\Big|(P-\P_n)\Delta\ell^{\mathrm{lin}}_{S,\hat g}(Z,Z';\phi)\Big|
+\Big|(P-\P_n)\Delta\ell^{\mathrm{lin}}_{S\cup a}(Z,Z';\phi)\Big|
\Big).
\end{align*}
Taking $\sup_{\Delta_\phi \in\calH(r)}$ and then expectation gives
\begin{align}
\label{eq:split-emp-proc}
\E&\Big[\sup_{\Delta_\phi \in\calH(r)}\Big|(P-\P_n)\Delta\ell(Z,Z';\phi,\hat g)\Big|
 \Big] \nonumber\\
&\le
\E\Big[\sup_{\Delta_\phi \in\calH(r)}\Big|(P-\P_n)\Delta \ell^{\text{sq}}(Z,Z';\phi)\Big|
 \Big] \nonumber\\
&\quad+
\sum_{S:a\notin S}w(S) 
\E\Big[\sup_{\Delta_\phi \in\calH(r)}\Big|(P-\P_n)\Delta\ell^{\mathrm{lin}}_{S,\hat g}(Z,Z';\phi)\Big|
 \Big] \nonumber\\
&\quad+
\sum_{S:a\notin S}w(S) 
\E\Big[\sup_{\Delta_\phi \in\calH(r)}\Big|(P-\P_n)\Delta\ell^{\mathrm{lin}}_{S\cup a}(Z,Z';\phi)\Big|
 \Big].
\end{align}

For any function class $\mathcal F$ of measurable functions of $(Z,Z')$, standard symmetrization argument \citep{wainwright2019high} gives us
\[
\E\Big[\sup_{f\in\mathcal F}\big|(P-\P_n)f\big|\Big]
\le
2 \E\Big[\sup_{f\in\mathcal F}\Big|\frac1n\sum_{i=1}^n \varepsilon_i f(Z_i,Z_i')\Big|\Big].
\]
Apply this with function classes classes indexed by $\Delta_\phi \in\calH(r)$:
\begin{align*}
    \mathcal F_{\mathrm{sq}}(r):=\{\Delta \ell^{\text{sq}}(\cdot;\phi):\Delta_\phi \in\calH(r)\},
&\qquad
\mathcal F^{\text{lin}}_{S}(r):=\{\Delta\ell^{\mathrm{lin}}_{S,\hat g}(\cdot;\phi):\Delta_\phi \in\calH(r)\},
\\
\mathcal F^{\text{lin}}_{S\cup a}(r)=&\{\Delta\ell^{\mathrm{lin}}_{S\cup a}(\cdot;\phi):\Delta_\phi \in\calH(r)\}.
\end{align*}

Applying the Symmetrization argument bound on each term in \eqref{eq:split-emp-proc} we get
\begin{align}
\label{eq:symm-split}
&\E\Big[\sup_{\Delta_\phi \in\calH(r)}\Big|(P-\P_n)\Delta\ell(Z,Z';\phi,\hat g)\Big|
 \Big] \nonumber\\
&\qquad\le
2 \E\Bigg[\sup_{\Delta_\phi \in\calH(r)}\Big|\frac1n\sum_{i=1}^n \varepsilon_i 
\Delta \ell^{\text{sq}}(Z_i,Z_i';\phi)\Big|
 \Bigg] \nonumber\\
&\qquad\quad+
2\sum_{S:a\notin S}w(S) 
\E\Bigg[\sup_{\Delta_\phi \in\calH(r)}\Big|\frac1n\sum_{i=1}^n \varepsilon_i 
\Delta\ell^{\mathrm{lin}}_{S}(Z_i,Z_i';\phi)\Big|
 \Bigg] \nonumber\\
&\qquad\quad+
2\sum_{S:a\notin S}w(S) 
\E\Bigg[\sup_{\Delta_\phi \in\calH(r)}\Big|\frac1n\sum_{i=1}^n \varepsilon_i 
\Delta\ell^{\mathrm{lin}}_{S\cup a}(Z_i,Z_i';\phi)\Big|
 \Bigg].
\end{align}


Under (B2), (B4), (B5) in Assumption \ref{ass:reg_loss}, all quantities $Y$, $\hat\mu$, $\hat\zeta^S$, $\phi_a$, and $\phi\in\Phi$ are uniformly
bounded by $C$, hence the following Lipschitz bounds hold :

\smallskip
\noindent\emph{(i) Squared part.}
For each $i$, $\Delta \ell^{\text{sq}}(Z_i,Z_i';\phi)$ depends on $\Delta$ only through $\Delta(X_i)$.
Moreover, the map $t\mapsto \Delta \ell^{\text{sq}}(Z_i,Z_i';\phi_a+t)$ is $C$-Lipschitz on the range
$|t|\le 2C$. Therefore, by the contraction inequality,
\begin{align}
\label{eq:contr-sq}
\E_{\varepsilon}\Big[\sup_{\Delta_\phi \in\calH(r)}\Big|\frac1n\sum_{i=1}^n \varepsilon_i 
\Delta \ell^{\text{sq}}(Z_i,Z_i';\phi)\Big|
 \Big| (Z_i,Z_i')_{i=1}^n\Big]
 \lesssim \
\E_{\varepsilon}\Big[\sup_{\Delta_\phi \in\calH(r)}\Big|\frac1n\sum_{i=1}^n \varepsilon_i \Delta(X_i)\Big|
 \Big| (Z_i,Z_i')_{i=1}^n\Big].
\end{align}

\smallskip
\noindent\emph{(ii) Linear part.}
For any fixed $S\subset[d]$,
\[
\Delta\ell^{\text{lin}}_{S}(Z_i,Z_i';\phi)
=
-\sigma(S) \Delta(U_{i,S})\cdot \hat\zeta^{S}(X_i,X_i')\{Y_i-\hat\mu(X_i)\},
\qquad U_{i,S}:=(X_{i,S},X'_{i,-S}).
\]
Conditional on the data, the coefficient $\hat\zeta^{S}(X_i,X_i')\{Y_i-\hat\mu(X_i)\}$ is bounded by $C$,
hence
\begin{align}
\label{eq:contr-lin}
\E_{\varepsilon}\Big[\sup_{\Delta_\phi \in\calH(r)}\Big|\frac1n\sum_{i=1}^n \varepsilon_i 
\Delta\ell^{\text{lin}}_{S}(Z_i,Z_i';\phi)\Big|
 \Big| (Z_i,Z_i')_{i=1}^n\Big]
 \lesssim \
\E_{\varepsilon}\Big[\sup_{\Delta_\phi \in\calH(r)}\Big|\frac1n\sum_{i=1}^n \varepsilon_i \Delta(U_{i,S})\Big|
 \Big| (Z_i,Z_i')_{i=1}^n\Big].
\end{align}

Combining \eqref{eq:symm-split} with \eqref{eq:contr-sq}--\eqref{eq:contr-lin} and using
$\sum_{S:a\notin S}w(S)=1$, then averaging over the data and hiding absolute constants, we get the following bound
\begin{align*}
\E\Big[\sup_{\Delta_\phi \in\calH(r)}\Big|(P-\P_n)\Delta\ell(Z,Z';\phi,\hat g)\Big|
 \Big]
 &\lesssim  \max_{S\subset[d]}
\E\Big[\sup_{\Delta_\phi \in\calH(r)}\Big|\frac1n\sum_{i=1}^n \varepsilon_i \Delta(U_{i,S})\Big|\Big]\\
 &\leq 
\E\Big[ \max_{S\subset[d]} \sup_{\Delta_\phi \in\calH(r)}\Big|\frac1n\sum_{i=1}^n \varepsilon_i \Delta(U_{i,S})\Big|\Big]\\
 &\le 2\wt{\calR}_n(r,\calH).
\end{align*}
The last inequality uses $0\in\calH(r)$ and symmetry of the Rademacher signs. The same contraction bounds, conditional on the data, bound the empirical Rademacher complexity of the excess-loss class by a constant times $\calR_n(r,\calH)$.

By the same boundedness assumptions (B2), (B4), (B5) used above, the class
$\{\Delta\ell(\cdot;\phi,\hat g):\Delta_\phi \in\calH(r)\}$ is uniformly bounded by an absolute constant
and has conditional variance $\lesssim r^2$ (using also Lemma~\ref{lem:mix-L2}).
Hence a standard concentration inequality for bounded empirical processes (\citet[Theorem~2.1]{bartlett2005local}) gives, for every fixed $r\ge 0$ and every $\delta\in(0,1)$,
with probability at least $1-\delta$,
\begin{equation}
\label{eq:local-dev-again}
\sup_{\Delta_\phi \in\calH(r)}\Big|(P-\P_n)\Delta\ell(Z,Z';\phi,\hat g)\Big|
 \lesssim\
\calR_n(r,\calH) + r\sqrt{\frac{\log(2/\delta)}{n}} + \frac{\log(2/\delta)}{n}.
\end{equation}

Applying Lemma~\ref{lem:emp-pop-rad}, for every fixed $r\ge 0$ we get, with probability at least $1-\delta$,
\begin{equation}
\label{eq:R-to-Rtilde}
\calR_n(r,\calH)
 \lesssim\
\wt{\calR}_n(r,\calH) + r\sqrt{\frac{\log(2/\delta)}{n}} + \frac{\log(2/\delta)}{n}.
\end{equation}
Applying \eqref{eq:R-to-Rtilde} and \eqref{eq:local-dev-again} with failure probability $\delta/2$ each and combining them yields, for every fixed $r\ge 0$, with probability at least $1-\delta$,
\begin{equation}
\label{eq:local-dev-Rtilde}
\sup_{\Delta_\phi \in\calH(r)}\Big|(P-\P_n)\Delta\ell(Z,Z';\phi,\hat g)\Big|
 \lesssim\
\wt{\calR}_n(r,\calH) + r\sqrt{\frac{\log(2/\delta)}{n}} + \frac{\log(2/\delta)}{n}.
\end{equation}

The bound \eqref{eq:local-dev-Rtilde} is proved for any \emph{fixed} radius $r$.
To apply it at the random radius $\|\hat\Delta_\phi\|_2$, we first make the bound uniform over $r\ge r_n$ via a standard peeling argument.

By (B2), for any $\Delta\in\calH$ we have $\|\Delta\|_\infty \le 2C$, hence also $\|\Delta\|_2\le 2C$.
We define dyadic radii $r_k:=2^k r_n$ and let
\[
K  :=  \min\big\{k\ge 0:\ r_k\ge 2C\big\}
 =  \Big\lceil \log_2\!\Big(\frac{2C}{r_n}\Big)\Big\rceil_+,
\]
so that $\calH(r_k)=\calH$ for all $k\ge K$ and only the finitely many shells $k=0,1,\dots,K$ matter.

Set $\delta=n^{-2}$ and $\delta_k:=6\delta/(\pi^2(k+1)^2)$, so $\sum_{k=0}^K\delta_k\le\delta$.
Applying \eqref{eq:local-dev-Rtilde} at radius $r_k$ with confidence $\delta_k$ and taking a union bound yields an event
of probability at least $1-\delta=1-n^{-2}$ on which, \emph{simultaneously for all} $k\in\{0,\dots,K\}$,
\begin{equation}
\label{eq:local-dev-shellk}
\sup_{\Delta\in\calH(r_k)}
\Big|(P-\P_n)\Delta\ell(Z,Z';\phi_a+\Delta,\hat g)\Big|
\;\lesssim\;
\wt{\calR}_n(r_k,\calH)
\;+\;
r_k\sqrt{\frac{\log(2/\delta_k)}{n}}
\;+\;
\frac{\log(2/\delta_k)}{n}.
\end{equation}
Note that this choice of $\delta_k$ only changes the logarithmic factors mildly:
\[
\log\Big(\frac{2}{\delta_k}\Big)
=
\log\Big(\frac{1}{\delta}\Big)
+\log\Big(\frac{\pi^2}{3}\Big)
+2\log(k+1)
\;\le\;
2\log n + O\big(1+\log(K+1)\big).
\]
Since $r_n\gtrsim\sqrt{\log n/n}$, we have $K=O(\log n)$, and thus
$\sqrt{\log(2/\delta_k)/n}\lesssim \sqrt{\log n/n}$ and $\log(2/\delta_k)/n\lesssim \log n/n$.

Finally, for an arbitrary $r\in[r_n,2C]$, we pick $k$ such that $r\le r_k\le 2r$; then $\calH(r)\subseteq\calH(r_k)$ and Lemma~\ref{lem:rad-scale} gives $\wt{\calR}_n(r_k,\calH)\le 3\wt{\calR}_n(r,\calH)$.

In case $\|\hat\Delta_\phi\|_2 \le r_n$, the theorem statement holds trivially ($\|\hat\Delta_\phi\|_2^2 \leq r_n^2$). Now suppose  $\|\hat\Delta_\phi\|_2 \geq r_n$. Applying the uniform bound obtained from \eqref{eq:local-dev-shellk} at $r=\|\hat\Delta_\phi\|_2$, with probability at least $1-n^{-2}$, we get
\begin{align}\label{eq:apply-random-radius}
\Big|(P-\P_n)\Delta\ell(Z,Z';\hat\phi,\hat g)\Big|
&\le
\sup_{\Delta\in\calH(\|\hat\Delta_\phi\|_2)}
\Big|(P-\P_n)\Delta\ell(Z,Z';\phi_a+\Delta,\hat g)\Big| \nonumber \\
&\lesssim 
\wt{\calR}_n(\|\hat\Delta_\phi\|_2,\calH)
+\|\hat\Delta_\phi\|_2\sqrt{\frac{\log(n)}{n}}
+\frac{\log(n)}{n} \nonumber\\
&\lesssim
\|\hat\Delta_\phi\|_2\,r_n^\ast + (r_n^\ast)^2
+\|\hat\Delta_\phi\|_2\sqrt{\frac{\log(n)}{n}}
+\frac{\log(n)}{n} \nonumber\\
&\lesssim
\|\hat\Delta_\phi\|_2\,r_n + r_n^2.
\end{align}

Where the second last inequality follows form Lemma \ref{lem:rad-scale} and the last inequality follows from the assumed bound $r_n\gtrsim\sqrt{\log n/n}\lor r_n^\ast$.

Substituting  the bound \eqref{eq:apply-random-radius} in \eqref{eq:A-basic} we get, 
\begin{equation}
\label{eq:A-bound-explicit}
A
 = P\,\Delta\ell(Z,Z';\hat\phi,\hat g)
 \le 
(P-\P_n)\Delta\ell(Z,Z';\hat\phi,\hat g)
 \lesssim\
r_n\|\hat\Delta_\phi\|_2+r_n^2 .
\end{equation}

We will now combine the above bound with Lemma~\ref{lem:loss_bound} and absorb the linear term,
\begin{equation}
\label{eq:lossbound-plug}
\|\hat\Delta_\phi\|_2^2
 \le
4A + D\|\hat g-g_0\|_{\calG}^4.
\end{equation}
Plugging \eqref{eq:A-bound-explicit} into \eqref{eq:lossbound-plug} yields, with probability at least $1-O(n^{-2})$,
\begin{equation}
\label{eq:pre-young}
\|\hat\Delta_\phi\|_2^2
 \lesssim\
r_n\|\hat\Delta_\phi\|_2+r_n^2+\|\hat g-g_0\|_{\calG}^4.
\end{equation}
Finally, we use $r_n\|\hat\Delta_\phi\|_2 \le \tfrac{1}{2c}\|\hat\Delta_\phi\|_2^2 + \tfrac{c}{2}r_n^2$ (for an appropriate constant $c$) to get
\[
\|\hat\Delta_\phi\|_2^2
 \lesssim\
r_n^2+\|\hat g-g_0\|_{\calG}^4,
\]
i.e.
\[
\|\hat\phi-\phi_a\|_2^2
 \lesssim\
r_n^2+\|\hat g-g_0\|_{\calG}^4,
\]
which is the claimed $O_\P(\cdot)$ rate. In particular, $\|\hat g-g_0\|_{\calG}=o_\P(r_n^{1/2})$ implies $\|\hat g-g_0\|_{\calG}^4=o_\P(r_n^2)$, giving $\|\hat\phi-\phi_a\|_2^2=O_\P(r_n^2)$.

\end{proof}

%% file: appendix/experiments.tex
\section{Experiments}
\label{app:experiments}

\begin{figure}[t]
  \centering
  \includegraphics[width=0.9\linewidth]{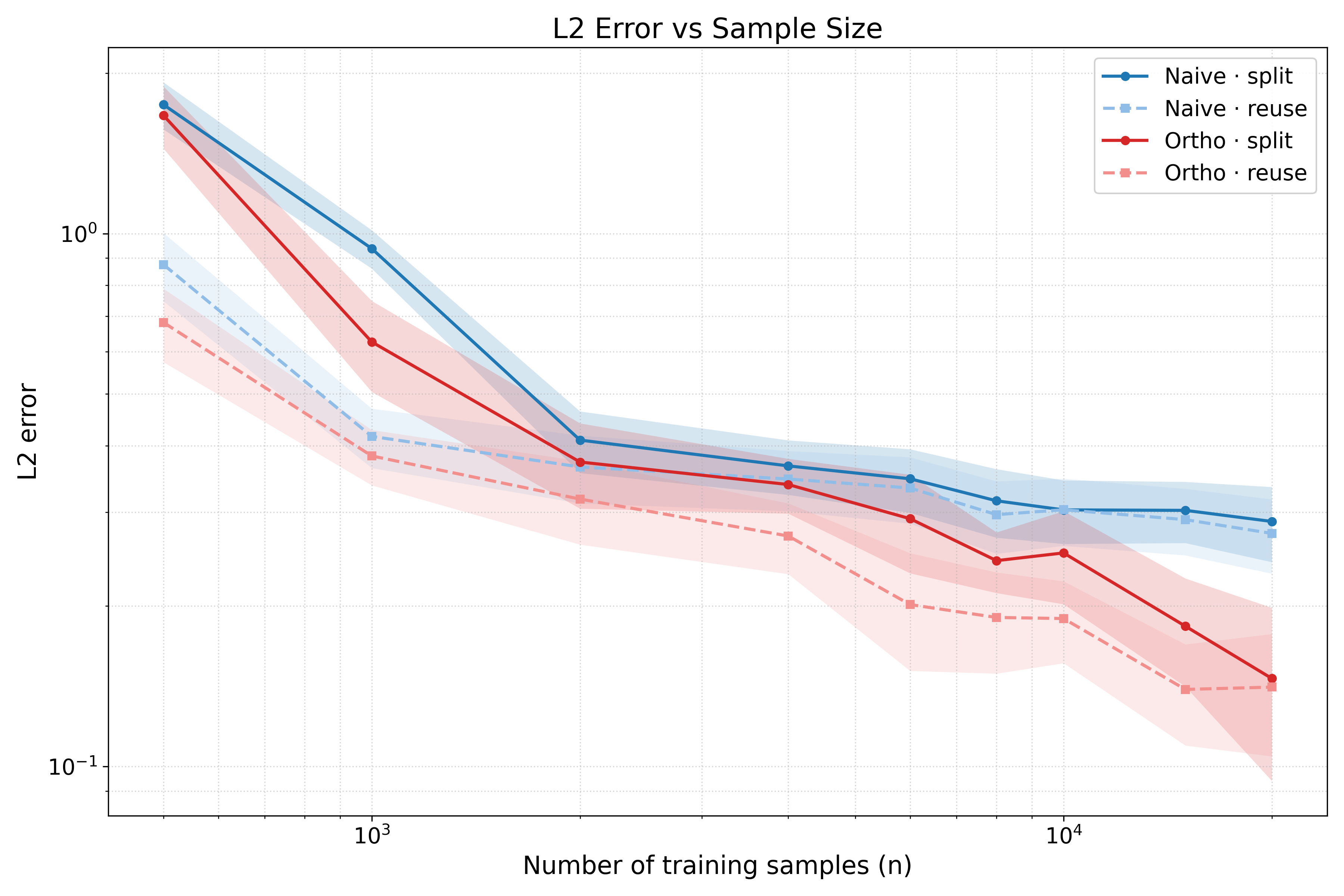}
  \caption{L2 risk versus training sample size $n$ (log--log scale). Curves show the mean across seeds for naive and orthogonal estimators under split vs reuse. Shaded bands indicate $\pm 1$ stderr around the mean.}
  \label{fig:l2-summary}
\end{figure}

\subsection{Data generating process}
We study a synthetic regression problem with $d=5$ covariates. For each run (for a fixed seed), we first generate a seed-specific covariance matrix $\Sigma\in\mathbb{R}^{d\times d}$ by drawing a random orthogonal matrix $Q$ (via QR decomposition of a standard Gaussian matrix) and independent eigenvalues
$\lambda_1,\dots,\lambda_d\sim \mathrm{Unif}(1-\text{spread},\,1+\text{spread})$ with $\text{spread}=0.6$, and setting
\[
\Sigma \;=\; Q\,\mathrm{diag}(\lambda_1,\dots,\lambda_d)\,Q^\top .
\]
This produces a positive-definite covariance with eigenvalues in $[0.4,1.6]$.

Given $\Sigma$, we draw training covariates i.i.d.\ as
\[
X \sim \mathcal{N}(0,\Sigma),
\]
and responses from an additive-noise model
\[
Y \;=\; f^\star(X) + \varepsilon,\qquad \varepsilon\sim\mathcal{N}(0,\sigma_\varepsilon^2),\ \ \sigma_\varepsilon=0.5,
\]
where the ground-truth regression function is the fixed nonlinear map
\[
f^\star(x) =
4.5\sin(1.2x_1) + 3.7\cos(0.8x_2x_3) + 2.0x_1x_2 - 0.9\tanh(x_3)
+ 0.6\exp(-0.5x_2^2) + \sigma\!\bigl(1.5(x_1-x_3)\bigr),
\]
with $\sigma(\cdot)$ the logistic sigmoid.

\subsection{Target: SHAP value of a single feature}
For each seed, we fix an evaluation set of $200$ covariate points drawn from the same distribution $X\sim\mathcal{N}(0,\Sigma)$ and define the target as the SHAP value curve for a single feature (the second coordinate).
Since the exact SHAP values are intractable here, we use a high-accuracy Monte-Carlo approximation as an oracle target: for each evaluation point $x$, we approximate the standard subset-form SHAP definition by averaging over $M=12{,}000$ background samples $X'\sim\mathcal{N}(0,\Sigma)$ and summing over all subsets with Shapley weights.

\subsection{Algorithms evaluated}
We compare two learners for the SHAP curve $\phi(x)$, both of which rely on an estimated nuisance regression function $\mu(x)\approx \mathbb{E}[Y\mid X=x]$.

\paragraph{Naive SHAP-curve learner.}
The naive approach learns $\phi(\cdot)$ by minimizing a squared-error objective derived from the SHAP characterization, but \emph{without} any debiasing/orthogonalization corrections. Conceptually, it plugs in the fitted nuisance $\widehat\mu$ into the SHAP-based loss and trains $\widehat\phi$ directly.

\paragraph{Orthogonal (debiased) SHAP-curve learner.}
The orthogonal approach augments the same core objective with Neyman-orthogonal correction terms that reduce sensitivity to nuisance estimation error. These corrections involve density-ratio factors computed under a Gaussian model for $X$, using an estimated covariance $\widehat\Sigma$ (the sample covariance of the available training covariates, with a small diagonal jitter for numerical stability).

\paragraph{Split vs.\ reuse training regimes.}
For each sample size $n$, we evaluate both algorithms under two data-use protocols:
(i) \textbf{split}, where half of the labeled data are used to fit $\widehat\mu$ and the other half to fit $\widehat\phi$; and
(ii) \textbf{reuse}, where the full labeled dataset is used for both $\widehat\mu$ and $\widehat\phi$ (with independent randomization for optimization).

This yields four methods: naive/split, naive/reuse, orthogonal/split, orthogonal/reuse.

\subsection{Neural network configurations}
Both the nuisance model and SHAP-curve model are multilayer perceptrons (MLPs) with ReLU activations, trained with Adam (learning rate $10^{-3}$) and weight decay $10^{-4}$.

\paragraph{Neural networks for $\widehat\mu$ and $\phi$}
We use a 3-hidden-layer MLP with hidden widths $(128,128,128)$.

When forming mixed inputs that require draws of $X'$, we use bootstrap resampling from the observed training covariates.

\subsection{Experimental grid, seeds, and metric}
We use 6 independent random seeds and the following training sample sizes:
\[
n \in \{500,1000,2000,4000,6000,8000,10000,15000,20000\}.
\]
Performance is measured by mean squared error between the learned SHAP curve and the oracle SHAP curve on the fixed evaluation set of 200 points, averaged over seeds; we also report standard errors across seeds.

\subsection{Results (from the combined run)}
The combined results show that orthogonalization consistently improves accuracy, and reuse improves performance especially at smaller sample sizes. For instance, the mean MSE for the orthogonal+reuse method decreases from $0.681$ (s.e.\ $0.107$) at $n=500$ to $0.141$ (s.e.\ $0.036$) at $n=20000$, whereas the naive+split baseline decreases from $1.745$ (s.e.\ $0.177$) to $0.288$ (s.e.\ $0.046$) over the same range.

%% file: bib.bib
@article{lundberg2017unified,
  title={A unified approach to interpreting model predictions},
  author={Lundberg, Scott M and Lee, Su-In},
  journal={Advances in neural information processing systems},
  volume={30},
  year={2017}
}

@article{lundberg2018consistent,
  title={Consistent individualized feature attribution for tree ensembles},
  author={Lundberg, Scott M and Erion, Gabriel G and Lee, Su-In},
  journal={arXiv preprint arXiv:1802.03888},
  year={2018}
}

@inproceedings{sundararajan2020many,
  title={The many Shapley values for model explanation},
  author={Sundararajan, Mukund and Najmi, Amir},
  booktitle={International conference on machine learning},
  pages={9269--9278},
  year={2020},
  organization={PMLR}
}

@inproceedings{janzing2020feature,
  title={Feature relevance quantification in explainable AI: A causal problem},
  author={Janzing, Dominik and Minorics, Lenon and Bl{\"o}baum, Patrick},
  booktitle={International Conference on artificial intelligence and statistics},
  pages={2907--2916},
  year={2020},
  organization={PMLR}
}

@article{heskes2020causal,
  title={Causal shapley values: Exploiting causal knowledge to explain individual predictions of complex models},
  author={Heskes, Tom and Sijben, Evi and Bucur, Ioan Gabriel and Claassen, Tom},
  journal={Advances in neural information processing systems},
  volume={33},
  pages={4778--4789},
  year={2020}
}

@article{miftachov2024shapley,
  title={Shapley Curves: A Smoothing Perspective},
  author={Miftachov, Ratmir and Keyilbar, Georg and H{\"a}rdle, Wolfgang Karl},
  journal={Journal of Business \& Economic Statistics},
  pages={1--12},
  year={2024},
  publisher={Taylor \& Francis}
}

@article{williamson2021nonparametric,
  title={Nonparametric variable importance assessment using machine learning techniques},
  author={Williamson, Brian D and Gilbert, Peter B and Carone, Marco and Simon, Noah},
  journal={Biometrics},
  volume={77},
  number={1},
  pages={9--22},
  year={2021},
  publisher={Oxford University Press}
}

@article{williamson2023general,
  title={A general framework for inference on algorithm-agnostic variable importance},
  author={Williamson, Brian D and Gilbert, Peter B and Simon, Noah R and Carone, Marco},
  journal={Journal of the American Statistical Association},
  volume={118},
  number={543},
  pages={1645--1658},
  year={2023},
  publisher={Taylor \& Francis}
}

@inproceedings{williamson2020efficient,
  title={Efficient nonparametric statistical inference on population feature importance using Shapley values},
  author={Williamson, Brian and Feng, Jean},
  booktitle={International conference on machine learning},
  pages={10282--10291},
  year={2020},
  organization={PMLR}
}

@article{chernozhukov2022automatic,
  title={Automatic debiased machine learning of causal and structural effects},
  author={Chernozhukov, Victor and Newey, Whitney K and Singh, Rahul},
  journal={Econometrica},
  volume={90},
  number={3},
  pages={967--1027},
  year={2022},
  publisher={Wiley Online Library}
}

@article{chernozhukov2022nested,
  title={Automatic debiased machine learning for dynamic treatment effects and general nested functionals},
  author={Chernozhukov, Victor and Newey, Whitney and Singh, Rahul and Syrgkanis, Vasilis},
  journal={arXiv preprint arXiv:2203.13887},
  year={2022}
}

@misc{chernozhukov2018double,
  title={Double/debiased machine learning for treatment and structural parameters},
  author={Chernozhukov, Victor and Chetverikov, Denis and Demirer, Mert and Duflo, Esther and Hansen, Christian and Newey, Whitney and Robins, James},
  year={2018},
  publisher={Oxford University Press Oxford, UK}
}

@software{lundberg2025shap,
  author       = {Scott M. Lundberg and the SHAP Developers},
  title        = {{SHAP}: {SH}apley {A}dditive {E}xplanations},
  year         = {2025},
  version      = {0.47.2},
  date         = {2025-04-17},
  url          = {https://github.com/shap/shap},
  note         = {Python package},
}

@incollection{shapley1953value,
  author       = {Shapley, Lloyd S.},
  title        = {A Value for n-Person Games},
  booktitle    = {Contributions to the Theory of Games, {V}olume {II}},
  editor       = {Kuhn, Harold W. and Tucker, Albert W.},
  series       = {Annals of Mathematics Studies},
  volume       = {28},
  pages        = {307--317},
  publisher    = {Princeton University Press},
  address      = {Princeton, NJ},
  year         = {1953}
}

@article{lundberg2020local,
  title={From local explanations to global understanding with explainable AI for trees},
  author={Lundberg, Scott M and Erion, Gabriel and Chen, Hugh and DeGrave, Alex and Prutkin, Jordan M and Nair, Bala and Katz, Ronit and Himmelfarb, Jonathan and Bansal, Nisha and Lee, Su-In},
  journal={Nature machine intelligence},
  volume={2},
  number={1},
  pages={56--67},
  year={2020},
  publisher={Nature Publishing Group}
}

@article{chen2023algorithms,
  title={Algorithms to estimate Shapley value feature attributions},
  author={Chen, Hugh and Covert, Ian C and Lundberg, Scott M and Lee, Su-In},
  journal={Nature Machine Intelligence},
  volume={5},
  number={6},
  pages={590--601},
  year={2023},
  publisher={Nature Publishing Group UK London}
}

@article{frye2020asymmetric,
  title={Asymmetric shapley values: incorporating causal knowledge into model-agnostic explainability},
  author={Frye, Christopher and Rowat, Colin and Feige, Ilya},
  journal={Advances in neural information processing systems},
  volume={33},
  pages={1229--1239},
  year={2020}
}

@article{wang2023total,
  title={Total variation floodgate for variable importance inference in classification},
  author={Wang, Wenshuo and Janson, Lucas and Lei, Lihua and Ramdas, Aaditya},
  journal={arXiv preprint arXiv:2309.04002},
  year={2023}
}

@software{R-shapley,
  author       = {Efthimios F. Haghish},
  title        = {{shapley}: Weighted Mean SHAP and Confidence Intervals for Robust Feature Assessment in Machine Learning},
  version      = {0.5},
  date         = {2025},
  year = {2025},
  url          = {https://CRAN.R-project.org/package=shapley},
  organization = {Comprehensive R Archive Network (CRAN)},
  type         = {R package}
}

@article{breiman2001random,
  title={Random Forests},
  author={Breiman, L},
  journal={Machine learning},
  volume={45},
  pages={5--32},
  year={2001}
}

@article{fisher2019all,
  title={All models are wrong, but many are useful: Learning a variable's importance by studying an entire class of prediction models simultaneously},
  author={Fisher, Aaron and Rudin, Cynthia and Dominici, Francesca},
  journal={Journal of Machine Learning Research},
  volume={20},
  number={177},
  pages={1--81},
  year={2019}
}

@article{van2006statistical,
  title={Statistical inference for variable importance},
  author={Van der Laan, Mark J},
  journal={The International Journal of Biostatistics},
  volume={2},
  number={1},
  year={2006},
  publisher={De Gruyter}
}

@inproceedings{ribeiro2016should,
  title={" Why should i trust you?" Explaining the predictions of any classifier},
  author={Ribeiro, Marco Tulio and Singh, Sameer and Guestrin, Carlos},
  booktitle={Proceedings of the 22nd ACM SIGKDD international conference on knowledge discovery and data mining},
  pages={1135--1144},
  year={2016}
}

@article{guidotti2018survey,
  title={A survey of methods for explaining black box models},
  author={Guidotti, Riccardo and Monreale, Anna and Ruggieri, Salvatore and Turini, Franco and Giannotti, Fosca and Pedreschi, Dino},
  journal={ACM computing surveys (CSUR)},
  volume={51},
  number={5},
  pages={1--42},
  year={2018},
  publisher={ACM New York, NY, USA}
}

@article{zhang2020floodgate,
  title={Floodgate: inference for model-free variable importance},
  author={Zhang, Lu and Janson, Lucas},
  journal={arXiv preprint arXiv:2007.01283},
  year={2020}
}

@article{bach2015pixel,
  title={On pixel-wise explanations for non-linear classifier decisions by layer-wise relevance propagation},
  author={Bach, Sebastian and Binder, Alexander and Montavon, Gr{\'e}goire and Klauschen, Frederick and M{\"u}ller, Klaus-Robert and Samek, Wojciech},
  journal={PloS one},
  volume={10},
  number={7},
  pages={e0130140},
  year={2015},
  publisher={Public Library of Science San Francisco, CA USA}
}

@inproceedings{shrikumar2017learning,
  title={Learning important features through propagating activation differences},
  author={Shrikumar, Avanti and Greenside, Peyton and Kundaje, Anshul},
  booktitle={International conference on machine learning},
  pages={3145--3153},
  year={2017},
  organization={PMlR}
}

@article{lei2018distribution,
  title={Distribution-free predictive inference for regression},
  author={Lei, Jing and G’Sell, Max and Rinaldo, Alessandro and Tibshirani, Ryan J and Wasserman, Larry},
  journal={Journal of the American Statistical Association},
  volume={113},
  number={523},
  pages={1094--1111},
  year={2018},
  publisher={Taylor \& Francis}
}

@book{bickel1993efficient,
  title={Efficient and adaptive estimation for semiparametric models},
  author={Bickel, Peter J and Klaassen, Chris AJ and Bickel, Peter J and Ritov, Ya’acov and Klaassen, J and Wellner, Jon A and Ritov, YA'Acov},
  volume={4},
  year={1993},
  publisher={Johns Hopkins University Press Baltimore}
}

@article{kennedy2016semiparametric,
  title={Semiparametric theory and empirical processes in causal inference},
  author={Kennedy, Edward H},
  journal={Statistical causal inferences and their applications in public health research},
  pages={141--167},
  year={2016},
  publisher={Springer}
}

@book{kosorok2008introduction,
  title={Introduction to empirical processes and semiparametric inference},
  author={Kosorok, Michael R},
  volume={61},
  year={2008},
  publisher={Springer}
}

@book{tsiatis2006semiparametric,
  title={Semiparametric theory and missing data},
  author={Tsiatis, Anastasios A},
  volume={4},
  year={2006},
  publisher={Springer}
}

@article{van2006targeted,
  title={Targeted maximum likelihood learning},
  author={Van Der Laan, Mark J and Rubin, Daniel},
  journal={The international journal of biostatistics},
  volume={2},
  number={1},
  year={2006},
  publisher={De Gruyter}
}

@inproceedings{chernozhukov2022riesznet,
  title={Riesznet and forestriesz: Automatic debiased machine learning with neural nets and random forests},
  author={Chernozhukov, Victor and Newey, Whitney and Quintas-Mart{\i}nez, V{\i}ctor M and Syrgkanis, Vasilis},
  booktitle={International Conference on Machine Learning},
  pages={3901--3914},
  year={2022},
  organization={PMLR}
}

@article{chernozhukov2023automatic,
  title={Automatic debiased machine learning for covariate shifts},
  author={Chernozhukov, Victor and Newey, Michael and Newey, Whitney K and Singh, Rahul and Srygkanis, Vasilis},
  journal={arXiv preprint arXiv:2307.04527},
  year={2023}
}

@article{foster2023orthogonal,
  title={Orthogonal statistical learning},
  author={Foster, Dylan J and Syrgkanis, Vasilis},
  journal={The Annals of Statistics},
  volume={51},
  number={3},
  pages={879--908},
  year={2023},
  publisher={Institute of Mathematical Statistics}
}

@article{whitehouse2024orthogonal,
  title={Orthogonal causal calibration},
  author={Whitehouse, Justin and Jung, Christopher and Syrgkanis, Vasilis and Wilder, Bryan and Wu, Zhiwei Steven},
  journal={arXiv preprint arXiv:2406.01933},
  year={2024}
}

@inproceedings{oprescu2019orthogonal,
  title={Orthogonal random forest for causal inference},
  author={Oprescu, Miruna and Syrgkanis, Vasilis and Wu, Zhiwei Steven},
  booktitle={International Conference on Machine Learning},
  pages={4932--4941},
  year={2019},
  organization={PMLR}
}

@article{wager2018estimation,
  title={Estimation and inference of heterogeneous treatment effects using random forests},
  author={Wager, Stefan and Athey, Susan},
  journal={Journal of the American Statistical Association},
  volume={113},
  number={523},
  pages={1228--1242},
  year={2018},
  publisher={Taylor \& Francis}
}

@article{luedtke2016statistical,
  title={Statistical inference for the mean outcome under a possibly non-unique optimal treatment strategy},
  author={Luedtke, Alexander R and Van Der Laan, Mark J},
  journal={Annals of statistics},
  volume={44},
  number={2},
  pages={713},
  year={2016}
}

@article{goldberg2014comment,
  title={Comment on “Dynamic treatment regimes: Technical challenges and applications”},
  author={Goldberg, Yair and Song, Rui and Zeng, Donglin and Kosorok, Michael R},
  journal={Electronic journal of statistics},
  volume={8},
  number={1},
  pages={1290},
  year={2014}
}

@article{chen2023inference,
  title={Inference on Optimal Dynamic Policies via Softmax Approximation},
  author={Chen, Qizhao and Austern, Morgane and Syrgkanis, Vasilis},
  journal={arXiv preprint arXiv:2303.04416},
  year={2023}
}

@article{levis2023covariate,
  title={Covariate-assisted bounds on causal effects with instrumental variables},
  author={Levis, Alexander W and Bonvini, Matteo and Zeng, Zhenghao and Keele, Luke and Kennedy, Edward H},
  journal={arXiv preprint arXiv:2301.12106},
  year={2023}
}

@inproceedings{balke1994counterfactual,
  title={Counterfactual probabilities: Computational methods, bounds and applications},
  author={Balke, Alexander and Pearl, Judea},
  booktitle={Uncertainty in artificial intelligence},
  pages={46--54},
  year={1994},
  organization={Elsevier}
}

@article{balke1997bounds,
  title={Bounds on treatment effects from studies with imperfect compliance},
  author={Balke, Alexander and Pearl, Judea},
  journal={Journal of the American statistical Association},
  volume={92},
  number={439},
  pages={1171--1176},
  year={1997},
  publisher={Taylor \& Francis}
}

@article{covert2020understanding,
  title={Understanding global feature contributions with additive importance measures},
  author={Covert, Ian and Lundberg, Scott M and Lee, Su-In},
  journal={Advances in Neural Information Processing Systems},
  volume={33},
  pages={17212--17223},
  year={2020}
}

@article{owen2017shapley,
  title={On Shapley value for measuring importance of dependent inputs},
  author={Owen, Art B and Prieur, Cl{\'e}mentine},
  journal={SIAM/ASA Journal on Uncertainty Quantification},
  volume={5},
  number={1},
  pages={986--1002},
  year={2017},
  publisher={SIAM}
}

@inproceedings{verhaeghe2023powershap,
  author    = {Verhaeghe, Jarne and Van Der Donckt, Jeroen and Ongenae, Femke and Van Hoecke, Sofie},
  title     = {{PowerShap}: A Power-Full {Shapley} Feature Selection Method},
  booktitle = {Machine Learning and Knowledge Discovery in Databases},
  year      = {2023},
  publisher = {Springer International Publishing},
  address   = {Cham},
  pages     = {71--87},
  isbn      = {978-3-031-26387-3},
}

@book{van2000asymptotic,
  title={Asymptotic statistics},
  author={Van der Vaart, Aad W},
  volume={3},
  year={2000},
  publisher={Cambridge university press}
}

@book{wainwright2019high,
  title={High-dimensional statistics: A non-asymptotic viewpoint},
  author={Wainwright, Martin J},
  volume={48},
  year={2019},
  publisher={Cambridge university press}
}

@inproceedings{hardt2021amazon,
  title={Amazon sagemaker clarify: Machine learning bias detection and explainability in the cloud},
  author={Hardt, Michaela and Chen, Xiaoguang and Cheng, Xiaoyi and Donini, Michele and Gelman, Jason and Gollaprolu, Satish and He, John and Larroy, Pedro and Liu, Xinyu and McCarthy, Nick and others},
  booktitle={Proceedings of the 27th ACM SIGKDD conference on knowledge discovery \& data mining},
  pages={2974--2983},
  year={2021}
}

@article{liu2022diagnosis,
  title={Diagnosis of Parkinson's disease based on SHAP value feature selection},
  author={Liu, Yuchun and Liu, Zhihui and Luo, Xue and Zhao, Hongjingtian},
  journal={Biocybernetics and Biomedical Engineering},
  volume={42},
  number={3},
  pages={856--869},
  year={2022},
  publisher={Elsevier}
}

@article{sebastian2024feature,
  title={A feature selection method based on Shapley values robust for concept shift in regression},
  author={Sebasti{\'a}n, Carlos and Gonz{\'a}lez-Guill{\'e}n, Carlos E},
  journal={Neural Computing and Applications},
  volume={36},
  number={23},
  pages={14575--14597},
  year={2024},
  publisher={Springer}
}

@article{dibaeinia2025interpretable,
  title={Interpretable AI for inference of causal molecular relationships from omics data},
  author={Dibaeinia, Payam and Ojha, Abhishek and Sinha, Saurabh},
  journal={Science Advances},
  volume={11},
  number={7},
  pages={eadk0837},
  year={2025},
  publisher={American Association for the Advancement of Science}
}

@inproceedings{marcilio2020explanations,
  title={From explanations to feature selection: assessing SHAP values as feature selection mechanism},
  author={Marc{\'\i}lio, Wilson E and Eler, Danilo M},
  booktitle={2020 33rd SIBGRAPI conference on Graphics, Patterns and Images (SIBGRAPI)},
  pages={340--347},
  year={2020},
  organization={Ieee}
}

@article{kraev2024shap,
  title={Shap-Select: Lightweight Feature Selection Using SHAP Values and Regression},
  author={Kraev, Egor and Koseoglu, Baran and Traverso, Luca and Topiwalla, Mohammed},
  journal={arXiv preprint arXiv:2410.06815},
  year={2024}
}

@article{keany2020borutashap,
  title={BorutaShap: A wrapper feature selection method which combines the Boruta feature selection algorithm with Shapley values.},
  author={Keany, Eoghan},
  journal={Zenodo},
  year={2020}
}

@article{bartlett2005local,
  title={Local rademacher complexities},
  author={Bartlett, Peter L and Bousquet, Olivier and Mendelson, Shahar},
  year={2005}
}

@article{roth1988introduction,
  title={Introduction to the Shapley value},
  author={Roth, Alvin E},
  journal={The Shapley value},
  volume={1},
  year={1988},
  publisher={University of Cambridge Press, Cambridge}
}

@article{winter2002shapley,
  title={The shapley value},
  author={Winter, Eyal},
  journal={Handbook of game theory with economic applications},
  volume={3},
  pages={2025--2054},
  year={2002},
  publisher={Elsevier}
}

@article{morzywolek2025inference,
  title={Inference on Local Variable Importance Measures for Heterogeneous Treatment Effects},
  author={Morzywolek, Pawel and Gilbert, Peter B and Luedtke, Alex},
  journal={arXiv preprint arXiv:2510.18843},
  year={2025}
}

@article{herbinger2024decomposing,
  title={Decomposing global feature effects based on feature interactions},
  author={Herbinger, Julia and Wright, Marvin N and Nagler, Thomas and Bischl, Bernd and Casalicchio, Giuseppe},
  journal={Journal of Machine Learning Research},
  volume={25},
  number={381},
  pages={1--65},
  year={2024}
}

@article{poudel2024explaining,
  title={Explaining customer churn prediction in telecom industry using tabular machine learning models},
  author={Poudel, Sumana Sharma and Pokharel, Suresh and Timilsina, Mohan},
  journal={Machine Learning with Applications},
  volume={17},
  pages={100567},
  year={2024},
  publisher={Elsevier}
}

@article{peng2023research,
  title={Research on customer churn prediction and model interpretability analysis},
  author={Peng, Ke and Peng, Yan and Li, Wenguang},
  journal={Plos one},
  volume={18},
  number={12},
  pages={e0289724},
  year={2023},
  publisher={Public Library of Science San Francisco, CA USA}
}

@article{aas2021explaining,
  title={Explaining individual predictions when features are dependent: More accurate approximations to Shapley values},
  author={Aas, Kjersti and Jullum, Martin and L{\o}land, Anders},
  journal={Artificial Intelligence},
  volume={298},
  pages={103502},
  year={2021},
  publisher={Elsevier}
}

@article{mahajan2023development,
  title={Development and validation of a machine learning model to identify patients before surgery at high risk for postoperative adverse events},
  author={Mahajan, Aman and Esper, Stephen and Oo, Thien Htay and McKibben, Jeffery and Garver, Michael and Artman, Jamie and Klahre, Cynthia and Ryan, John and Sadhasivam, Senthilkumar and Holder-Murray, Jennifer and others},
  journal={JAMA Network Open},
  volume={6},
  number={7},
  pages={e2322285--e2322285},
  year={2023},
  publisher={American Medical Association}
}

@article{yap2021verifying,
  title={Verifying explainability of a deep learning tissue classifier trained on RNA-seq data},
  author={Yap, Melvyn and Johnston, Rebecca L and Foley, Helena and MacDonald, Samual and Kondrashova, Olga and Tran, Khoa A and Nones, Katia and Koufariotis, Lambros T and Bean, Cameron and Pearson, John V and others},
  journal={Scientific reports},
  volume={11},
  number={1},
  pages={2641},
  year={2021},
  publisher={Nature Publishing Group UK London}
}

@article{whitehouse2025inference,
  title={Inference on optimal policy values and other irregular functionals via smoothing},
  author={Whitehouse, Justin and Austern, Morgane and Syrgkanis, Vasilis},
  journal={arXiv preprint arXiv:2507.11780},
  year={2025}
}

@book{pearl2014probabilistic,
  title={Probabilistic reasoning in intelligent systems: networks of plausible inference},
  author={Pearl, Judea},
  year={2014},
  publisher={Elsevier}
}

@book{bogachev2007measure,
  title={Measure theory},
  author={Bogachev, Vladimir Igorevich and Ruas, Maria Aparecida Soares},
  volume={1},
  number={1},
  year={2007},
  publisher={Springer}
}

@book{dellacherie2011probabilities,
  title={Probabilities and potential, c: potential theory for discrete and continuous semigroups},
  author={Dellacherie, Claude and Meyer, P-A},
  volume={151},
  year={2011},
  publisher={Elsevier}
}

@article{poussin1915integrale,
  title={Sur l'int{\'e}grale de Lebesgue},
  author={Poussin, C De La Vall{\'e}e},
  journal={Transactions of the American Mathematical Society},
  pages={435--501},
  year={1915},
  publisher={JSTOR}
}
